\documentclass{article} 
\usepackage{iclr2027_conference,times}

\usepackage[utf8]{inputenc} 
\usepackage[T1]{fontenc}    
\usepackage{hyperref}       
\usepackage{url}            
\usepackage{booktabs}       
\usepackage{amsfonts}       
\usepackage{nicefrac}       
\usepackage{microtype}      
\usepackage{xcolor}         

\usepackage{microtype}
\usepackage{graphicx}
\usepackage{subcaption}
\usepackage{array}

\usepackage{amsmath}
\usepackage{amssymb}
\usepackage{mathtools}
\usepackage{amsthm}

\usepackage{microtype}
\usepackage{graphicx}
\usepackage{subcaption}
\usepackage{booktabs} 

\usepackage[listings]{tcolorbox}
\usepackage{listings}
\usepackage{caption}
\usepackage[capitalize,noabbrev]{cleveref}
\usepackage{placeins}

\theoremstyle{plain}
\newtheorem{theorem}{Theorem}[section]
\newtheorem{proposition}[theorem]{Proposition}

\theoremstyle{definition}

\theoremstyle{remark}

\usepackage[textsize=tiny]{todonotes}

\usepackage{amsmath,amsfonts,bm}

\def\eqref#1{equation~\ref{#1}}

\def\1{\bm{1}}

\def\vx{{\bm{x}}}

\def\mX{{\bm{X}}}

\DeclareMathAlphabet{\mathsfit}{\encodingdefault}{\sfdefault}{m}{sl}
\SetMathAlphabet{\mathsfit}{bold}{\encodingdefault}{\sfdefault}{bx}{n}

\newcommand{\R}{\mathbb{R}}

\newcommand{\softmax}{\mathrm{softmax}}

\usepackage{algorithm}
\usepackage{algorithmic}
\newcommand{\If}[1]{\IF{#1}}

\newcommand{\EndIf}{\ENDIF}
\newcommand{\BlankLine}{\vspace{0.5\baselineskip}}
\newcommand{\KwIn}[1]{\STATE \textbf{Input:} #1}
\newcommand{\KwOut}[1]{\STATE \textbf{Output:} #1}
\newcommand{\Return}[1]{\STATE \textbf{Return} #1}

\definecolor{commandcolor}{rgb}{0.0, 0.5, 0.0}
\definecolor{keywordcolor}{rgb}{0.0, 0.0, 0.6}
\definecolor{stringcolor}{rgb}{0.6, 0.0, 0.0}
\definecolor{boxbg}{rgb}{0.97, 0.97, 0.97}

\lstdefinestyle{pytorch}{
    backgroundcolor=\color{boxbg},
    frame=single,
    language=Python,
    basicstyle=\ttfamily\small,
    keywordstyle=\color{keywordcolor}\bfseries,
    stringstyle=\color{stringcolor},
    commentstyle=\color{gray}\itshape,
    emph={[1]torch,nn,optim,DataLoader,Tensor},
    emphstyle={[1]\color{commandcolor}\bfseries},
    emph={[2]einsum,view},
    emphstyle={[2]\color{keywordcolor}\bfseries},
    commentstyle=\color{orange}\itshape,  
    showstringspaces=false,
    breaklines=true,
    tabsize=4,
    numbers=left,                
    numberstyle=\tiny\color{gray}, 
    numbersep=5pt               
}

\newtcblisting{pytorchbox}[1][]{
   colback=boxbg,
   colframe=black,
   listing only,
   listing options={style=pytorch},
   #1
}

\title{HOLA: Hadamard Higher-Order Attention}

\author{%
  Hanno Ackermann, Hong Cai, Mohsen Ghafoorian \& Amirhossein Habibian \\
  Qualcomm AI Research
}

\iclrfinalcopy 
\begin{document}

\maketitle
\lhead{}          


\begin{abstract}

The attention mechanism is an important reason for the success of transformers. It relies on computing pairwise relations between a query and all key tokens and applying softmax normalization. 
The expressiveness of softmax stems from the exponential kernel, 
which, however, causes quadratic computational complexity. 
Linear attention replaces the exponential by a single inner product of feature maps. In terms of the Taylor-series expansion of the exponential function, this amounts to a first-order approximation. 
Prior work has shown that higher-order terms can also be computed in linear time.

In this paper, we identify a limitation shared by existing higher-order linear attention methods based on Hadamard products: they repeat the same query map in every factor. We show that this tying confines the lifted queries to a symmetric subspace, which lowers the
upper bound on the rank of the score matrix from $d_\phi^F$ to $\binom{d_\phi+F-1}{F}$, a
reduction by a factor approaching $F!$, i.e., almost half at $F=2$.
We propose Hadamard
Higher-Order Linear Attention (HOLA), which uses a distinct feature map for each factor and
thereby restores the full $d_\phi^F$ ceiling at the same state size and essentially the same
cost. HOLA reduces to standard linear attention at order one, contains the tied construction
of prior work as a special case, and is compatible with gating, decay, and delta-rule updates.

\end{abstract}



\section{Introduction}
The transformer architecture has demonstrated remarkable success across a wide range of domains, for instance, question answering~\cite{sima2025drivelmdrivinggraphvisual}, reasoning~\cite{wei2022cot}, and even tasks such as 3D reconstruction~\cite{wang2024dust3rgeometric3dvision}. 
The attention mechanism is central to the success of transformers, but its cost grows quadratically with
sequence length. While recent advances~\cite{dao2022flashattention, dao2023flashattention2,
shah2024flashattention3fastaccurateattention} have removed the quadratic \emph{memory} cost, the
\emph{computational} cost remains quadratic. This is the binding constraint in long-context regimes such
as long-context question answering, code and video generation, where sequences easily reach $10^4$--$10^6$ tokens.


Linear attention~\cite{pmlr-v119-katharopoulos20a, schlag2021lineartransformerssecretlyfast} replaces
the exponential function used by the standard attention mechanism by a separable kernel, 
which 
yields linear cost and a constant-size recurrent state. The price is expressiveness:
the resulting state matrix has a low rank constraint, 
whereas softmax attention maps are empirically
high-rank 
with extreme peaks. 

Prior methods can be partitioned into several categories based on the aspect they address. 
The first is a better or learned feature map~\cite{zhang2024hedgehog, zhang2024lolcats, qin2022cosformerrethinkingsoftmaxattention, meng2025polaformer}. The second is the \emph{state update rule}: decay, gating and the delta rule~\cite{schlag2021lineartransformerssecretlyfast, yang2024gla, yang2025gateddeltanetworksimproving, dao2024transformersssmsgeneralizedmodels}, or sparsity and hybridization with exact attention~\cite{zhang2025slasparsitydiffusiontransformers, yuan2025nsa, lu2025moba}. 
Both leave the attention score \emph{bilinear} in the transformed pair, so have \emph{polynomial degree} $1$. 
Softmax attention, on the other hand, differs precisely 
with respect to this, 
since the degrees of the monomials of the Taylor series expansion of the exponential function increase without bound. 

The order-two term of this expansion admits an efficient  factorization. 
This observation is due to BASED~\cite{arora2024based} and, independently,
TaylorShift~\cite{nauen2024taylorshift}; PolySketchFormer~\cite{kacham2024polysketchformer} extends a
related factorization to higher degrees using sketching. 
They compute a Hadamard 
product of a single score matrix with itself. The score matrix is built from tied query and key maps: the same map is applied to both terms of the Hadamard product for BASED~\cite{arora2024based} and TaylorShift~\cite{nauen2024taylorshift}. While PolySketchFormer~\cite{kacham2024polysketchformer} uses different feature maps for queries and keys, the maps are still tied among the factors of the Hadamard product. 

Although computing the Hadamard product of a score matrix with itself should significantly relax the low-rank constraint on the state matrix, we will point out that using \emph{shared} feature maps, i.e. the same feature maps for queries or keys across Hadamard factors,  induces a symmetry which in turn causes a much tighter low-rank constraint. For instance, with $F=2$ Hadamard factors and latent dimensions of lifted queries and keys $d_\phi=16$, the rank of the score matrix is 
limited to $136$ instead of $256$, i.e. the upper bound is reduced to almost half its theoretical value. 
We therefore propose to use independent feature maps for each factor of the Hadamard product. 
We evaluate the proposed algorithm on challenging long-range associative retrieval and video generation with sequence lengths of several $10$k tokens.

\section{Related Works}

Linear attention has been proposed to solve the quadratic memory requirement and computation complexity that standard softmax-based attention has~\cite{shen2024efficientattentionattentionlinear, pmlr-v119-katharopoulos20a}. It has been noted that algorithms relying on linear attention often perform poorly compared to algorithms using standard attention. The authors of~\cite{schlag2021lineartransformerssecretlyfast} point out that linear attention is equivalent to fast weight programmers~\cite{schmidhuber92fastweight}. The same authors also propose to use a decay factor to hinder retention of irrelevant information over a sequence via the Delta rule. Since the latency of this algorithm increases with the sequence length, a chunked version of this algorithm has been proposed in~\cite{qiu2025gatedattentionlargelanguage}. For LLM training, a similar idea has been proposed in \cite{sun2023retentivenetworksuccessortransformer}.

\cite{peng2021randomfeatureattention, choromanski2022rethinkingattentionperformers, qin2022cosformerrethinkingsoftmaxattention, meng2025polaformer} propose how to better approximate the exponential function in the context of linear attention. The low-rank constraint inherent to linear attention is a likely cause of the poor performance of linear attention. A solution has been proposed in~\cite{fan2025lowrankdilemma}. Other approaches have been proposed as well, for instance based on maximal coding rate reduction~\cite{wu2024tokenstatisticstransformerlineartime} or locality sensitive hashing~\cite{han2023hyperattentionlongcontextattentionnearlinear}.

A different approach to reducing the quadratic complexity of standard attention is State Space Methods (SSMs). They model the temporal change by a set of first-order differential equations which are recurrently applied to the input sequence. They can perform well on long-range tasks~\cite{gu2022efficientlymodelinglongsequences} and can match or even transformers in some tasks~\cite{gu2024mambalineartimesequencemodeling}. 
Connections between linear attention and SSMs were revealed in  \cite{dao2024transformersssmsgeneralizedmodels}. Using the delta rule as in~\cite{schlag2021lineartransformerssecretlyfast} can also benefit SSMs~\cite{yang2025gateddeltanetworksimproving}. Drawing on ideas from SSMs, the authors of~\cite{egorov2025myosotisstructuredcomputationattention} improve a standard attention attention mechanism without sacrificing performance.

Linear attention has been used to improve diffusion networks, for instance in~\cite{xie2024sanaefficienthighresolutionimage, zhu2024digscalableefficientdiffusion} for images. Linear attention improved by the delta rule has also been used in diffusion algorithms for video generation~\cite{chen2025sanavideoefficientvideogeneration, kimiteam2025kimilinearexpressiveefficient}. The authors of~\cite{ye2025differentialtransformer} improve standard attention by performing it twice and subtracting the two results. 

Recently, \cite{zhang2025slasparsitydiffusiontransformers} and \cite{ghafoorian2025attentionsurgeryefficientrecipe} combined softmax attention with linear attention for video diffusion. This idea is to model large values in the attention matrix by softmax, yet smaller ones by linear attention. 
While these hybrid methods offer good performance, they require dynamical indexing
and are not softmax-free. Dynamical indexing means to directly selecting one or multiple individual values in a tensor. This operation is not supported by some low-power compute units.


\section{Background}
\label{sec:background}

\paragraph{Softmax attention.}
Let $\mX=\left[\vx_1^\top,\ldots,\vx_N^\top\right]\in\R^{N\times d}$ be $N$ tokens, mapped to queries,
keys and values $Q,K,V$ by learned projections $W_q,W_k,W_v$. Attention forms a convex combination of
values~\citep{vaswani2017attention},
\begin{equation}
  \softmax\!\left(\tfrac{1}{\sqrt{d}}\,QK^\top\right)V ,
  \label{eq:sm.att}
\end{equation}
at $\mathcal{O}(N^2d)$ cost. 
Implementations optimized for efficiency~\citep{dao2022flashattention,dao2023flashattention2,shah2024flashattention3fastaccurateattention}
remove the quadratic \emph{memory} cost but not the quadratic \emph{compute}.

\paragraph{Linear attention.}
Linear attention~\citep{pmlr-v119-katharopoulos20a,schlag2021lineartransformerssecretlyfast} replaces the
exponential by a separable kernel,
\begin{equation}
  \exp\langle q_i,k_j\rangle \;\approx\; \langle \phi(q_i),\phi(k_j)\rangle ,
  \label{eq:exp.sep.kernel}
\end{equation}
with a \emph{feature map} $\phi:\R^{d}\!\to\!\R^{d_\phi}$. Separability permits reassociating the
products,
\begin{equation}
  \phi(Q)\left(\phi(K)^\top V\right) \;=\; \phi(Q)\,C ,
  \label{eq:lin.att}
\end{equation}
so the $N\times N$ score matrix is never formed: cost drops to $\mathcal{O}(N d_\phi d)$ and the
$d_\phi\times d$ \emph{state} (or context) $C$ is all that must be stored, which also yields an
$\mathcal{O}(1)$-per-step recurrent form. The price is expressiveness. Since
$\mathrm{rank}(\phi(Q)C)\le\min(\mathrm{rank}\,\phi(Q),\mathrm{rank}\,C)\le d_\phi$, the implied score
matrix is confined to rank $d_\phi$, whereas measured softmax maps are
high-rank~\citep{zhang2025slasparsitydiffusiontransformers} and strongly
peaked~\citep{fan2025lowrankdilemma}; equivalently, the output is a low-degree rational function of the
inputs~\citep{letourneau2024padreunifyingpolynomialattention}.

Existing algorithms address two problems. The first is the \emph{feature map}: random or orthogonal Fourier
features~\citep{peng2021randomfeatureattention,choromanski2020rethinking}, trigonometric
surrogates~\citep{qin2022cosformerrethinkingsoftmaxattention,meng2025polaformer}, and maps trained to
mimic softmax attention~\citep{zhang2024hedgehog,zhang2024lolcats}; the last also \emph{untie} the query
and key maps, using separate $\phi_q$ and $\phi_k$. The second is the \emph{state update rule}: decay,
gating and the delta rule~\citep{schlag2021lineartransformerssecretlyfast,yang2024gla,%
yang2025gateddeltanetworksimproving,dao2024transformersssmsgeneralizedmodels}, or sparsity and
hybridisation with exact attention~\citep{zhang2025slasparsitydiffusiontransformers,yuan2025nsa,lu2025moba}.

\paragraph{Interaction order: a third, orthogonal approach.}
The algorithms in the paragraph above leave the score \emph{bilinear} in the transformed pair: however elaborate $\phi$ is,
$\phi(q_i)^\top\phi(k_j)$ is a single inner product. Softmax attention differs precisely here, since
\begin{equation}
  \exp\!\left(q_i^\top k_j\right) \;=\; \sum_{m\ge 0}^{\infty} \frac{\left(q_i^\top k_j\right)^m}{m!}
  \label{eq:exp.series}
\end{equation}
contains terms of every degree in the \emph{pair}. We call $F$ the \emph{interaction order} of a score
that is a degree-$F$ monomial in the pair, i.e.\ a product of $F$ separable terms; linear attention
has order one for \emph{any} feature map, and \eqref{eq:exp.series} has unbounded order. 

The degree-two term of \eqref{eq:exp.series} admits an exact linear-cost factorization, because
$(q^\top k)^2=(q\otimes q)^\top(k\otimes k)$ is an inner product in a lifted space. This is the
construction of Based~\citep{arora2024based} and, independently, TaylorShift~\citep{nauen2024taylorshift}.
Both order-two
constructions apply a single, shared map to queries and keys, so all factors of the product are
identical. 
Section~\ref{sec:hola} starts from exactly this factorization.

\begin{table}[t]
  \centering\small
  \caption{Notation. Throughout, $\odot$ is the Hadamard (element-wise) product, $\otimes$ the Kronecker
  product, and $\ast$ the row-wise Kronecker (face-splitting) product,
  $(A\ast B)_{i,:}=A_{i,:}\otimes B_{i,:}$. For $v\in\R^{n}$, $v^{\otimes F}=v\otimes\cdots\otimes v$
  ($F$ copies) denotes the $F$-fold Kronecker power, with entries
  $(v^{\otimes F})_{a_1\ldots a_F}=\prod_{f}v_{a_f}$.}
  \label{tab:notation}
  \begin{tabular}{ll}
    \toprule
    $N$, $d$ & sequence length, channels per head \\
    $Q,K,V\in\R^{N\times d}$ & queries, keys, values \\
    $F$ & interaction order (number of factors) \\
    $d_\phi$ & width of each feature map, $d_\phi\le d$ \\
    $\phi_{q_f},\phi_{k_f}:\R^{d}\!\to\!\R_{\ge0}^{d_\phi}$ & feature maps of factor $f=1,\dots,F$ \\
    $\tilde q_i,\tilde k_j$ & lifted features, dimension $d_\phi^F$ \\
    $S$, $z$ & state and normaliser \\
    $\mathrm{Sym}^F(\R^{d_\phi})$ & symmetric subspace, $\dim=\binom{d_\phi+F-1}{F}$ \\
    \bottomrule
  \end{tabular}
\end{table}


%
\section{Higher-Order Linear Attention}
\label{sec:hola}

\subsection{From the exponential series to order-\texorpdfstring{$F$}{F} scores}
\label{subsec:definition}

Linear attention truncates \eqref{eq:exp.series} at interaction order one. We ask instead for the cheapest
exact realisation of the higher-order terms. Let
$\phi_{q_f},\phi_{k_f}:\R^{d}\!\to\!\R_{\ge0}^{d_\phi}$, $f=1,\dots,F$, be learned feature maps and define
the \emph{order-$F$} score matrix as a Hadamard product of $F$ separable score matrices,
\begin{equation}
  A^{(F)} \;=\; \bigodot_{f=1}^{F}\Big(\phi_{q_f}(Q)\,\phi_{k_f}(K)^\top\Big)\;\in\;\R^{N\times N},
  \qquad
  A^{(F)}_{ij}=\prod_{f=1}^{F}\big\langle \phi_{q_f}(q_i),\,\phi_{k_f}(k_j)\big\rangle .
  \label{eq:hola}
\end{equation}
Nonnegative maps keep every entry nonnegative, so \eqref{eq:hola} can be row-normalised exactly as in
linear attention. Setting $F=1$ recovers linear attention; setting $\phi_{q_f}=\phi_{k_f}=\phi$ for all $f$
yields 
$\big(\phi(Q)\phi(K)^\top\big)^{\odot F}$ used at $F=2$ by
Based~\citep{arora2024based} and TaylorShift~\citep{nauen2024taylorshift}. Our construction differs from
those in three respects: 
the order $F$ is arbitrary and the factors are \emph{untied}, i.e.\
each carries its own maps rather than repeating one lifted feature map for queries or keys. The width $d_\phi$ is chosen
\emph{below} the per-head channel count $d$, which is what makes $F\ge3$ affordable at all.

\subsection{The symmetry constraint}
\label{subsec:symmetry}

Raising the interaction order is meant to lift the rank ceiling $d_\phi$ of linear attention. We now
show that how far it is lifted depends on whether the factors are tied. 
Since $A^{(F)}=\tilde Q\tilde K^\top$, or alternatively since
$\mathrm{rank}(A\odot B)\le\mathrm{rank}(A)\,\mathrm{rank}(B)$,
\begin{equation}
  \mathrm{rank}\big(A^{(F)}\big)\le\min\!\big(N,\ d_\phi^{F}\big).
  \label{eq:rank.ceiling}
\end{equation}
This bound is correct but, for the tied constructions of prior work, not attainable.

\begin{proposition}[Rank ceiling under tied query maps]
\label{prop:rank}
Let the query map be shared across factors, $\phi_{q_1}=\cdots=\phi_{q_F}=\phi_q$, and let the key maps
$\phi_{k_1},\dots,\phi_{k_F}$ be arbitrary. Then
\begin{equation}
  \mathrm{rank}\big(A^{(F)}\big)\;\le\;\min\!\left(N,\ \binom{d_\phi+F-1}{F}\right),
  \label{eq:rank.ceiling.symmetry}
\end{equation}
irrespective of whether the key maps are tied or untied. Exceeding this ceiling therefore requires
untying the \emph{query} maps.
\end{proposition}

\begin{proof}[Proof sketch]
Let $u=\phi_q(q_i)$ and let $\mathcal{I}=\{1,\dots,d_\phi\}$ index its coordinates.
For every multi-index $a=(a_1,\dots,a_F)\in\mathcal{I}^F$, the entry of $\tilde q_i=u^{\otimes F}$ is $(\tilde q_i)_a=\prod_{f=1}^{F}u_{a_f}$.
Since this product does not depend on the order of its factors, $(\tilde q_i)_a=(\tilde q_i)_{\sigma(a)}$ for every permutation $\sigma$ of the $F$ index positions, where $\sigma(a)=(a_{\sigma(1)},\dots,a_{\sigma(F)})$.
Hence, we have $\tilde q_i\in\mathrm{Sym}^F(\R^{d_\phi})$, whose dimension $\binom{d_\phi+F-1}{F}$ is the number of multisets of size $F$ drawn from $\mathcal{I}$.
Since $A^{(F)}=\tilde Q\tilde K^\top$ and the $N$ rows of $\tilde Q$ lie in $\mathrm{Sym}^F$, $\mathrm{rank}(A^{(F)})\le\mathrm{rank}(\tilde Q)\le\min\big(N,\dim\mathrm{Sym}^F\big)$, for any choice of key maps.
\end{proof}

The key maps cannot raise this ceiling: since $\tilde q_i$ lies in $\mathrm{Sym}^F$, its score with any key satisfies $\tilde q_i^\top\tilde k_j=\tilde q_i^\top P\tilde k_j$, where $P$ is the orthogonal projector onto $\mathrm{Sym}^F$. 
Tied queries therefore only see the symmetric part of the key lift, and untying the key maps alone changes nothing that the queries can read.


Based~\citep{arora2024based} and TaylorShift~\citep{nauen2024taylorshift} use one map for queries and keys,
and PolySketchFormer~\citep{kacham2024polysketchformer} uses different maps for queries and keys; all three,
however, repeat the same query map in every factor. Their score matrices are therefore limited to the
ceiling \eqref{eq:rank.ceiling.symmetry}, which lies below $d_\phi^F$ by a factor that approaches $F!$ as
$d_\phi$ grows (Table~\ref{tab:rank.ceiling}). Untying the query maps, as HOLA does, removes this
limitation at the same lift dimension $d_\phi^F$, and hence at the same state size; the only extra cost is
$F$ small query maps instead of one. Untying only the key maps does not raise the ceiling.

\begin{table}[t]
  \centering
  \caption{Upper bound on $\mathrm{rank}(A^{(F)})$ for $N\ge d_\phi^F$: tied query maps
  \eqref{eq:rank.ceiling.symmetry} vs.\ untied query maps \eqref{eq:rank.ceiling}.}
  \label{tab:rank.ceiling}
  \begin{tabular}{ccrrr}
    \hline
    $d_\phi$ & $F$ & tied & untied & ratio \\
    \hline
    $8$  & $2$ & $36$   & $64$   & $1.78$ \\
    $16$ & $2$ & $136$  & $256$  & $1.88$ \\
    $64$ & $2$ & $2080$ & $4096$ & $1.97$ \\
    $6$  & $3$ & $56$   & $216$  & $3.86$ \\
    $8$  & $3$ & $120$  & $512$  & $4.27$ \\
    \hline
  \end{tabular}
\end{table}

Proposition~\ref{prop:rank} is an upper bound: it shows that tied
query maps \emph{cannot} exceed $\binom{d_\phi+F-1}{F}$, not that untied maps always reach $d_\phi^F$. We
verify the latter empirically: on a single head with $N=512$, trained untied maps attain the full rank
$d_\phi^F$ at $F=2$ and $496$--$508$ of $512$ at $(d_\phi,F)=(8,3)$, whereas tied maps saturate at or just
below the tied ceiling (Fig.~\ref{fig:taylor_fit}). 
Furthermore, because a tied query lift only reads $P\tilde k$, the
state of a tied model could be stored exactly in $\binom{d_\phi+F-1}{F}$ rows. 
For a given lift dimension $d_\phi^F$, a tied model can match the
ceiling only with a larger feature width $d_\phi'$ (e.g., $d_\phi'=22$ instead of $16$ at $F=2$),
which costs $d_\phi'^F$ state entries in the standard uncompressed lift, or a comparable number
of entries if its state is stored compressed in $\binom{d_\phi'+F-1}{F}$ rows.

Our claim
is therefore stated at a given lift dimension $d_\phi^F$, which is what determines memory and compute in
existing higher-order kernels~\citep{arora2024based,nauen2024taylorshift,kacham2024polysketchformer}: at
that dimension, tying the query maps lowers the attainable rank by a factor approaching $F!$.

\subsection{Linear-time evaluation}
\label{subsec:efficient}

Na\"ively evaluating \eqref{eq:hola} directly has memory and computational complexity $\mathcal{O}(N^2)$. The Hadamard product of Gram matrices is
itself a Gram matrix in a lifted space, by the face-splitting (row-wise Kronecker) identity
$(AB^\top)\odot(CD^\top)=(A\ast C)(B\ast D)^\top$. 
Applying it $F-1$ times to \eqref{eq:hola} gives
$A^{(F)}_{ij}=\tilde q_i^\top\tilde k_j$ with the lifted features
\begin{equation}
  \tilde q_i=\phi_{q_1}(q_i)\otimes\cdots\otimes\phi_{q_F}(q_i)\in\R^{d_\phi^F},
  \qquad
  \tilde k_j=\phi_{k_1}(k_j)\otimes\cdots\otimes\phi_{k_F}(k_j)\in\R^{d_\phi^F} .
  \label{eq:lift}
\end{equation}
Notice that $\tilde q_i$ and $\tilde k_j$ are degree-$F$ polynomials of $q_i$ and $k_j$. 
By \eqref{eq:lift}, \eqref{eq:hola} may therefore be expressed by a single state, in $\mathcal{O}(N d_\phi^F d)$ time and
$\mathcal{O}(d_\phi^F d)$ memory,
\begin{equation}
  S=\sum_{j}\tilde k_j v_j^\top\in\R^{d_\phi^F\times d},
  \qquad
  z=\sum_{j}\tilde k_j\in\R^{d_\phi^F},
  \qquad
  o_i=\frac{\tilde q_i^\top S}{\tilde q_i^\top z} .
  \label{eq:state}
\end{equation}
Restricting both sums to $j\le i$ gives causal attention; multiplying the summands by decay factors gives
gated variants; and \eqref{eq:state} is a recurrence, so it also yields $\mathcal{O}(1)$-per-step decoding
with the state as the carried quantity. Because \eqref{eq:state} is ordinary linear attention on the lifted
key $\tilde k$, gating, decay and the delta
rule~\citep{schlag2021lineartransformerssecretlyfast,yang2024gla,yang2025gateddeltanetworksimproving}
readily apply with the key dimension $d$ replaced by $d_\phi^F$; Sec.~\ref{sec:experiments} verifies this
by upgrading a gated delta network. 

\subsection{Sharpening} 
\label{subsec:sharpening}

Order acts on the attention distribution in an exactly characterisable way. Fix a row $i$, assume tied
factors $\phi_{k_1}=\cdots=\phi_{k_F}=\phi_k$ and write $b_j=\langle\phi_q(q_i),\phi_k(k_j)\rangle>0$.
Then $A^{(F)}_{ij}=b_j^{F}$, so with $z_j=\log b_j$ the normalised weights are
\begin{equation}
  p^{(F)}_j=\frac{b_j^{F}}{\sum_{\ell}b_{\ell}^{F}}=\softmax(Fz)_j=\softmax_{T}(z)_j,
  \qquad T=\tfrac{1}{F} ,
  \label{eq:temperature}
\end{equation}
and every finite pairwise log-odds is scaled exactly by $F$,
$\log\!\big(p^{(F)}_j/p^{(F)}_{j'}\big)=F\,(z_j-z_{j'})$. Raising the interaction order is therefore
\emph{exactly} equivalent to lowering a softmax temperature on the log-scores: the ranking of keys is
preserved, while contrasts are amplified linearly in $F$, and the distribution concentrates on the
arg-max as $F$ grows. This is the sense in which order-$F$ attention is peaked like softmax and unlike
order-one linear attention, and it requires no approximation argument. With untied factors the
corresponding statement is an inequality: for factor scores $s_f(j)\in[0,1]$ one has
$A^{(F)}_{ij}=\prod_f s_f(j)\le\min_f s_f(j)$, so a key scores highly only if it scores highly in
\emph{every} factor simultaneously -- a product-of-experts effect. 
\section{Experiments}
\label{sec:experiments}

\subsection{Experimental Overview}
\label{sec:experimental-overview}

We evaluate the proposed higher-order attention mechanisms in two distinct domains. We study long-context associate retrieval by replacing the attention layers of Qwen2-1.5B and evaluating the resulting models on natural-language modeling, long-range associative retrieval. In the appendix, we integrate the proposed mechanism into the Wan2.1-1.3B video diffusion model and evaluate its ability to replace softmax attention in large-scale video generation. These experiments test whether the proposed mechanism transfers across modalities and whether its computational advantages can be obtained without substantially compromising generation quality. Detailed dataset construction procedures, architecture diagrams, and additional training specifications are provided in the appendix.

\subsection{Long-Context Associative Retrieval}
\label{sec:language-experiments}

\subsubsection{Models and Training}
\label{sec:language-training}

\paragraph{Model variants:}
We use Qwen2.5-1.5B-Instruct as the base language model. We evaluate this baseline both in a zero-shot mode and after fine-tuning using the same schedule as for the other models (see below). Furthermore, we replace its attention layers with the several linear attention models as well as the asymmetric Hadamard Linear Attention  mechanisms considered in this work. Specifically, we evaluate linear attention (linear-512) with a state size of $d=512$, sink sliding-window attention (SWA)~\cite{jolicoeurmartineau2026SWA}, chunked Gated Delta Net (GDN)~\cite{schlag2021lineartransformerssecretlyfast, yang2024gla, yang2025gateddeltanetworksimproving} and GDN-512 which uses a significantly larger state size of linear attention with latent dimension $512$, symmetric Hadamard Linear Attention with two factors, HLA-F2 $d_\phi=16$ and HLA-F2 $d_\phi=64$. The first uses a latent size of $d_{lat}=16$ for the MLPs whereas the latter uses a latent size of $d_{lat}=128$.  We indicate asymmetric HLA with $F=2$ by HLA-F2-asym, a model with all Taylor-series terms up to third order and tied feature maps, HOLA-F3-sym, and a hybrid architecture combining GDN with asymmetric HLA-F2, denoted GDN-asym-HLA-F2. 
Notice that symmetric HLA-F2 is similar to the model proposed by \cite{arora2024based, nauen2024taylorshift} but unties the key feature maps. 
HOLA-F3 differs from the model proposed by \cite{arora2024based} because it uses different feature maps for each term of the Taylor-series, though it still is \emph{symmetric} since the query feature maps are tied for the second and third order terms. As pointed out by a reviewer\footnote{We thank reviewer \emph{mZhf} of ICML 2027.}, there is causal leakage due to an incorrectly defined gate function in \eqref{eq:hola.f3.gate}. 
Unless stated otherwise, we fine-tune all model parameters. For SWA, full-parameter fine-tuning did not yield a stable training; we therefore train rank-$16$ LoRA adapters while keeping the remaining parameters fixed. Detailed definitions and architecture diagrams for all model variants are provided in Appendix~\ref{app:model-cards}.

\paragraph{Post-surgery recovery:}
Replacing the pretrained attention layers causes an immediate degradation in language-modeling performance because the newly introduced sequence-mixing modules are not aligned with the pretrained representations. We therefore perform a recovery phase before long-context fine-tuning. Each modified model is trained on next-token prediction using Wiki101 for $16{,}000$ optimization steps with a batch size of $8$. During this phase, we use a learning rate of $10^{-3}$ for the newly introduced sequence-mixing modules and $10^{-4}$ for all remaining trainable parameters.

\paragraph{Long-context curriculum.}
Following recovery training, we fine-tune the models using a three-stage curriculum with sequence lengths of $8$k, $16$k, and $32$k tokens. The learning rates for the three stages are $10^{-4}$, $5 \times 10^{-5}$, and $2.5 \times 10^{-5}$, respectively, for all models but the GDN variants which required lower learning rates for stable training ($5 \times 10^{-5}$, $2.5 \times 10^{-5}$ and $1.0 \times 10^{-5}$). At every stage, the sequence-mixing modules use a learning rate ten times larger than that of the remaining trainable parameters. We trained all but one model for $8000$, $4000$ and $4000$ steps. Only GDN-512 was trained with $8000$ steps for the third stage to yield better results. Longer training is not possible since the limited size of the dataset causes the model to overfit. 

The curriculum combines BookSum-based text summarization, open-vocabulary retrieval and copy tasks, and raw-prose continuation. BookSum examples train the models on long-form natural-language understanding, whereas the constructed tasks provide controlled signals for copying and retrieving information over long distances. See sec.~\ref{app:datasets} for more information about dataset and task creation. All tasks use a causal language-modeling objective, with loss computed only over assistant-response tokens. User instructions and source contexts are masked.

The loss weights for BookSum instruction following, associative retrieval and copy, and prose continuation are $0.25$, $0.75$, and $1.5$, respectively. Since the effective contribution of a task depends on both its loss weight and its sampling frequency, we report the complete data mixture and the corresponding sampling proportions in appendix~\ref{app:datasets}. The appendix also describes the BookSum preprocessing pipeline, the generation of synthetic retrieval examples, and the deterministic resampling procedure used to prevent memorization across training epochs.

\subsubsection{Evaluation Protocol}
\label{sec:language-evaluation}

We evaluate the language models along two complementary dimensions. First, we measure long-range retrieval using multi-query associative recall (MQAR) tasks. Second, we evaluate natural-language modeling on held-out BookSum examples.

MQAR evaluates the retrieval of multiple arbitrary facts inserted into the input. MQAR increases the difficulty over needle-in-a-haystack (NIAH) by varying the number of competing key-value associations. Preliminary experiments indicated that a straightforward construction of the retrieval data allowed models to exploit unintended shortcuts. We therefore substantially expanded the underlying sampling space to minimize the feasibility of shortcut-based solutions. We additionally evaluate generic copying as a lower-complexity control task. For all controlled tasks, the model generates its answer autoregressively, and we report exact-match accuracy.

We use the final checkpoints and evaluate them on the held-out validation data at sequence lengths of $8$k, $16$k and $32$k tokens. 




\subsubsection{Main Results}
\label{sec:language-results}

\begin{table}[t] 
    \centering 
    \setlength{\tabcolsep}{4pt} 
    \caption{Results on MQAR across different sequence lengths. NTP denotes next-token prediciotn loss, MQAR the retrieval loss (cross-entropy), ACC the accuracy in percent.}
    \label{tab:language-main-results} 
    \begin{tabular}{l*{9}{c}} 
        \toprule & \multicolumn{3}{c}{\textbf{8k}} & \multicolumn{3}{c}{\textbf{16k}} & \multicolumn{3}{c}{\textbf{32k}} \\ 
        \cmidrule(lr){2-4} 
        \cmidrule(lr){5-7} 
        \cmidrule(lr){8-10} Model & {\small NTP$\downarrow$} & {\small MQAR$\downarrow$} & {\small ACC$\uparrow$} & {\small NTP$\downarrow$} & {\small MQAR$\downarrow$} & {\small ACC$\uparrow$} & {\small NTP$\downarrow$} & {\small MQAR$\downarrow$} & {\small ACC$\uparrow$} \\ 
        \midrule \midrule 
        Qwen2.5-1.5B zero-shot & 2.41 & 1.05 & 0.0 & 1.85 & 1.02 & 0.0 & 1.30 & 1.00 & 0.0 \\ 
        Qwen2.5-1.5B fine-tuned & 3.26 & 0.1 & 99.7 & 2.55 & 0.2 & 99.0 & 1.78 & 0.1 & 99.3 \\ 
        \midrule
        Linear-512 & 5.35 & 3.69 & 0.0 & 5.41 & 2.64 & 0.0 & 5.94 & 2.61 & 0.0 \\ 
        SWA-lora16 & 3.59 & 2.76 & 0.0 & 3.64 & 2.89 & 0.0 & 3.60 & 2.92 & 0.0 \\ 
        GDN & 6.43 & 0.32 & 39.1 & 6.32 & 0.37 & 43.4 & 6.28 & 0.42 & 40.6 \\
        GDN-512 & 8.64 & 0.91 & 21.3 & 8.54 & 0.87 & 25.6 & 8.52 & 0.85 & 20.6 \\
        \midrule 
        \textbf{HLA-F2-sym} & 6.27 & 2.34 & 0.0 & 6.13 & 2.34 & 0.0 & 5.98 & 2.32 & 0.0 \\
        \textbf{HLA-F2-sym, $d_\phi=64$} & 4.94 & 1.55 & 0.6 & 5.14 & 1.69 & 0.1 & 5.33 & 1.74 & 0.6 \\
        \textbf{HOLA-F3-sym} & 5.03 & 1.03 & 8.7 & 5.12 & 0.84 & 19.4 & 5.31 & 0.92 & 24.7 \\
        \textbf{HLA-F2-asym} & 6.17 & 1.03 & 26.3 & 6.03 & 0.99 & 27.8 & 5.77 & 1.13 & 27.5 \\
        \textbf{GDN-HLA-F2-asym} & 6.83 & 0.07 & \textbf{84.1} & 6.74 & 0.05 & \textbf{82.8} & 6.67 & 0.08 & \textbf{79.4} \\
        
        \end{tabular} 
\end{table}







Table~\ref{tab:language-main-results} compares the sequence-mixing
mechanisms after the complete long-context curriculum. We organize the
evaluation around four questions: Firstly, can the modified models recover the language-modeling capabilities disrupted by attention replacement? Secondly, how does each mechanism scale as context length and retrieval distance increase? Thirdly, does asymmetric HLA improve performance of other linear attention variants such as Gated Delta Nets? Fourthly, does asymmetric HLA improve performance of other linear attention variants such as Gated Delta Nets?


The zero-shot Qwen2.5-1.5B baseline achieves no measurable retrieval accuracy at any evaluated sequence length, highlighting the difficulty of the MQAR task under this evaluation setting. Since softmax attention has access to all query-key interactions, the model achieves almost perfect accuracy after fine-tuning. 
Linear attention (linear-512) and SWA-lora16 fail to retrieve the correct association. Gated delta nets (GDN) perform much better. Increasing the state size of the internal linear attention model to $d=512$ worsen the performance, surprisingly. 

Symmetric HLA-F2 at $d_\phi=16$ also fails to retrieve the any correct association. Even symmetric HLA-F2 with $d_\phi=64$ performs significantly worse than the asymmetric version HLA-F2-asym that also uses $d_\phi=64$. Increasing the interaction order improves retrieval: HOLA-F3-sym reaches accuracies of $8.7\%$, $19.4\%$, and $24.7\%$ at sequence lengths of 8k, 16k, and 32k, respectively. Even though this model has a causal leakage, its performance is below that of the more simple asymmetric variant with $F=2$: 
Despite using only second-order interactions, HLA-F2-asym consistently outperforms both HLA-F2-sym and HOLA-F3-sym and attains accuracies between $26.3\%$ and $27.8\%$. Its performance also remains approximately stable across sequence lengths, suggesting that its retrieval capability transfers across the evaluated context-length range.

These retrieval gains are accompanied by weaker next-token prediction performance. All attention-replacement models have substantially higher NTP loss than the Qwen2.5-1.5B baseline. Thus, although higher-order and asymmetric interactions improve associative retrieval, they do not by themselves recover the language-modeling quality of the pretrained transformer. Since the focus of this study is the long-range associative retrieval, we did not investigate the language modeling any further but conjecture that a training on more language data would eventually recover the models' language capabilities. 
HOLA-F3-sym additionally exhibits a non-monotonic dependence on sequence length: its retrieval accuracy increases with sequence length, while its NTP loss is lowest at 8k and highest at 32k. 
This might be caused by forgetting: As the model is trained on the 16k and 32k stages, its performance on the 8k task deteriorates. The reason why the stronger models do not exhibit this behavior might be that they are sufficiently strong not to be affected.

GDN provides a stronger retrieval baseline than the standalone HLA variants, achieving accuracies of $39.1\%$-$43.4\%$. 
The best overall results are obtained by GDN-HLA-F2-asym, which replaces GDN's linear-attention component with asymmetric HLA-F2. The hybrid reaches accuracies of $84.1\%$, $82.8\%$, and $79.4\%$ at 8k, 16k, and 32k, respectively, corresponding to absolute improvements of $45.0$, $39.4$, and $38.8$ percentage points over standard GDN. It also achieves the lowest MQAR loss at every sequence length. These results indicate that the asymmetric HLA mechanism and the gating and state-update structure of GDN provide complementary benefits. 

\subsubsection{Ablation Studies}
\label{sec:ablations}

We conduct ablations to isolate the contributions of interaction order, symmetry and architectural hybridization.



We compare symmetric HLA-F2 at $d_\phi=64$, asymmetric HLA-F2 also at $d_\phi=64$, and HOLA-F3 under an otherwise identical language-model training setup, thereby isolating the effects of the symmetry constraint and interaction order from differences in data and optimization. HOLA-F3 provides a higher-order approximation of the Taylor expansion and has more parameters than HLA-F2. Nevertheless, asymmetric HLA-F2 consistently outperforms HOLA-F3. This comparison suggests that relaxing the symmetry constraint is more consequential than increasing the interaction order from two to three. In particular, the results support our hypothesis that symmetry restricts the effective rank of the state matrix, limiting model capacity despite the more accurate higher-order approximation provided by HOLA-F3.


We compare GDN, asymmetric HLA-F2, and GDN-asym-HLA-F2 to assess whether the two mechanisms provide complementary benefits. Specifically, we examine whether replacing the linear-attention component of GDN with asymmetric HLA-F2 improves robustness to competing associations without compromising natural-language modeling performance. While GDN and asymmetric HLA-F2 achieve comparable performance in isolation, their hybrid, GDN-asym-HLA-F2, performs best overall. These results indicate that asymmetric HLA-F2 can serve as an effective replacement for linear attention within GDN and that the two mechanisms provide complementary modeling capabilities.




\subsection{Computational Efficiency}
\label{sec:efficiency}

\begin{figure}[t]
  \begin{center}
      \includegraphics[width=0.32\linewidth]{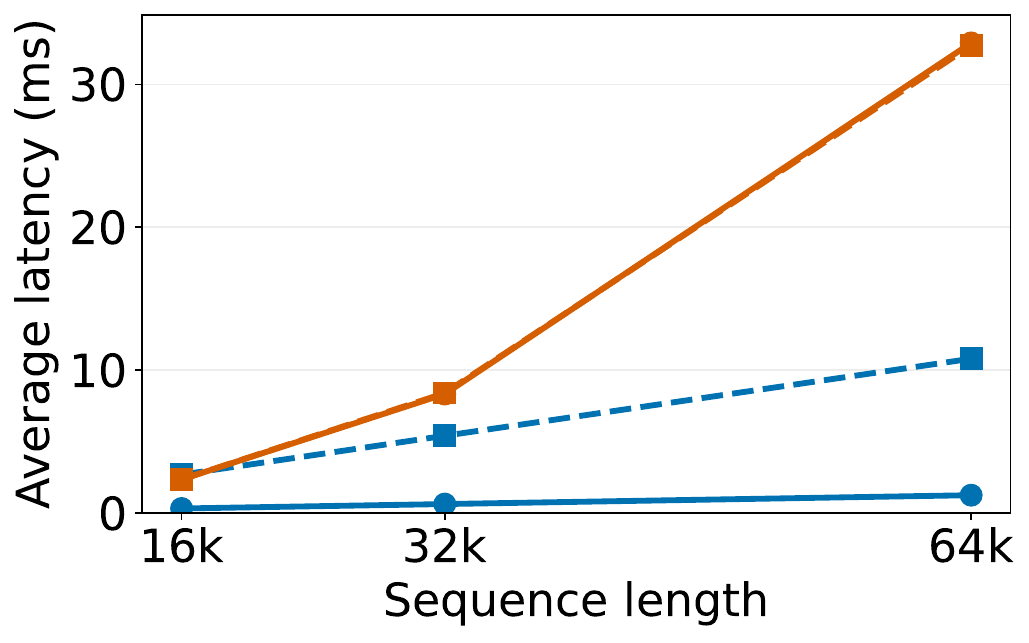}
      \includegraphics[width=0.32\linewidth]{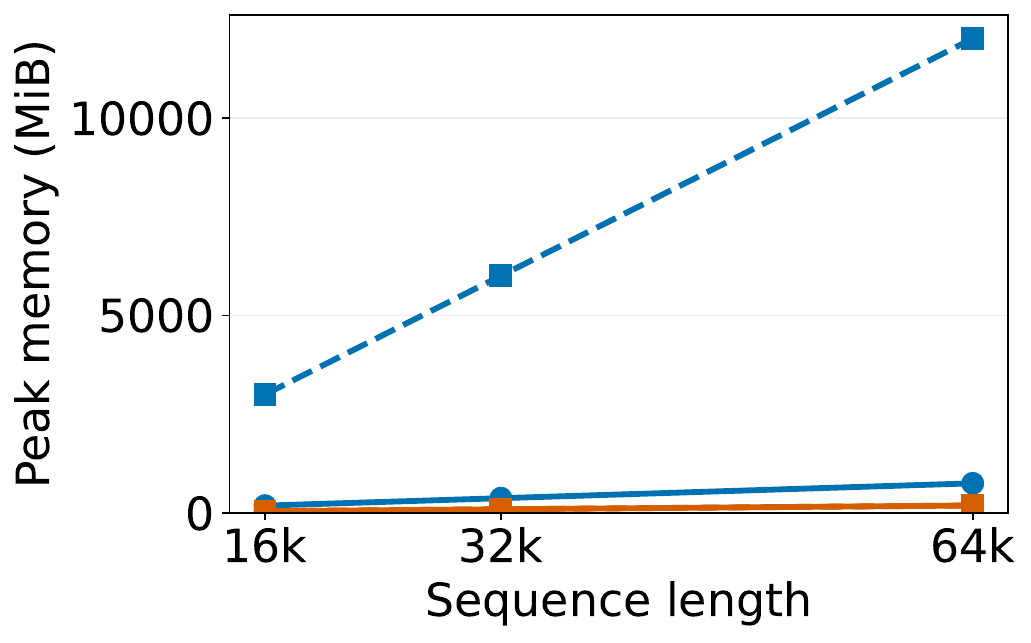}
      \includegraphics[width=0.32\linewidth]{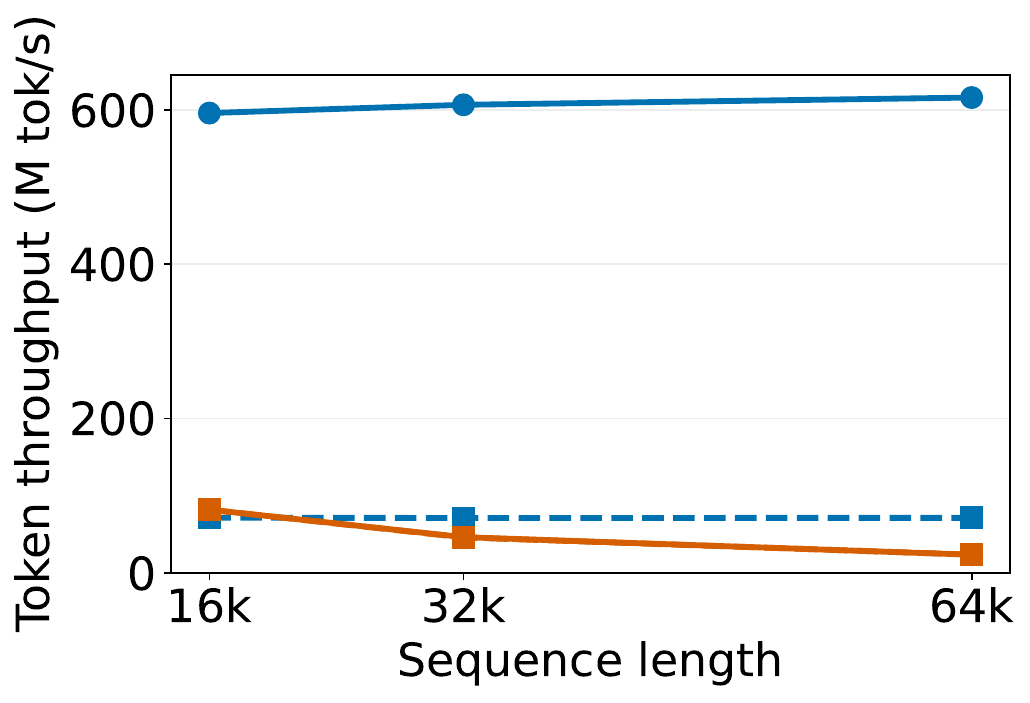}
      \caption{Computational efficiency of asymmetric HLA-F2 and FlashAttention. Forward latency (left), peak memory (middle), and token throughput (right) for sequence lengths from $16$k to $64$k, with $d_k=16$ (solid) and $d_k=64$ (dashed). HLA-F2 exhibits approximately linear latency growth and near-constant throughput, whereas FlashAttention shows quadratic latency growth and decreasing throughput. FlashAttention, however, retains a substantially smaller memory footprint.}
      \label{fig:efficiency}
  \end{center}
  \vspace{-0.6cm}
\end{figure}

We evaluate the computational efficiency of asymmetric HLA-F2 against FlashAttention-2 causal attention at batch size $1$, number of heads $12$, sequence lengths of $16$k, $32$k, and $64$k, using $d_k\in{16,64}$ and $d_v=128$. We report forward-pass latency, peak accelerator memory, and token throughput under otherwise fixed benchmark settings.

Figure~\ref{fig:efficiency} shows distinct scaling behavior between the two mechanisms. The latency of asymmetric HLA-F2 grows approximately linearly with sequence length, while FlashAttention exhibits the quadratic growth expected from softmax attention. Consequently, the latency advantage of HLA-F2 becomes increasingly pronounced at longer sequences. For $d_k=16$, HLA-F2 requires $0.32$ms, $0.63$ms, and $1.25$ms at $16$k, $32$k, and $64$k tokens, respectively, compared with $2.35$, $8.33$, and $32.90$ ms for FlashAttention. At $64$k tokens, this corresponds to a $26.3\times$ latency reduction.

The two approaches exhibit the opposite trade-off in memory. FlashAttention maintains a small memory footprint that grows only modestly with sequence length, whereas the HLA-F2 CUDA kernel exhibits linear memory growth. This dependence is particularly pronounced for $d_k=64$. In throughput, HLA-F2 remains approximately constant as sequence length increases, achieving roughly $596$-$616$M tokens/s for $d_k=16$ and about $71$M tokens/s for $d_k=64$. FlashAttention throughput, in contrast, decreases with increasing sequence length, from approximately $82$ M tokens/s at $16$k to $23$M tokens/s at $64$k. These results demonstrate the expected compute-scaling advantage of HLA-F2 at long sequence lengths, while also exposing its higher memory cost relative to FlashAttention.

\subsection{Approximating Softmax Attention}
\label{subsec:taylor_fit}

\begin{figure}[t]
  \centering
  \includegraphics[width=\linewidth]{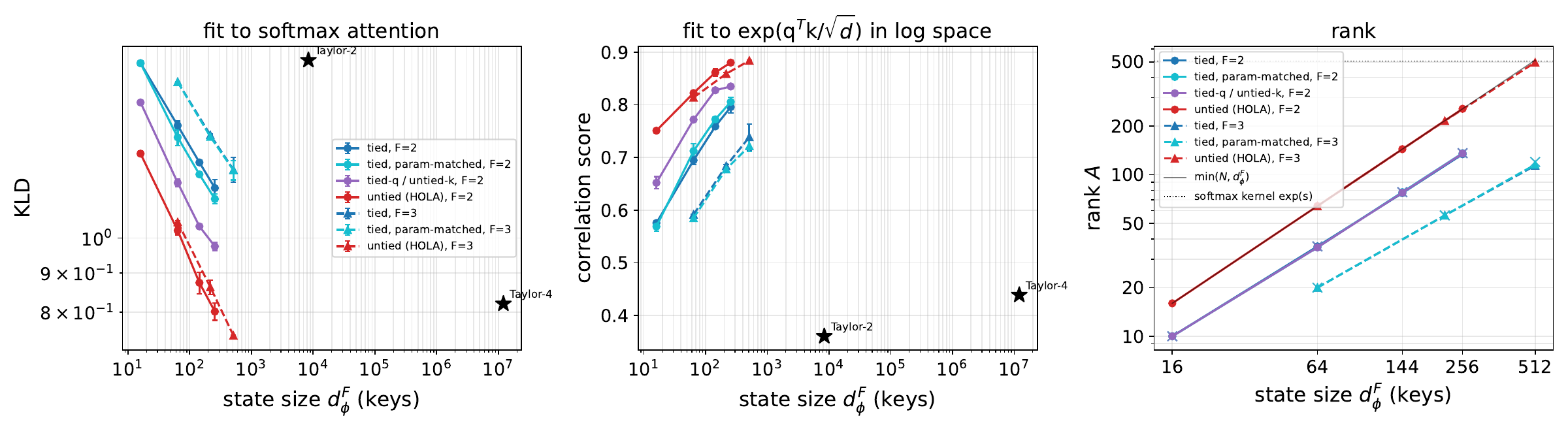}\\[-1mm]
  {\small (a) bidirectional attention}\\[1mm]
  \includegraphics[width=\linewidth]{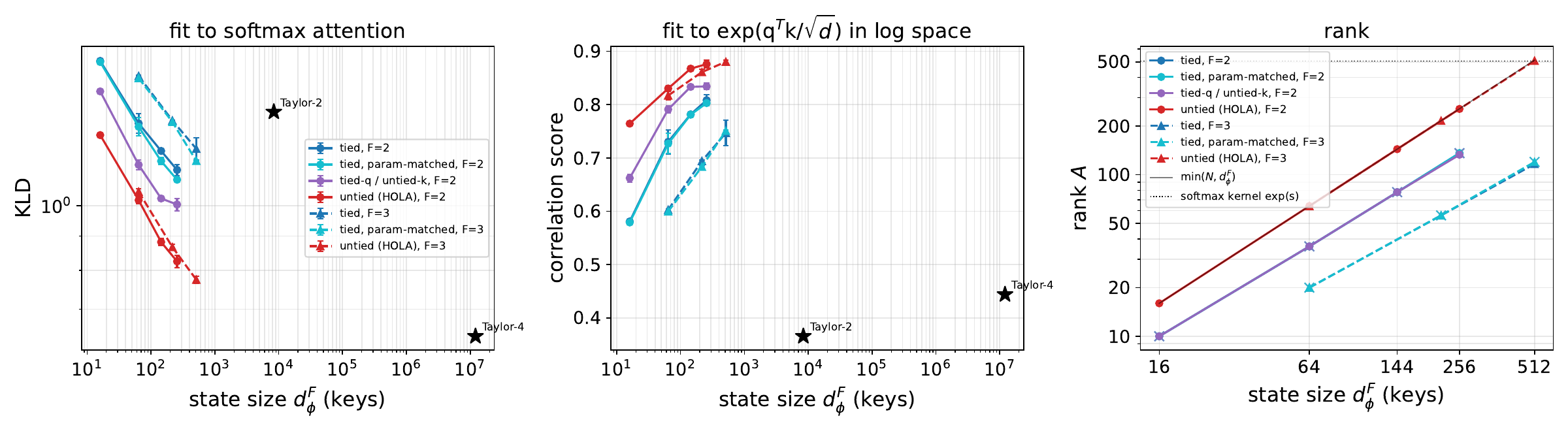}\\[-1mm]
  {\small (b) causal attention}
  \caption{Single synthetic head ($N=512$, $d=128$) trained to reproduce softmax attention, vs.\ state
  size $d_\phi^F$. Solid lines: $F=2$; dashed: $F=3$. Left: KLD to softmax (lower is better). Middle:
  correlation of $\log A$ with the true logits (higher is better). Right: numerical rank of $A$; crosses
  mark the tied ceiling $\binom{d_\phi+F-1}{F}$ (Prop.~\ref{prop:rank}), the solid black line
  $\min(N,d_\phi^F)$, the dotted line the rank of the softmax kernel. Stars: exact truncated Taylor series
  (no training). Error bars: std over 2 seeds. Details in App.~\ref{app:taylor_fit}.}
  \label{fig:taylor_fit}
  \vspace{-0.5cm}
\end{figure}

Proposition~\ref{prop:rank} bounds the rank that tied and untied lifts can reach, but it does not provide any information how well either can reproduce softmax attention in practice. We therefore ask how well each variant approximates full softmax attention at a given state size, isolating the attention mechanism from all other model components. To this end, we train only the feature maps of a single head on synthetic, realistically peaked queries and keys. The training objective is the Kullback-Leibler Divergence (KLD) between the row-normalised scores $A_{ij}=\prod_f\langle\phi_{q_f}(q_i),\phi_{k_f}(k_j)\rangle$ and the softmax attention distribution. For more details on the experimental setup, see the appendix~\ref{app:taylor_fit}. 
We compare four variants: First, \emph{tied} maps, i.e.\ one $\phi_q$ and one $\phi_k$ repeated in every factor as in Based and TaylorShift. Second,  \emph{tied, param-matched}, with an $F$-fold wider MLP so that it has as many parameters as the untied model. Third, \emph{tied-q\,/\,untied-k} and, last, \emph{untied} (HOLA). 
As an untrained reference, the stars \emph{Taylor-2} and \emph{Taylor-4} indicate variants that replace $\exp(s_{ij})$, with $s_{ij}=q_i^\top k_j/\sqrt d$, by its
Taylor polynomial of order $m=2$ or $4$, $\sum_{n\le m}s_{ij}^n/n!$. Their horizontal position is the
state needed to evaluate this polynomial exactly in linear time, i.e.\ $\sum_{n\le m}\binom{d+n-1}{n}$
entries ($8{,}385$ and ${\approx}12.1$M for $d=128$).

At every matched state size, for $F=2$ and $F=3$ and for both bidirectional and causal attention, untied maps fit softmax best, e.g.\ KLD $0.75$ vs.\ $1.23$ for tied at state 512. Matching parameters barely helps the tied model, so the gain stems from untying. 
The right plots of Fig.~\ref{fig:taylor_fit} confirm Prop.~\ref{prop:rank}: tied and tied-q models saturate at or just below the tied ceiling (crosses), whereas untied models follow $d_\phi^F$. 
Untying only the keys slightly improves the fit but does not lift the ceiling. 
The exact Taylor polynomials fit no better than the learned untied maps despite a far larger state:
Taylor-2 is clearly worse (KLD $1.71$ vs.\ $0.75$) at $16\times$ the state, and Taylor-4 is roughly on par ($0.82$ bidirectional, $0.64$ causal, vs.\ $0.75$ and $0.77$) at about $10^4\times$ the state, because truncation overweights keys with strongly negative logits (see sec.~\ref{app:subsec:taylor.approx} for more details). 
The feature maps of HOLA and HLA should therefore not be read as Taylor terms. Only the factorized order structure is inspired by the series.

\section{Conclusions}
We identified a limitation shared by existing higher-order linear attention methods: they repeat
the same query feature map in every factor of the Hadamard product. This confines the lifted
queries to a symmetric subspace and lowers the upper bound on the rank of the score matrix from
$d_\phi^F$ to $\binom{d_\phi+F-1}{F}$, a reduction by a factor approaching $F!$. Our proposed Higher-order
Linear Attention removes this limitation by untying the factors, at the same lift dimension
and essentially the same cost. It reduces to linear attention at order one, contains the tied
construction of prior work as a special case, and remains compatible with gating, decay and
delta-rule updates.

Empirically, untied factors fit softmax attention better than tied ones at equal state size and
parameter count, and their measured ranks exceed the tied ceiling, as predicted. Untying improved
associative retrieval more than raising the interaction order did. Integrated into a chunked Gated
Delta Network, it gave the best results on challenging multi-query associative retrieval (MQAR). 

Several limitations remain. Our rank result is an upper bound, and whether untied maps attain
$d_\phi^F$ was verified only empirically. At equal \emph{stored} memory, a tied model with
compressed state and a larger feature width could reach a comparable ceiling. The cost grows as
$d_\phi^F$, so the method pays off for long sequences and small orders; for short causal sequences, FlashAttention remains faster. 

\bibliographystyle{plainnat}
\bibliography{references}


\appendix
\clearpage

\section{Benchmark Results}

\begin{table*}[t]
  \centering
  \caption{Single-layer \emph{causal} attention on one NVIDIA H100 80GB (batch size 1, 12 heads,
  [precision]): asymmetric HLA-F2 (fused CUDA kernel, segmented, checkpointed backward) versus
  FlashAttention-2. Latencies are means over 20 repeats after 5 warm-up iterations (std
  $\le 6\%$). Backward is (fwd+bwd)$-$fwd. Speed-up is FA2 / HLA latency for fwd+bwd.
  Throughput is tokens per second through one layer (all heads). Peak memory is the additional
  allocated memory during fwd+bwd. $^\dagger$segment length 128 (all other $d_\phi{=}64$ rows:
  256). $^\ddagger$FA2 run with head dimension 64.}
  \label{tab:app-efficiency-causal}
  \setlength{\tabcolsep}{3.5pt}
  \resizebox{\textwidth}{!}{%
  \begin{tabular}{ccr cc cc cc c cc cc}
    \toprule
    & & & \multicolumn{2}{c}{Fwd (ms)} & \multicolumn{2}{c}{Bwd (ms)}
        & \multicolumn{2}{c}{Fwd+Bwd (ms)} & Speed-up
        & \multicolumn{2}{c}{Peak mem (GiB)} & \multicolumn{2}{c}{Train tput (M tok/s)} \\
    \cmidrule(lr){4-5}\cmidrule(lr){6-7}\cmidrule(lr){8-9}\cmidrule(lr){11-12}\cmidrule(lr){13-14}
    $d_\phi$ & $d_v$ & $N$ & HLA & FA2 & HLA & FA2 & HLA & FA2 & fwd+bwd & HLA & FA2 & HLA & FA2 \\
    \midrule
    16 & 128 & 16k & 2.40 & 2.35 & 10.71 & 6.00 & 13.11 & 8.35 & 0.64$\times$ & 1.64 & 0.28 & 1.22 & 1.92 \\
    16 & 128 & 32k & 4.64 & 8.33 & 20.76 & 23.08 & 25.40 & 31.41 & 1.24$\times$ & 3.25 & 0.55 & 1.26 & 1.02 \\
    16 & 128 & 64k & 9.11 & 32.90 & 40.88 & 94.40 & 49.99 & 127.30 & \textbf{2.55$\times$} & 6.45 & 1.11 & 1.28 & 0.50 \\
    \midrule
    16 & 64$^\ddagger$ & 32k & 4.13 & 5.25 & 15.87 & 12.65 & 20.00 & 17.90 & 0.90$\times$ & 2.06 & 0.28 & 1.60 & 1.79 \\
    16 & 64$^\ddagger$ & 64k & 7.97 & 19.69 & 30.79 & 52.15 & 38.76 & 71.84 & \textbf{1.85$\times$} & 4.09 & 0.56 & 1.65 & 0.89 \\
    \midrule
    64 & 128 & 16k$^\dagger$ & 21.17 & 2.33 & 123.53 & 6.01 & 144.70 & 8.34 & 0.06$\times$ & 23.6 & 0.28 & 0.11 & 1.92 \\
    64 & 128 & 32k & 31.68 & 8.40 & 225.43 & 23.10 & 257.11 & 31.50 & 0.12$\times$ & 32.5 & 0.55 & 0.12 & 1.02 \\
    64 & 128 & 64k & 63.34 & 32.72 & 450.54 & 95.41 & 513.88 & 128.13 & 0.25$\times$ & 65.0 & 1.10 & 0.12 & 0.50 \\
    \bottomrule
  \end{tabular}}
\end{table*}

\begin{table*}[t]
  \centering
  \caption{Single-layer \emph{non-causal} attention on one NVIDIA H100 80GB (batch size 1, 12 heads,
  $d_\phi{=}64$, $d_v{=}128$): asymmetric HLA-F2 versus FlashAttention-2. ``Ret.'' keeps
  the autograd graph; ``Ckpt.'' recomputes activations in the backward pass. Latencies are means over
  20 repeats after 5 warm-up iterations (std $\le 6\%$). Backward is (fwd+bwd)$-$fwd. Speed-ups are
  FA2 / HLA latency (HLA retained graph for fwd+bwd). Peak memory is the additional allocated memory
  during fwd+bwd; forward-only peak memory of HLA is 12.0\,GiB (32k) and 23.7\,GiB (64k). Throughput
  is tokens per second through one layer (all heads), fwd+bwd.}
  \label{tab:app-efficiency-noncausal}
  \setlength{\tabcolsep}{3.5pt}
  \resizebox{\textwidth}{!}{%
  \begin{tabular}{r cc ccc ccc cc ccc cc}
    \toprule
    & \multicolumn{2}{c}{Fwd (ms)} & \multicolumn{3}{c}{Bwd (ms)} & \multicolumn{3}{c}{Fwd+Bwd (ms)}
    & \multicolumn{2}{c}{Speed-up} & \multicolumn{3}{c}{Peak mem (GiB)} & \multicolumn{2}{c}{Train tput (M tok/s)} \\
    \cmidrule(lr){2-3}\cmidrule(lr){4-6}\cmidrule(lr){7-9}\cmidrule(lr){10-11}\cmidrule(lr){12-14}\cmidrule(lr){15-16}
    $N$ & HLA & FA2 & Ret. & Ckpt. & FA2 & Ret. & Ckpt. & FA2 & fwd & fwd+bwd & Ret. & Ckpt. & FA2 & HLA & FA2 \\
    \midrule
    32k & 17.16 & 15.75 & 164.27 & 284.74 & 49.59 & 181.43 & 301.90 & 65.34 & 0.92$\times$ & 0.36$\times$ & 25.1 & 59.8 & 0.55 & 0.18 & 0.49 \\
    64k & 34.31 & 64.82 & 328.32 & OOM & 196.12 & 362.63 & OOM & 260.94 & \textbf{1.89$\times$} & 0.72$\times$ & 50.2 & OOM & 1.11 & 0.18 & 0.25 \\
    \bottomrule
  \end{tabular}}
\end{table*}

HLA-F2 has length-independent throughput, so its advantage over FA2 grows linearly with $N$. At $d_\phi=64$ it is faster than FA2 in the non-causal forward pass from about 34k tokens (1.9× at 64k). The extrapolated training crossover is about 90k tokens (non-causal) and about 260k tokens (causal). At $d_\phi=16$ it trains faster than FA2 at 32k tokens (causal). The backward pass is currently not fused (7–10 times the forward cost).

\section{Datasets, Tasks, and Training Objectives}
\label{app:datasets}

Here, we describe the data and objectives used in the final long-context curriculum. 

\subsection{BookSum}
\label{app:datasets.booksum}

We derive long-context summarization examples from BookSum, a collection of literary texts paired with human-written summaries. Preprocessing first filters and deduplicates the source records and creates fixed train, development, and test splits. The second stage groups consecutive chapters from the same book into non-overlapping, length-budgeted windows. We use a curriculum with nominal context budgets of 8,192, 16,384, and 32,768 tokens. 
We arrive at $69960$ training samples of length 8k, $28150$ samples of length 16k and $10475$ samples of length 32k.




For a tokenized dialogue \(x_{1:T}\), the causal language-model loss is
\[
\mathcal{L}(x) =
-\frac{1}{|S|}\sum_{t\in S}
\log p_\theta(x_t\mid x_{<t}),
\]
where \(S\) contains only assistant-response positions. User tokens—including the instruction and source chapters—are visible as context but masked from the loss. For BookSum records, \(S\) is therefore the human summary. This prevents the much longer source text from dominating the objective. Summary cross-entropy should nevertheless be interpreted cautiously: a reference is one valid paraphrase among many, so its loss includes irreducible lexical and stylistic uncertainty. We retain it as the nominal task objective and a consistency metric, but use controlled retrieval evaluations to demonstrate the effectiveness of the proposed attention mechanism in long-context associations.

\subsection{Open-Vocabulary Associative Recall And Pose Continuation}
\label{app:datasets.tasks}

We augment every curriculum stage with synthetic instruction-response records. Here, synthetic refers to the task construction and arbitrary payloads, not to the background language: 
the generator uses random contiguous token spans from the  stream instruction and chapter text and 
and trims or concatenates it to the requested context budget. 
Training records use prose from the BookSum training split, whereas validation records use BookSum validation prose. It creates three record types: open-vocabulary multi-query associative recall (MQAR) records with several inserted key-value statements and a terminal query (see below); generic-copy records with one inserted arbitrary span to reproduce; and prose-continuation records with no inserted fact, for which the model predicts the remainder of the original passage. 
Each training set (8k, 16k or 32k) reserves 50\% of its records for ordinary prose continuation. The remaining 50\% is retrieval data, of which 25\% is a generic copy subtask and 75\% is MQAR. The resulting counts are listed in table~\ref{tab:synthetic_data_composition}.

\begin{table}[t]
    \centering
    \caption{Composition of the synthetic training data at each stage.}
    \label{tab:synthetic_data_composition}
    \small
    \begin{tabular}{lrrrr}
        \toprule
        Stage & Total & MQAR retrieval & Generic copy & Prose continuation \\
        \midrule
        8k & \textbf{60,000} & 22,500 & 7,500 & 30,000 \\
        16k & \textbf{30,000} & 11,250 & 3,750 & 15,000 \\
        32k & \textbf{18,000} & 6,750 & 2,250 & 9,000 \\
        \bottomrule
    \end{tabular}
\end{table}



Combining the data from the summarization and the MQAR task, we arrive at $129960$, $58150$, and $28475$ samples per curriculum epoch at 8k, 16k, and 32k context length.

\subsection{MQAR And Copy Tasks}
\label{app:datasets.ovar}

For creating our MQAR data, we modify modify multi query associative recall (MQAR)~\citep{arora2024based}. 

We insert randomly generated key-value associations into the BookSum chapter text used as haystack. 
Each association is rendered as a short carrier sentence and inserted at a sampled depth among the original sentences of the long passage. A query is appended after the completed haystack, asking the model to reproduce the value associated with its key. The final recipe uses at most eight planted pairs, one query, and no additional distractors; the other planted pairs already serve as competing associations. The assistant answer is short (approximately 12–15 tokens), unique, and fully supervised, while the long prompt is masked. A low loss therefore requires locating the queried association over a long distance rather than merely modeling local prose.

\subsection{Examples And Scale Of The Open-Vocabulary Codes}

The value space is deliberately open and extremely large. For the first versions of this dataset, we used a small closed label inventory and later disjoint character inventories for training and evaluation. However, both versions admitted shortcuts. The final generator instead draws train and test values from the same character and subword inventory but from a combinatorial space large enough that accidental repetition is negligible. Consequently, token identity does not reveal the split, and memorizing a small answer vocabulary cannot solve the task. The generic copy subtask plants a single arbitrary value and requests it directly. It removes key-binding difficulty while preserving exact long-range copying, helping to distinguish failures of span copying from failures of associative binding.

The payload generator mixes several key generators rather than drawing answers from a small, fixed set. The measured training-side prefix space contains 24 color strings, 26 ordinary-word strings, 20 rare-word strings, 168,400 invented-word strings, approximately \(5.6\times10^9\) random-string strings, and 162,240 mixed-case words. A numeric suffix is then added, producing keys such as:
\begin{itemize}
    \item bI8\_talen-49755    
    \item sG8\_brix-873
    \item TIMBER-645
    \item cresp\_Je1-99651    
    \item INGLENOOK-55       
    \item THOMOLDUNSAR-78334
    \item PNNBMKH-847        
\end{itemize}

Here, a \emph{value} is the complete answer code, consisting of a generated prefix plus a numeric suffix—for example, the value `TIMBER-645` has prefix `TIMBER` and suffix `645`. Multiple values can therefore share a prefix while remaining different because their numeric suffixes differ (for example, `TIMBER-645` and `TIMBER-2719`). 
In a sample of $20000$ complete values, $19867$ were unique ($99.33$\%). Removing the numeric suffixes left $10852$  distinct prefixes; $93.1$\% of those distinct prefixes appeared in only one of the $20000$ draws, while the remainder were reused with different suffixes or, rarely, formed duplicate complete values. Likewise, drawing $100$ keys produced $100$ distinct keys.
The measured overlap between complete train and test values was zero. These statistics make memorization of a compact answer inventory ineffective and force the model to copy the particular complete value embedded in the current context.

Table~\ref{tab:carrier_queries} provides examples of carrier phrases with embedded key-value pairs and possible query phrases. As explained above, for the MQAR task, we plant up to 8 such phrases into each sequence but query only one of them. 
\begin{table}[t]
    \centering
    \caption{Examples of embedded carrier phrases and corresponding queries.}
    \label{tab:carrier_queries}
    \begin{tabular}{p{0.47\linewidth} p{0.47\linewidth}}
    \toprule
    \textbf{Embedded carrier phrase} & \textbf{Query} \\
    \midrule
    Remember: the RAMPART code is bI8\_talen-49755. & What is the value of the RAMPART code? \\
    Remember: the LKKGF code is wN5\_brix-881. & What is the value of the LKKGF code? \\
    Important: LANTERN maps to TIMBER-645. & Which code value belongs to LANTERN? \\
    Important: DXRCBDJKV maps to YONTHOBREL-55082. & Which code value belongs to DXRCBDJKV? \\
    For reference, the HARBOR code equals DEFENESTRATE-681. & State the value of the HARBOR code. \\
    For reference, the JZWRKSXLB code equals SWFQD-5803. & State the value of the JZWRKSXLB code. \\
    \bottomrule
    \end{tabular}
\end{table}

\subsection{Raw-Prose Next-Token Prediction}

Purely synthetic retrieval training can preserve the puzzle-solving behavior while degrading general language modeling. We therefore devote half of each synthetic file to prose examples. A passage from the corresponding BookSum split (training prose for training files and development prose for validation files) is divided approximately 1:7: the first portion appears in the user prompt (\emph{“Continue the following passage”}), and the remaining prose is the assistant target. Unlike summarization, the continuation has a unique observed target and places loss on thousands of natural-language tokens. It acts as an anti-forgetting regularizer and a diagnostic: rising held-out loss on prose continuation indicates degradation of ordinary language modeling even if retrieval loss improves. The generator draws filler from a large pool (3,000 source records in the final recipe) to avoid the prose reuse. 


All records use the same token-level cross-entropy implementation, but per-example losses are assigned task-dependent multipliers. In the active configuration, the coarse BookSum summary group has weight $0.25$, MQAR/copy has weight $0.75$, and prose continuation has a weight of $1.5$. 

\section{Approximating Softmax Attention}
\label{app:taylor_fit}

We train only the feature maps of a single attention head to reproduce softmax attention. Queries and keys
are $q=xW_q+b_q$ and $k=xW_k$ for random tokens $x\sim\mathcal N(0,I)$, with head dimension $d=128$, RoPE,
partially aligned $W_q,W_k$ and a few query outlier channels; logits are scaled so that
$\mathrm{std}(q^\top k/\sqrt d)=3$, which makes the target about as peaked as a real head. Each model
row-normalises $A_{ij}=\prod_{f}\langle\phi_{q_f}(q_i),\phi_{k_f}(k_j)\rangle$ into $\hat P_{ij}$ and is
trained for 2000 Adam steps with a batch size of 8 to minimise the KLD $\frac1N\sum_i\sum_j P_{ij}\log(P_{ij}/\hat P_{ij})$
to the softmax target $P_{ij}=\mathrm{softmax}_j(q_i^\top k_j/\sqrt d)$. Metrics are measured on 16 held-out
sequences. The correlation score is the Pearson correlation between $\log A_{ij}$ and $q_i^\top k_j/\sqrt d$
after subtracting each row's mean, which removes the per-row normalisation constant; a score of 1 means
$A\propto\exp(q^\top k/\sqrt d)$ row by row. We compare \emph{tied} maps (one $\phi_q$, one $\phi_k$ repeated in
every factor, as in Based and TaylorShift), \emph{tied, param-matched} (the same, but with the MLP hidden width
enlarged $F$-fold so that its parameter count equals the untied model), \emph{tied-q\,/\,untied-k} (shared
query map, separate key maps per factor), and \emph{untied} maps (HOLA). Taylor-2 and Taylor-4 are the exact
truncated series $\sum_{n\le m}s^n/n!$ of the softmax kernel with identity maps, placed at the state an exact
lift needs ($8{,}385$ and ${\approx}12.1$M entries, using symmetric compression).

At every matched state size and for both $F=2$ and $F=3$, untied maps fit softmax attention best, e.g.\ KLD
$0.80$ vs.\ $1.16$ for tied at state 256 ($F=2$) and $0.75$ vs.\ $1.23$ at state 512 ($F=3$). The causal
setting behaves the same ($0.83$ vs.\ $1.13$ and $0.77$ vs.\ $1.22$). Matching parameters barely helps the tied
model (KLD $1.12$ vs.\ $1.16$), so the gain comes from untying, not from parameter count. The measured ranks
confirm Proposition~\ref{prop:rank}: tied and tied-q\,/\,untied-k models sit at or just below
$\binom{d_\phi+F-1}{F}$, whereas untied models reach $d_\phi^F$; untying only the keys therefore improves the
fit somewhat 
but does not lift the rank ceiling. The truncated Taylor series is a poor model of
softmax at realistic logit scales: $1+s+s^2/2$ grows again for $s<-1$, so suppressed keys (see app~\ref{app:subsec:taylor.approx}) receive substantial
weight (KLD $1.71$, correlation $0.36$ despite $8{,}385$ state entries). Even Taylor-4, with roughly
$20{,}000\times$ the state of the best untied model, is still worse in the bidirectional case (KLD $0.82$ vs.\
$0.75$) and only moderately better in the causal case ($0.64$ vs.\ $0.77$). The Taylor polynomials do recover the
arg-max more often (top-1 agreement $0.86$ vs.\ ${\le}0.47$) but misallocate the remaining mass. The learned
feature maps are thus not Taylor terms; only the order structure is inspired by the series.

\subsection{Taylor Approximations}
\label{app:subsec:taylor.approx}

\begin{figure}[h]
  \centering
  \includegraphics[width=0.6\linewidth]{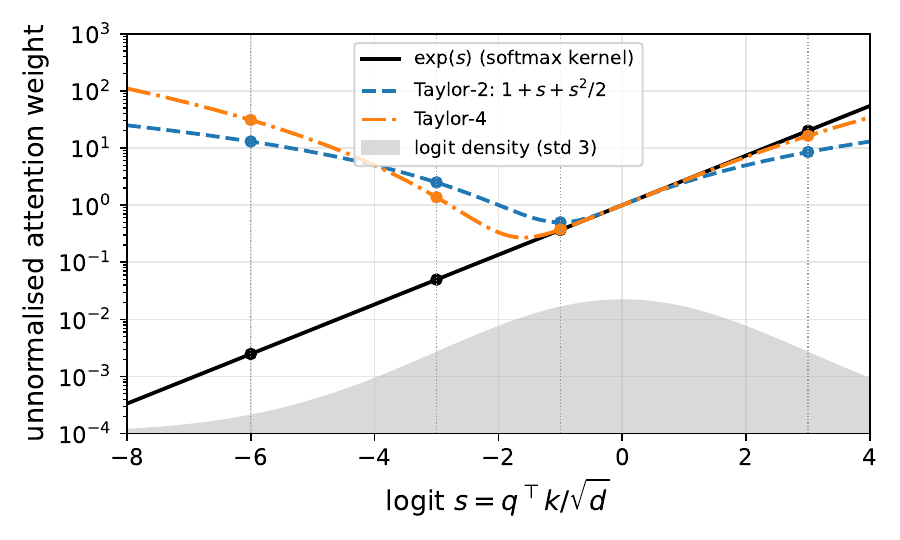}
  \caption{Unnormalised attention weight as a function of the logit $s=q^\top k/\sqrt d$ for the softmax
  kernel $\exp(s)$ and its Taylor polynomials of order 2 and 4. The truncations are not monotone: they
  grow again for negative $s$. At $s=-6$, softmax assigns about $8{,}000\times$ less weight than at
  $s=+3$, whereas Taylor-2 and Taylor-4 assign $1.5\times$ and $1.9\times$ \emph{more}
  ($13.0$ vs.\ $8.5$ and $31.0$ vs.\ $16.4$). The grey area shows the distribution of logits in the
  experiment of Fig.~\ref{fig:taylor_fit} (std $3$), where such values are common.}
  \label{app:taylor_vs_exp}
\end{figure}

Softmax assigns a key with a strongly negative logit almost no weight, since $\exp(s)\to0$ as
$s\to-\infty$. A truncated Taylor polynomial does the opposite: it is not monotone and grows again for
strongly negative $s$, so keys that softmax suppresses receive large weight (App.~\ref{app:taylor_vs_exp}).
Adding terms only widens the range of logits where the approximation holds, which is why
even Taylor-4 fits poorly at realistic logit scales.

\section{Video Generation Results}

\subsection{Model Variants}
We integrate the proposed attention mechanism into the Wan2.1-1.3B video diffusion model~\cite{wan2025wanopenadvancedlargescale}. We retain the native output configuration of the base model, corresponding to $81$ frames at a spatial resolution of $480 \times 832$ pixels.

For the HLA-F2 variant, we replace softmax attention with non-causal HLA-F2 in several of the model's $30$ attention layers. The remaining layers retain softmax attention. These layers are selected based on an analysis of pre-softmax attention-score magnitudes: we preserve softmax attention in the nine layers that exhibit the most pronounced extreme scores across a set of representative inputs. More information about the different models can be found in the appendix.

We computed maximum attention scores for each self-attention block of the teacher model averaged over several batches. For \textbf{HLA-F2-21-480p}, we selected the 21 blocks (1,2,3,4,5,6,7,8, 11,12,13,14,15,16,17,18, 22,23,24,25,26) with the smallest scores. For \textbf{HLA-F3-15-240p}, we applied HLA to blocks 1,2,3,4,5,6,7,8, 11,12,13,14,15,16,17. For \textbf{HLA-F3-10-480p}, we only use HLA in the odd layers (1,3,5,7, 11,13,15,17, 23,25). We apply two small MLPs (see~\ref{app:model-cards}) to the queries and attention outputs before element-wise multiplication.


\paragraph{Training data.}
We fine-tune the video-generation models on a subset of approximately $350{,}000$ videos from the data used by OpenSoraPlan~\cite{lin2024opensoraplanopensourcelarge}. We augment this dataset with approximately $100{,}000$ synthetic videos generated using Wan2.1-14B. 

\subsubsection{Evaluation Protocol}
\label{sec:video-evaluation}

We evaluate video-generation quality using VBench~\cite{huang2023vbench}. We report the aggregate VBench score and its component metrics to distinguish changes in perceptual quality, temporal consistency, motion quality, and text-video alignment.



\paragraph{Training procedure.}
Training consists of two phases. In the first phase, we freeze the pretrained backbone and train only the newly introduced HLA-F2 or linear attention modules for $10{,}000$ optimization steps. In the second phase, we unfreeze the complete model and continue training all parameters. See the appendix for information how many training steps was used for each model. 

We use a learning rate of $1 \times 10^{-4}$ in all phases and train on eight NVIDIA H100 GPUs with a global batch size of $32$. We use $1{,}000$ warmup steps in the full-model fine-tuning phase. 

\subsubsection{Main Results}
\label{sec:video-results}

Table~\ref{table:main.results.video} reports video quality for the original Wan2.1-1.3B model, several prior works of comparable parameter count, a variant that replaced softmax attention by linear attention at 15 of its  layers~\cite{ghafoorian2025attentionsurgeryefficientrecipe}, and the HLA-F3-10-480p variant. HLA-F2 achieves a VBench score of $81.4$ compared with 78.9 for linear attention and 83.3 for the Wan baseline. At the same time, HLA-F3-10-480p requires 218 TFLOPs, corresponding to a reduction of almost 23\% relative to the original model that requires 283 TFLOPs for a single forward pass through the model. Model HLA-F3-21-480p that replaces softmax attention in 21 out of 30 layers still achieves a VBench score almost 2 points higher than the linear attention variant that replaces softmax-attention in only 15 of its layers, see table~\ref{table:times}.

The component-wise VBench results indicate that the proposed hybrid models reduce the necessary compute under a small impact to the quality of the generated videos. 
These findings show that higher-order attention is not restricted to autoregressive language modeling. Its successful integration into a video diffusion transformer indicates that the mechanism can operate under substantially different input structures, objectives, and generation procedures.

\begin{table}[t!]
\vspace{0pt}
\centering
\begin{tabular}{lccc}
\toprule
 & \textbf{Total↑} & \textbf{Quality↑} & \textbf{Semantic↑} \\
\midrule
\multicolumn{4}{l}{\textbf{Models with at most 2B parameters}} \\
\cite{peng2025opensora20trainingcommerciallevel} & 79.76 & 81.35 & 73.39 \\
\cite{wu2025snapgenvgeneratingfivesecondvideo} & 81.14 & 83.47 & 71.84 \\
\cite{hacohen2024ltxvideorealtimevideolatent} & 80.00 & 82.30 & 70.79 \\
\cite{yang2025cogvideoxtexttovideodiffusionmodels} & 81.55 & 82.48 & 77.81 \\
\cite{isobe2025amdhummingbirdefficienttexttovideomodel} & 81.35 & 83.73 & 71.84 \\
\cite{huang2025m4vmultimodalmambatexttovideo} & 81.91 & 83.36 & 76.10 \\
\cite{jin2025pyramidalflowmatchingefficient} & 81.72 & 84.74 & 69.62 \\
\cite{wan2025wanopenadvancedlargescale} & 83.31 & 85.23 & 75.65 \\
Wan2.1-1.3B + linear attention, \cite{ghafoorian2025attentionsurgeryefficientrecipe} & 78.9 & 82.2 & 65.9 \\
\bottomrule
Wan2.1-1.3B-HLA & \multicolumn{3}{l}{} \\
\multicolumn{2}{l}{$\sim23\%$ less total compute than Wan2.1-1.3B} &&  \\
{\textbf{HLA-F3-10-480p}} & 81.42 & 83.22 & 74.20 \\
\bottomrule
\end{tabular}
\caption{Comparison on VBench with other models for video diffusion. We report the numbers published in the respective papers. While the proposed method has slightly lower scores, it about $23$\% less TFLOPs than Wan2.1-1.3B. The results for linear attention stem from \cite{ghafoorian2025attentionsurgeryefficientrecipe} who replaced self-attention in 15 transformer blocks by standard linear attention in Wan2.1-1.3B. 
}
\label{table:main.results.video}
\vspace{-0.3cm}
\end{table}

\begin{table}[h]
\centering
\begin{tabular}{ccccc}
\toprule
\textbf{Method} & \textbf{Tokens} & \textbf{Score} & \textbf{TFLOPs} & \textbf{Time} \\
\toprule
Wan2.1 (quad.) & 32760 & 83.3 & 283.03 & 1:36 \\
Wan2.1 (quad.) & 12600 & 81.0 &  62.13 & 0:45 \\
\midrule
\textbf{HLA-F3-10-480p} & 32760 & 81.42 & 218.59 & 1:26 \\
\textbf{HLA-F3-15-240p} & 12600 & 79.78 &  48.37 & 0:26 \\
\textbf{HLA-F3-21-480p} & 32760 & 80.54 & 147.71 & 1:16 \\
\toprule
\end{tabular}
\caption{VBench scores and times to generate a single video of several model variants. TFLOPs are for a single forward pass through the model. All measurement were done on an H100 GPU.}
\label{table:times}
\vspace{-0.4cm}
\end{table}

Qualitative results are shown in fig.~\ref{fig:teaser_grid}. See the appendix for the complete optimization setting and more qualitative and also quantitative results.

\begin{figure*}[ht]
    \centering
    \includegraphics[width=14.cm,height=2.414cm]{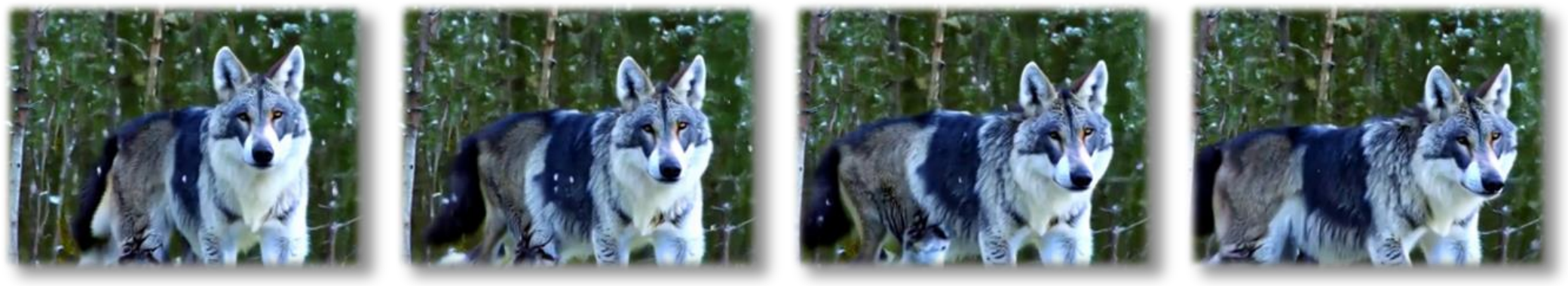}
    \includegraphics[width=14.cm,height=2.414cm]{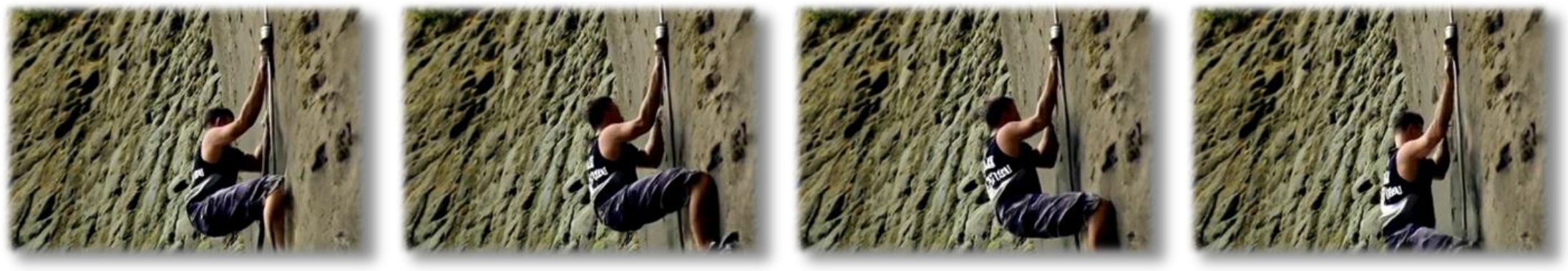}
    \includegraphics[width=14.cm,height=2.414cm]{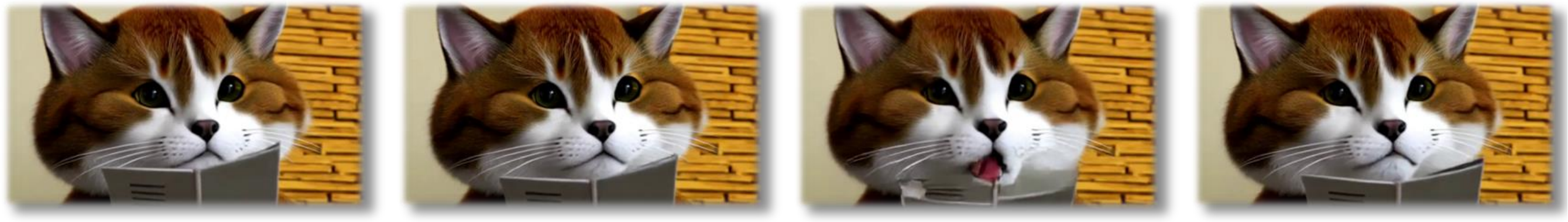}
    \captionof{figure}{Results generated using the proposed video diffusion transformer with $10$ layers HLA-F2, see the appendix for more details.}
    \label{fig:teaser_grid}
\end{figure*}

\section{Pseudocode}
Algorithm~1 provides pseudocode of the entire attention procedure. PyTorch-like code for the implementation of the subroutines to compute two or three factor HLA can be found in figures \ref{lst:hla-2f} and \ref{lst:hla-3f}.

\begin{algorithm}[t]
\caption{Hadamard Linear Attention}
\label{algo:hla}
\begin{algorithmic}

\KwIn{Input tokens $X$}
\KwOut{Attention output}

\BlankLine
\STATE \textbf{Step 1: Generate queries, keys, values}
\STATE $q, k, v \gets \texttt{to\_q}(X), \texttt{to\_k}(X), \texttt{to\_v}(X)$
\STATE \# reshape tensors to allow for multi-head attention

\BlankLine
\STATE \textbf{Step 2: Apply rotational embedding}
\STATE $q, k \gets \texttt{apply\_rope}(q), \texttt{apply\_rope}(k)$

\BlankLine
\STATE \textbf{Step 3: Apply scale factor}
\STATE $q \gets q \cdot \texttt{scale\_factor}$

\BlankLine
\STATE \textbf{Step 4: HLA-specific transformations}
\STATE $q \gets \texttt{hla\_to\_q}(q)$
\STATE $k_1, k_2 \gets \texttt{hla\_to\_k1}(k), \texttt{hla\_to\_k2}(k)$
\If{\texttt{use\_hla\_3\_factors}}
\STATE $k_3 \gets \texttt{hla\_to\_k3}(k)$
\EndIf

\BlankLine
\STATE \textbf{Step 5: Compute HLA}
\If{\texttt{use\_hla\_2\_factors}}
\STATE \# by algorithm in Fig.~4 (supplementary)
\STATE \texttt{attn} $\gets \texttt{2\_factor\_hla}(q, k_1, k_2, v)$
\EndIf
\If{\texttt{use\_hla\_3\_factors}}
\STATE \# by algorithm in Fig.~5 (supplementary)
\STATE \texttt{attn} $\gets \texttt{3\_factor\_hla}(q, k_1, k_2, k_3, v)$
\EndIf

\BlankLine
\Return \texttt{to\_out(attn)}

\end{algorithmic}
\end{algorithm}

\begin{figure}[t]
  \includegraphics[width=\linewidth]{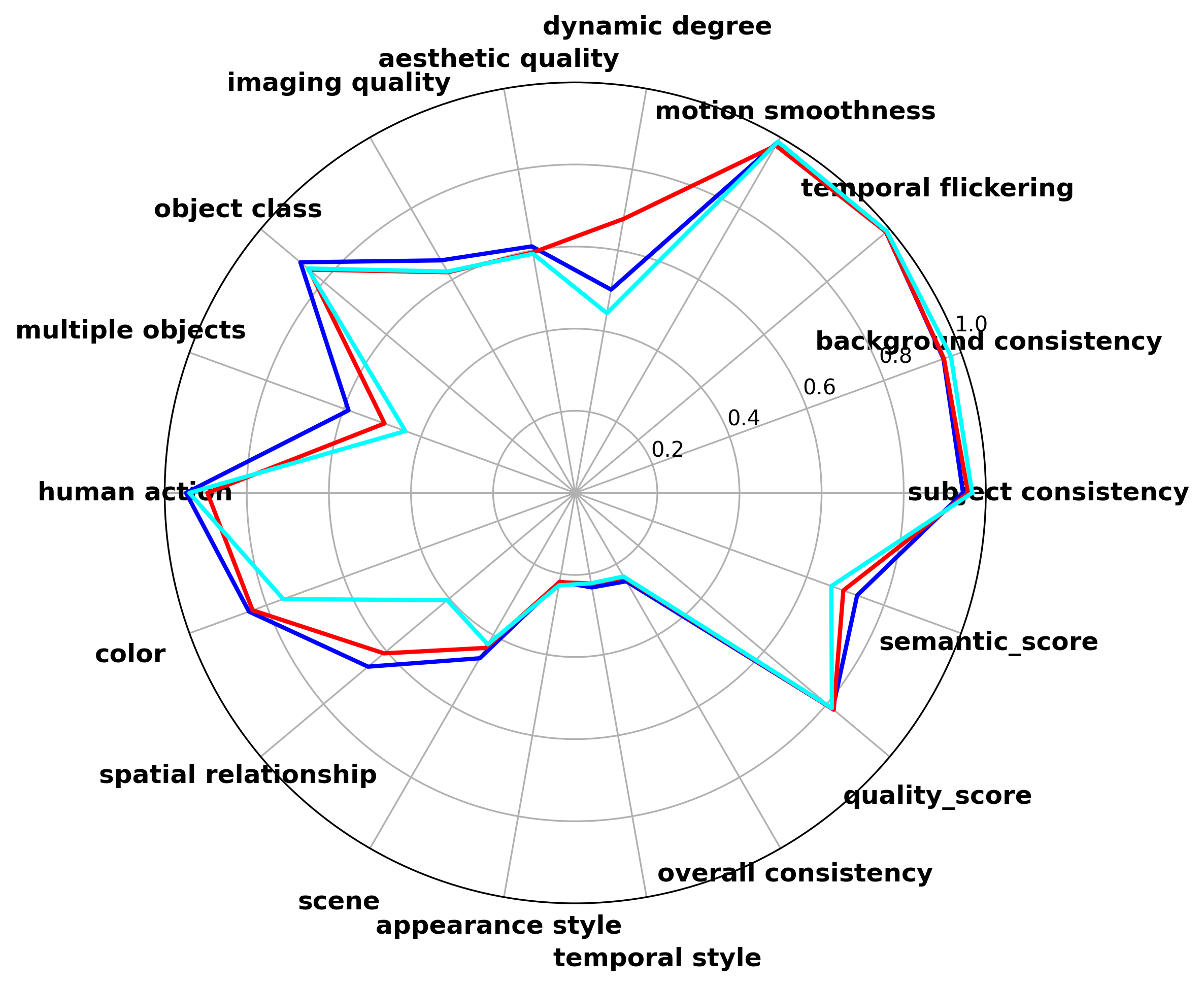}
  \caption{VBench scores by categories. Dark blue line corresponds to \textbf{HLA-3F-R1-10}, red to \textbf{HLA-3F-R1-21} and light blue to \textbf{HLA-3F-R2-15}.}
  \label{fig:vbench.category}
\end{figure}

\begin{figure}[t]
  \begin{center}
      \includegraphics[width=0.75\linewidth, height=7.5cm]{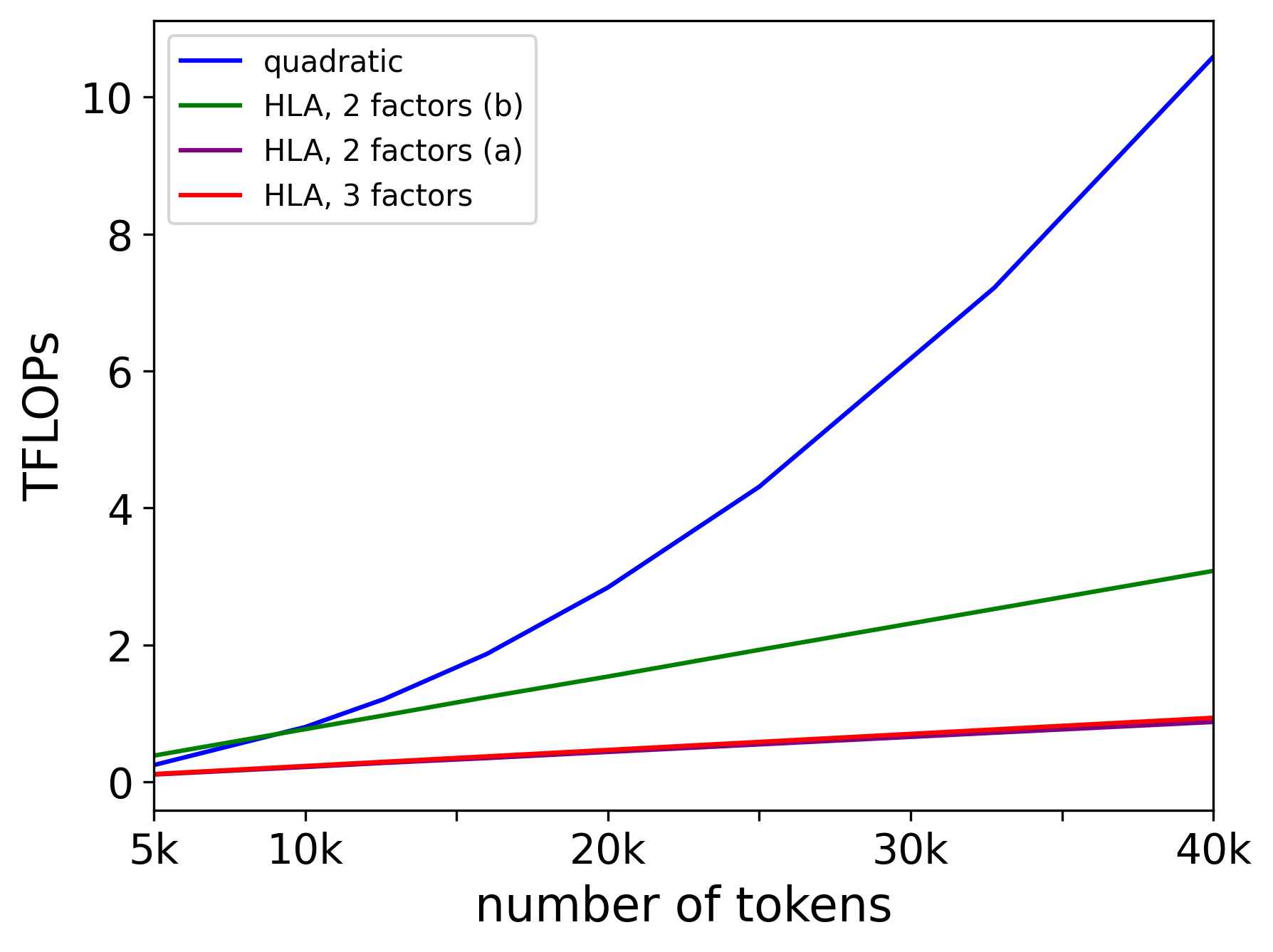}
      \caption{Comparison of computational complexities for different sequence lengths between standard quadratic attention and 2 or 3-factor HLA. HLA 2 factors (a) uses small MLPs with identical parameters as HLA 3 factors whereas HLA 2 factors (b) uses larger MLPs. As can be seen, the curves for HLA 3 factors and HLA 2 factors (a) almost coincide. This implies that in the regime of large number of tokens and between 2 and 3 Hadamard factors, the size of the MLPs is decisive.}
      \label{fig:flops}
  \end{center}
\end{figure}



\section{AI Policy}

\section{Ethics Statement}
This work investigates architectural and training mechanisms for improving sequence modeling and associative recall. Our experiments primarily rely on established public data. For multi-query associate recall, we use synthetic data as established by prior works. Our experiments do not involve human subjects or personally identifiable information. The proposed methods are intended to advance the understanding of sequence models rather than to target a particular downstream application. As with general-purpose improvements to machine learning models, the techniques introduced here could potentially be incorporated into systems deployed in applications with societal impact. Such deployment may inherit risks associated with the underlying models and training data, including biases, unreliable predictions, and inappropriate use. These risks are not specific to the methods proposed in this work. We encourage appropriate evaluation and safeguards when applying the methods in high-impact settings.

\section{Model Cards} 
\label{app:model-cards} 

All models are derived from \texttt{Qwen2.5-1.5B-Instruct}. Unless stated otherwise, we preserve the pretrained query, key, value, and output projections, the rotary position embeddings, and the remaining transformer components. We replace only the self-attention module in every transformer layer. Newly introduced feature maps and gating projections are randomly initialized and trained during adaptation. 

\paragraph{MLP Feature Map} Several models use a learned per-head feature map $\phi:\mathbb{R}^{d_h}\rightarrow\mathbb{R}^{d_\phi}$ of the form 
\begin{equation} \phi(\mathbf{x}) = \operatorname{softplus}\!\left( \mathbf{W}_2 \operatorname{GELU}\!\left(\mathbf{W}_1\mathbf{x}+\mathbf{b}_1\right) + \mathbf{b}_2 \right), 
\label{eq:mlp-feature-map} 
\end{equation} 
where $\mathbf{W}_1\in\mathbb{R}^{d_{\mathrm{lat}}\times d_h}$ and $\mathbf{W}_2\in\mathbb{R}^{d_\phi\times d_{\mathrm{lat}}}$. The softplus output activation ensures that the resulting linear-attention features are strictly positive. Both weight matrices are initialized uniformly from $\mathcal{U}(-0.1,0.1)$, and both biases are initialized to zero. No input layer normalization or intermediate output ReLU is used in the configurations below. 

\paragraph{Softmax Attention Baseline} The baseline is the unmodified self-attention mechanism of \texttt{Qwen2.5-1.5B-Instruct}. It uses the pretrained query, key, value, and output projections together with rotary position embeddings and grouped-query attention. This model serves as the dense quadratic-attention reference. 

\paragraph{Symmetric HLA-F2} Hadamard Linear Attention with two factors, denoted HLA-F2, replaces softmax attention with a causal factorized linear-attention mechanism. After the pretrained query and key projections and rotary position embeddings, we compute 
\begin{equation} \mathbf{Q}=\phi_q(\mathbf{q}),\qquad \mathbf{K}_1=\phi_{k_1}(\mathbf{k}),\qquad \mathbf{K}_2=\phi_{k_2}(\mathbf{k}), 
\end{equation} 
using three independently parameterized MLP feature maps. The corresponding second-order query feature is symmetric, whereas the key feature is represented by the Hadamard-factorized pair $(\mathbf{K}_1,\mathbf{K}_2)$. The causal HLA-F2 operator is evaluated using a custom CUDA kernel. 
The difference to the operator proposed in \cite{arora2024based,nauen2024taylorshift} is that we use untied featuer maps for $\mathbf{K}_1$ and $\mathbf{K}_2$.

\paragraph{Asymmetric HLA-F2} The asymmetric HLA-F2 variant learns two independent query factors and two independent key factors: 
\begin{equation} \mathbf{Q}_1=\phi_{q_1}(\mathbf{q}),\qquad \mathbf{Q}_2=\phi_{q_2}(\mathbf{q}),\qquad \mathbf{K}_1=\phi_{k_1}(\mathbf{k}),\qquad \mathbf{K}_2=\phi_{k_2}(\mathbf{k}). 
\end{equation} 
Thus, both the query and key addressing features are represented by factorized second-order feature maps. Unlike symmetric HLA-F2, the two query factors are parameterized independently. We use $d_{\mathrm{lat}}=128$ and $d_\phi=64$ for each factor. 

\paragraph{HOLA-F3}
This model combines first-, second-, and third-order causal linear-attention operators. The first-order query and key features are obtained using full-width projections over the complete hidden representation:
\begin{equation}
\mathbf{Q}^{(1)} = \operatorname{ReLU}\!\left(\mathbf{W}^{(1)}_q\mathbf{q}\right), \qquad \mathbf{K}^{(1)} = \operatorname{ReLU}\!\left(\mathbf{W}^{(1)}_k\mathbf{k}\right).
\end{equation}
The first-order components are subtracted from the per-head query and key representations. Independent MLP feature maps are then applied to the resulting residuals to construct symmetric second- and third-order features. The second-order feature maps have width $d_\phi$, while the third-order feature maps have width $d_\phi/2$. The outputs $\mathbf{U}^{(1)}$, $\mathbf{U}^{(2)}$, and $\mathbf{U}^{(3)}$ of the three causal attention operators are combined using input-dependent mixture weights:
\begin{equation}
\boldsymbol{\gamma} = \operatorname{softmax}\!\left( g\!\left( \frac{1}{N}\sum_{t=1}^{N}\mathbf{q}_t \right) \right), \qquad \mathbf{U} = \sum_{r=1}^{3}\gamma_r\mathbf{U}^{(r)},
\label{eq:hola.f3.gate}
\end{equation}
where $g$ is a per-head two-layer MLP producing three gate logits. We use $d_{\mathrm{lat}}=128$ and $d_\phi=16$, giving a third-order factor width of $8$.

The gate in \eqref{eq:hola.f3.gate} pools $\mathbf{q}_t$ over the entire sequence. In causal settings, the output at position $t$ therefore depends on future tokens $s>t$. The HOLA-F3 results reported in this paper were obtained with this (leaky) gate. A causal variant replaces the global mean with a prefix mean, yielding position-dependent mixture weights:
\begin{equation}
\bar{\mathbf{q}}_t = \frac{1}{t}\sum_{s=1}^{t}\mathbf{q}_s, \qquad
\boldsymbol{\gamma}_t = \operatorname{softmax}\!\left( g\!\left( \bar{\mathbf{q}}_t \right) \right), \qquad
\mathbf{U}_t = \sum_{r=1}^{3}\gamma_{t,r}\,\mathbf{U}^{(r)}_t.
\label{eq:hola.f3.gate.causal}
\end{equation}
Since $\bar{\mathbf{q}}_t$ depends only on positions $s \le t$, \eqref{eq:hola.f3.gate.causal} preserves causality. All prefix means are obtained with a single cumulative sum, $\bar{\mathbf{q}}_t = \mathbf{S}_t/t$ with $\mathbf{S}_t = \mathbf{S}_{t-1} + \mathbf{q}_t$, at $\mathcal{O}(N)$ cost. During recurrent inference, this requires one additional state vector per head. For non-causal (bidirectional) tasks, the global-mean gate in \eqref{eq:hola.f3.gate} is valid.

\paragraph{Sink Sliding-Window Attention} Sink sliding-window attention~\cite{jolicoeurmartineau2026SWA} retains exact softmax attention but restricts the causal attention pattern. Each query attends to the first $N_{\mathrm{sink}}$ tokens and to the most recent $W$ tokens: 
\begin{equation} 
\mathcal{A}(t) = \{0,\ldots,N_{\mathrm{sink}}-1\} \cup \{\max(0,t-W+1),\ldots,t\}. 
\end{equation} 
We use a sliding-window size of $W=512$ and $N_{\mathrm{sink}}=4$ sink tokens. All pretrained attention projections are preserved. For this model, we apply LoRA with rank $16$, scaling parameter $32$, and dropout $0.05$ to the attention projections (\texttt{q\_proj}, \texttt{k\_proj}, \texttt{v\_proj}, and \texttt{o\_proj}) and the feed-forward projections (\texttt{gate\_proj}, \texttt{up\_proj}, and \texttt{down\_proj}). We train LoRA biases only and do not use rank-stabilized LoRA or DoRA. 

\paragraph{Chunked Gated Delta Net} Chunked Gated DeltaNet replaces softmax attention with a linear-recurrent delta-rule update. It preserves the pretrained query, key, value, and output projections and introduces two projections that produce one decay gate and one write-strength gate per attention head and token: 
\begin{equation} 
\alpha_t = \sigma\!\left(\mathbf{W}_{\alpha}\mathbf{h}_t+\mathbf{b}_{\alpha}\right), \qquad \beta_t = \sigma\!\left(\mathbf{W}_{\beta}\mathbf{h}_t+\mathbf{b}_{\beta}\right). 
\end{equation} 
Here, $\alpha_t$ controls retention of the recurrent state, while $\beta_t$ controls the strength of the delta-rule update. Computation is organized into chunks of $512$ tokens. The gate biases are initialized to $b_{\alpha}=4.0$ and $b_{\beta}=0.0$, corresponding to initial gate values $\sigma(4.0)\approx0.982$ and $\sigma(0)=0.5$, respectively. 

\paragraph{Asymmetric HLA-F2 Gated Delta Net} This model combines the chunked gated delta rule with the addressing space of asymmetric HLA-F2. Instead of using first-order query and key features, it constructs two independent positive factors for each: 
\begin{equation} 
\boldsymbol{\phi}_q(\mathbf{q}) = \phi_{q_1}(\mathbf{q})\otimes\phi_{q_2}(\mathbf{q}), \qquad \boldsymbol{\phi}_k(\mathbf{k}) = \phi_{k_1}(\mathbf{k})\otimes\phi_{k_2}(\mathbf{k}), 
\end{equation} 
where $\otimes$ denotes the factorized outer-product feature representation. These features are used for addressing within the gated delta-rule recurrence. The model otherwise uses the same per-token decay and write-strength gates as Chunked Gated Delta Net. We use a chunk size of $512$, $b_{\alpha}=4.0$, $b_{\beta}=0.0$, $d_{\mathrm{lat}}=128$, and $d_\phi=64$ for each Hadamard factor.

\section{Use of LLMs} 
LLM tools were used for software-development support, including the creation and optimization of CUDA kernels, the development of theoretical and conceptual frameworks, the formulation of mathematical claims, hypothesis generation and refinement, feedback on research methodology and experimental design and drafting, revising, and improving the readability of the manuscript. 

All AI-generated outputs were reviewed and, where necessary, modified by the authors. Generated code was evaluated using reference implementations and numerical or experimental tests. Mathematical claims, theoretical suggestions, hypotheses, and methodological recommendations were independently assessed by the authors. AI-assisted text, factual claims, and references were reviewed for accuracy, originality, and appropriate attribution. The authors made all final research decisions and take full responsibility for the content of the paper.

\clearpage 

\section{Qualitative Examples}
We present further qualitative examples of all three variants in Figs.~\ref{fig:qualitative.examples.hla.3f.r1.10}, \ref{fig:qualitative.examples.hla.3f.r1.21} and \ref{fig:qualitative.examples.hla.3f.r2.15}. As can be seen, all three methods achieve to generate high-quality videos. Please also see the original mp4 files in the supplementary material.

\clearpage

\begin{figure*}[htbp]
    \centering
    
        \includegraphics[width=14cm,height=2.03cm]{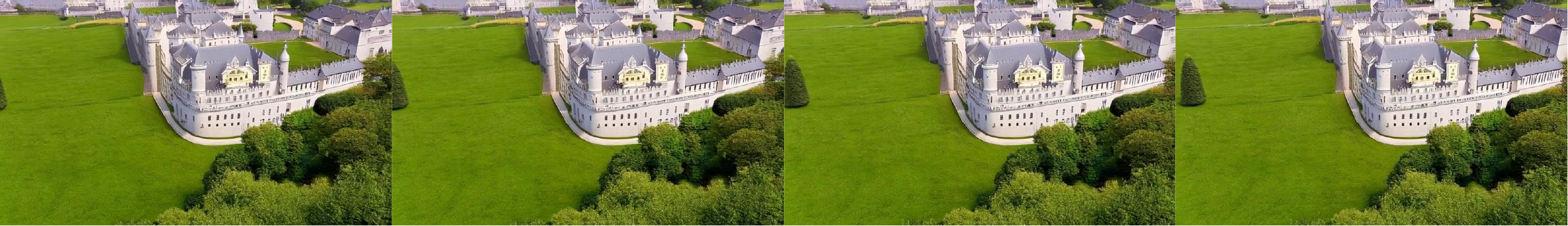}
        \includegraphics[width=14cm,height=2.03cm]{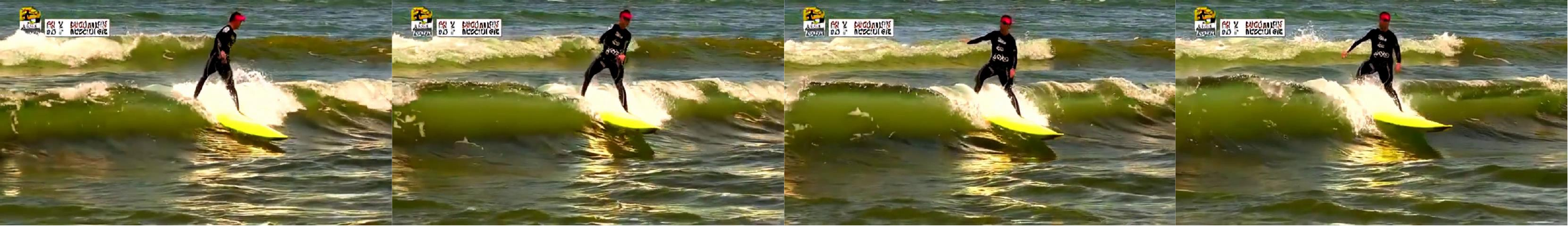}
        \includegraphics[width=14cm,height=2.03cm]{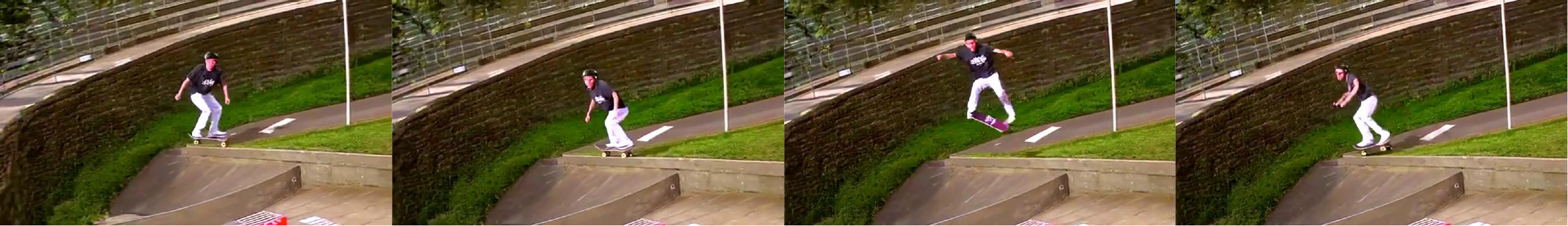}
        \includegraphics[width=14cm,height=2.03cm]{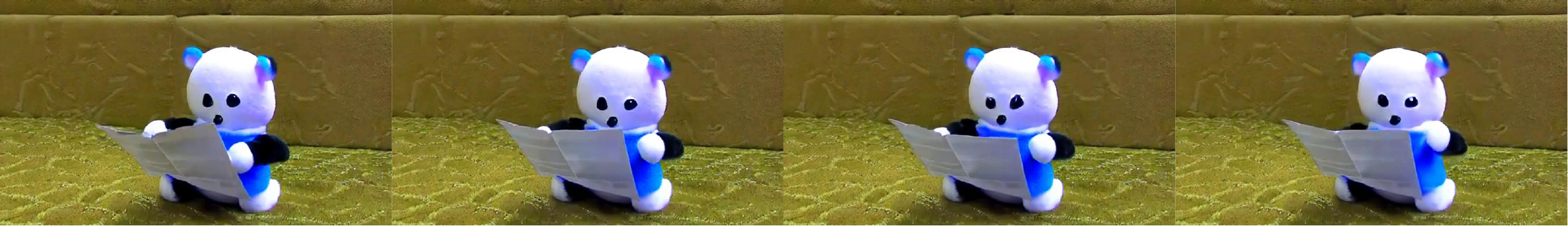}
        \includegraphics[width=14cm,height=2.03cm]{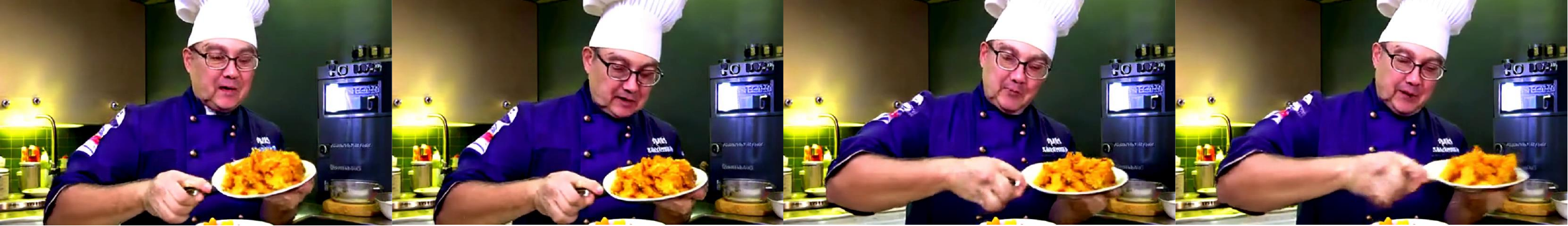}
        \includegraphics[width=14cm,height=2.03cm]{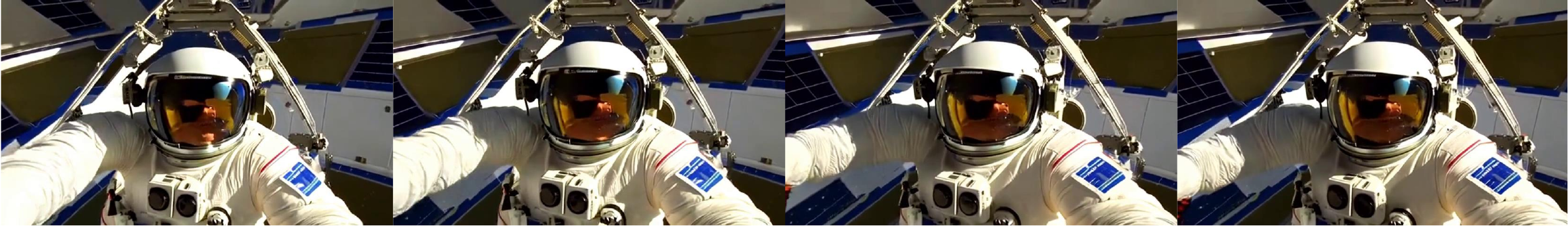}
        \includegraphics[width=14cm,height=2.03cm]{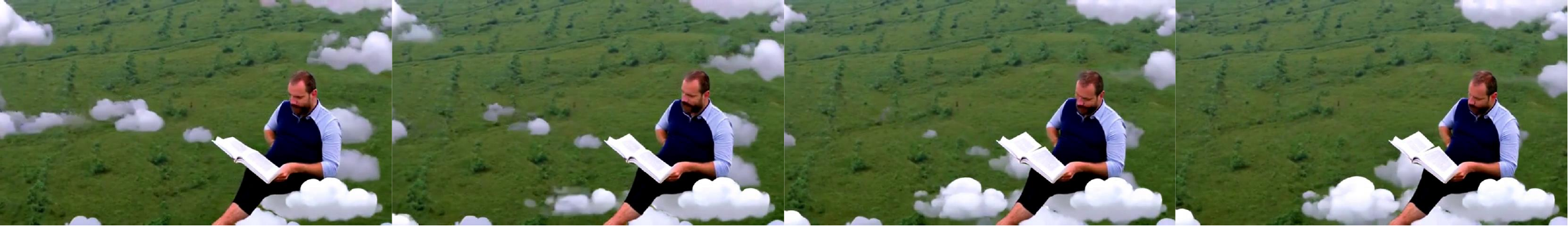}
        \includegraphics[width=14cm,height=2.03cm]{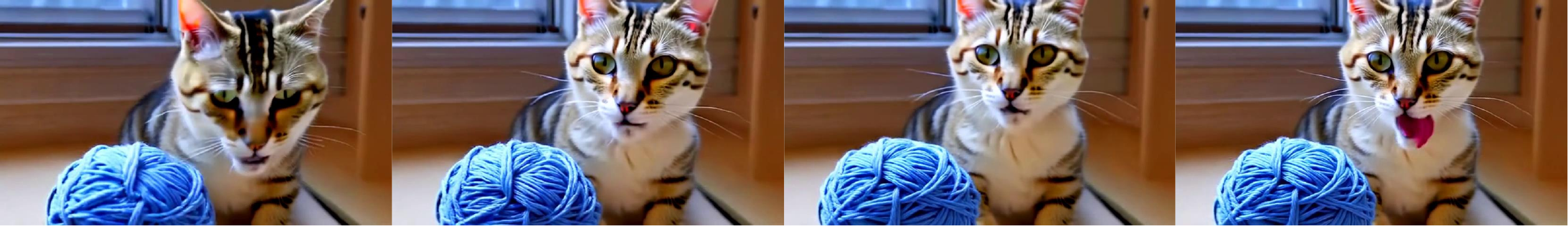}        
    \caption{Examples of generated videos by \textbf{HLA-3F-R1-10}.}
    \label{fig:qualitative.HLA-3F-R1-10}
\end{figure*}

\begin{figure*}[htbp]
    \centering
        \includegraphics[width=14cm,height=2.03cm]{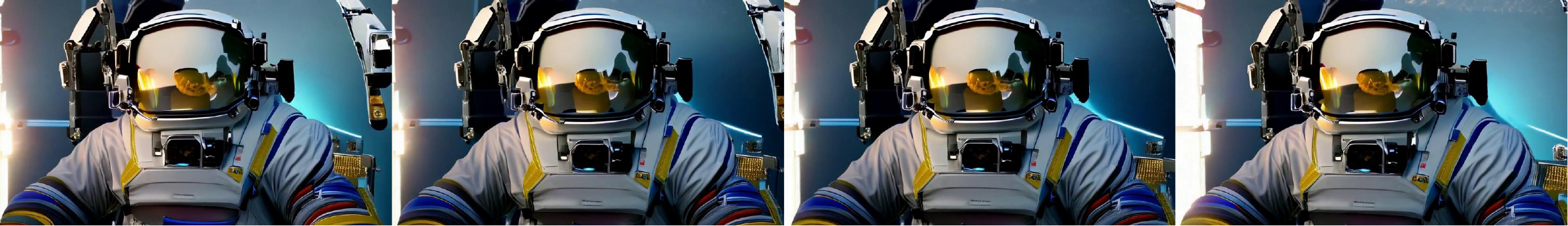}
        \includegraphics[width=14cm,height=2.03cm]{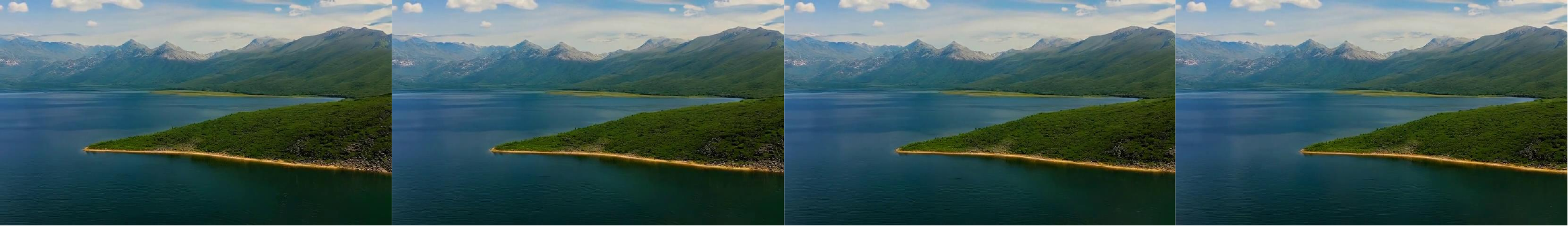}
        \includegraphics[width=14cm,height=2.03cm]{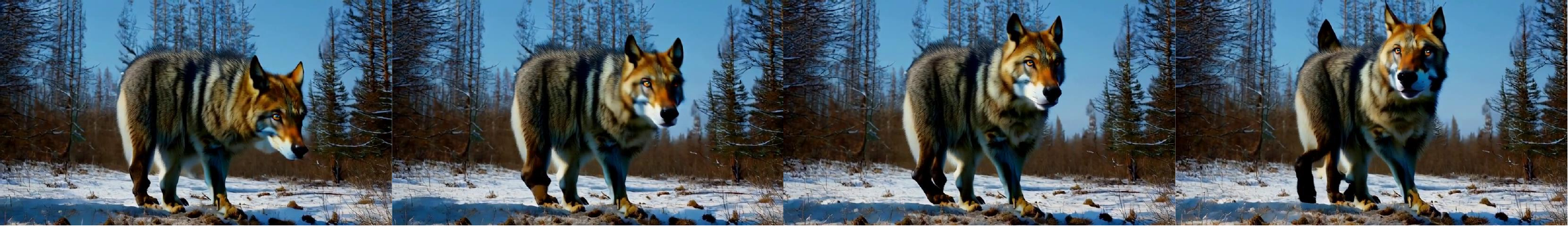}
        \includegraphics[width=14cm,height=2.03cm]{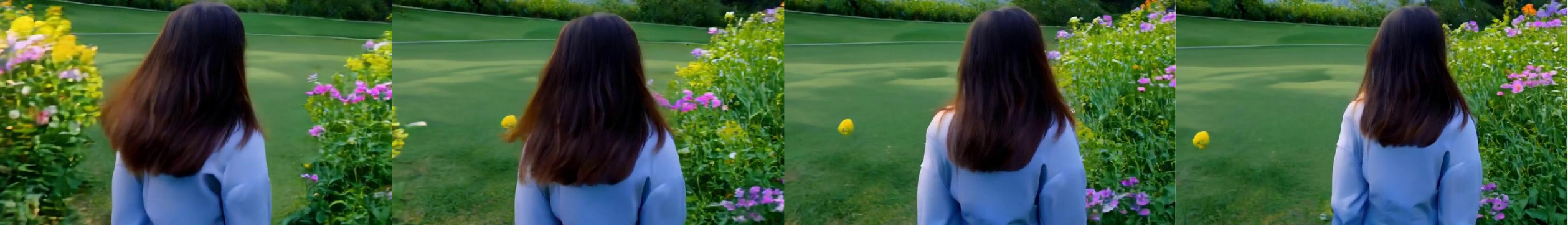}
        \includegraphics[width=14cm,height=2.03cm]{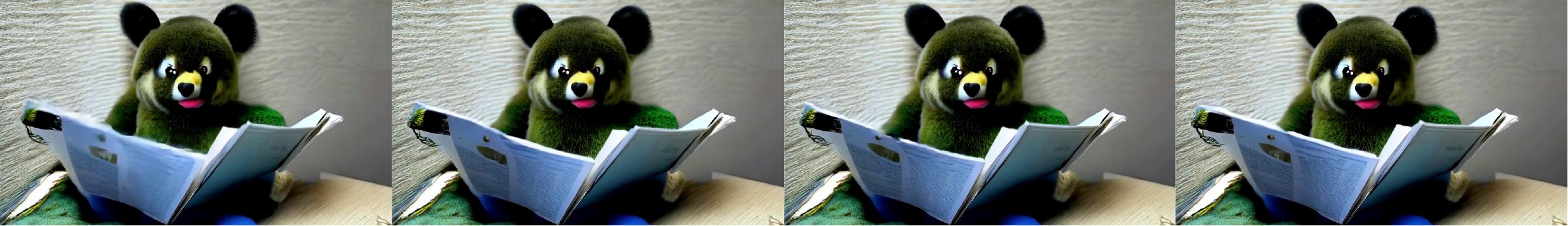}
        \includegraphics[width=14cm,height=2.03cm]{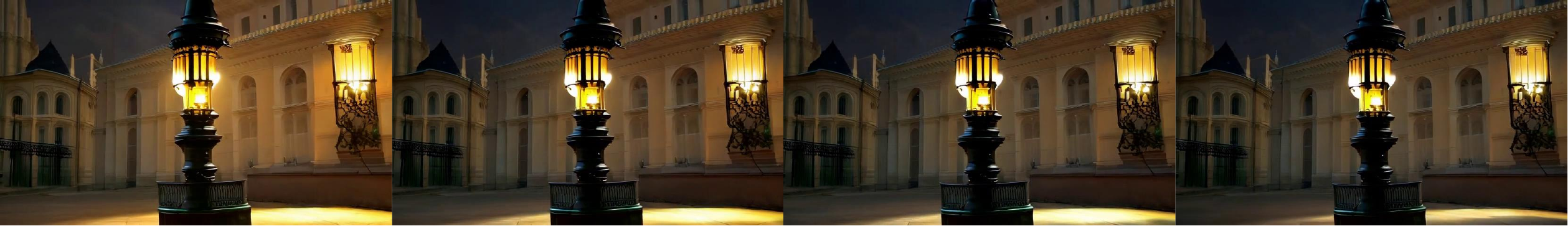}
        \includegraphics[width=14cm,height=2.03cm]{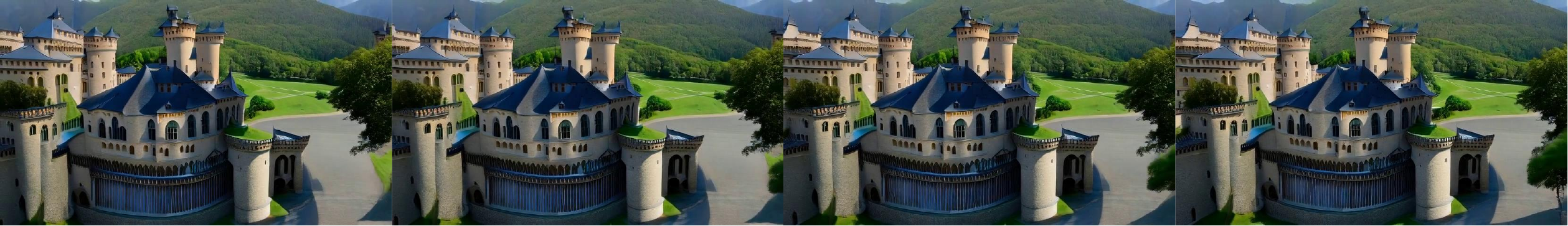}
        \includegraphics[width=14cm,height=2.03cm]{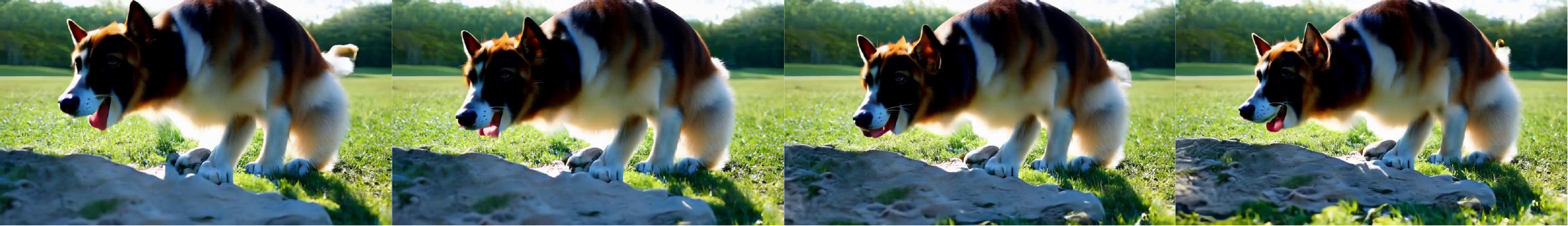}
    \caption{Examples of generated videos by \textbf{HLA-3F-R1-21}.}
    \label{fig:qualitative.HLA-3F-R1-21}
\end{figure*}

\begin{figure*}[htbp]
    \centering
        \includegraphics[width=14cm,height=2.03cm]{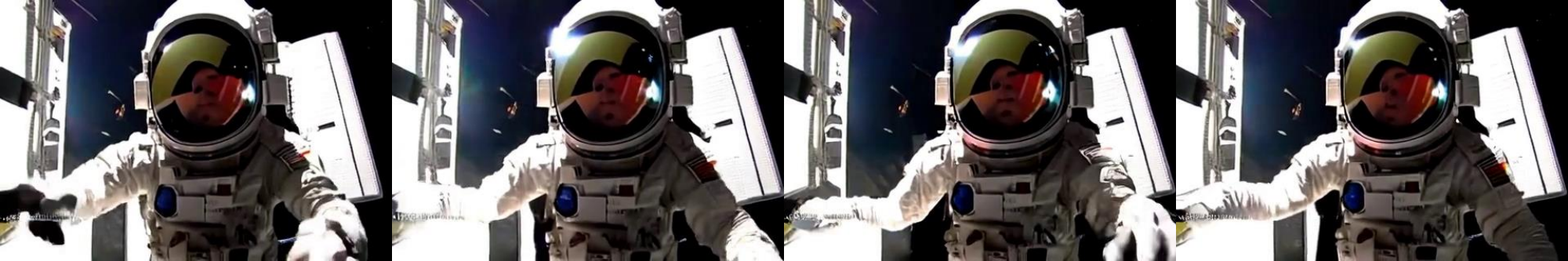}
        \includegraphics[width=14cm,height=2.03cm]{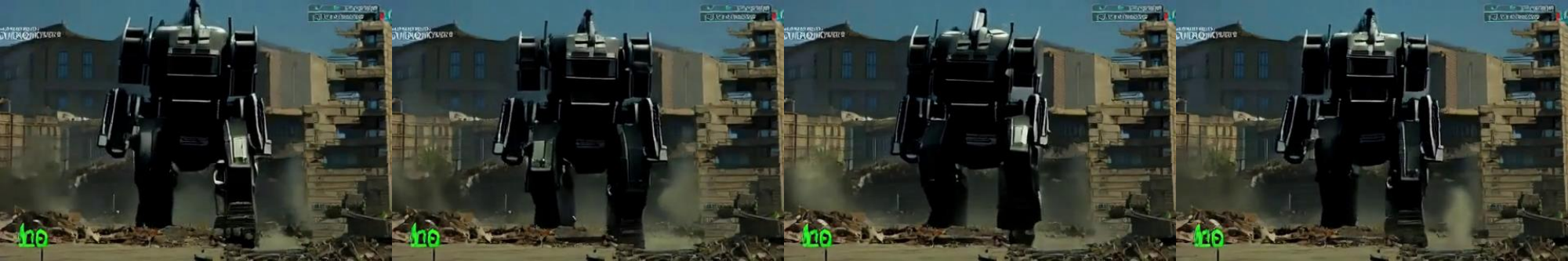}
        \includegraphics[width=14cm,height=2.03cm]{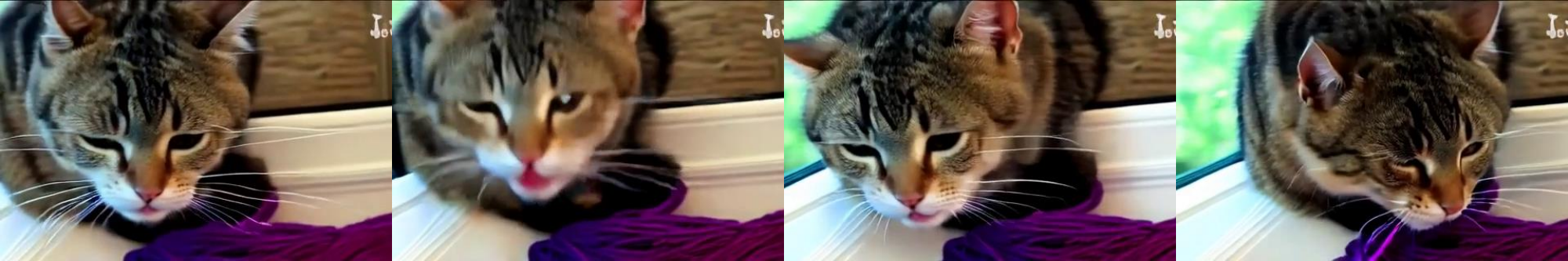}
        \includegraphics[width=14cm,height=2.03cm]{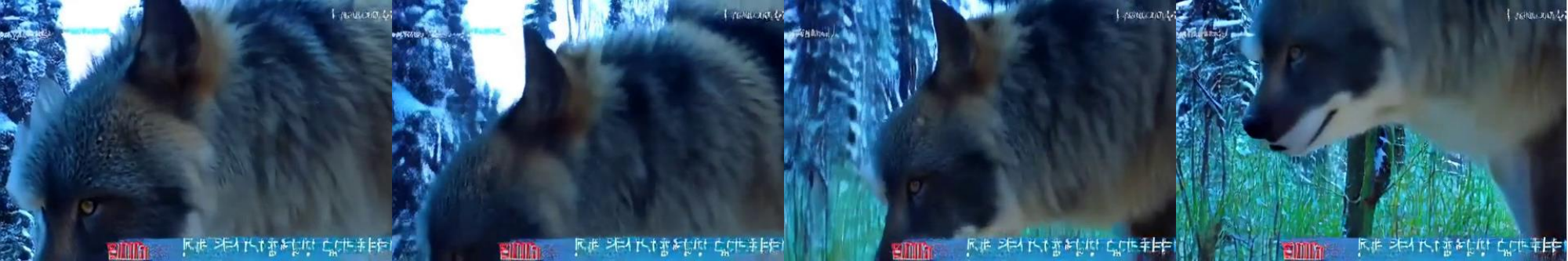}
        \includegraphics[width=14cm,height=2.03cm]{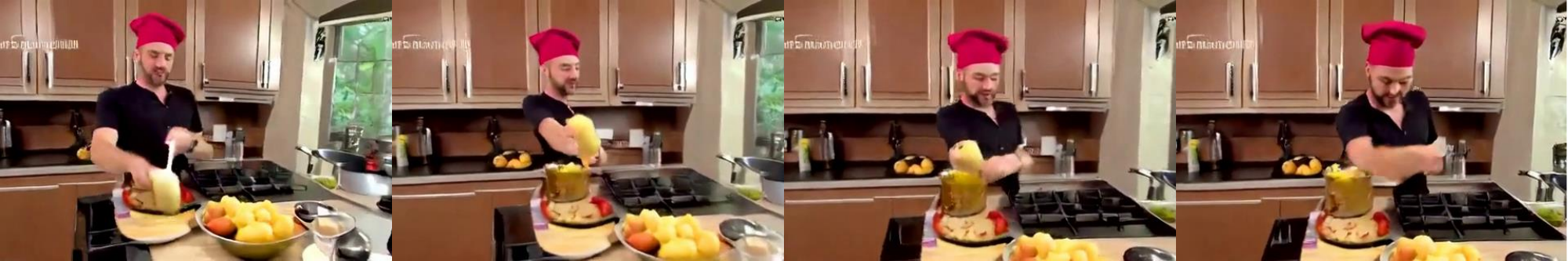}
        \includegraphics[width=14cm,height=2.03cm]{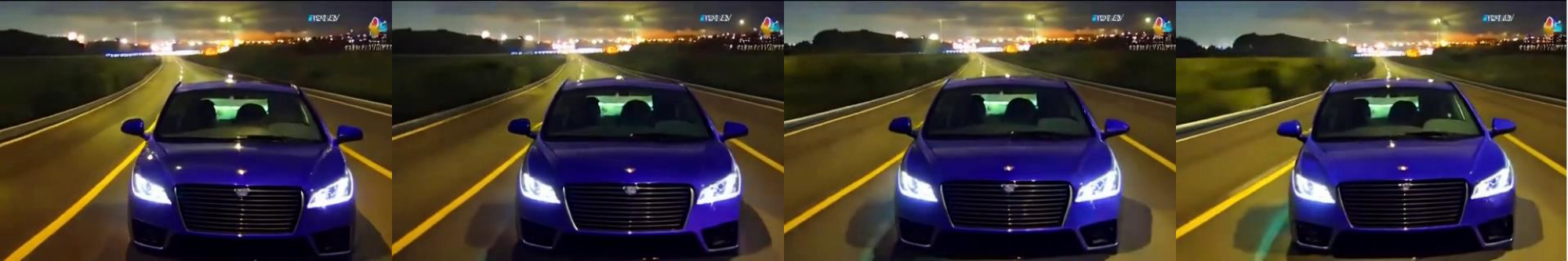}
        \includegraphics[width=14cm,height=2.03cm]{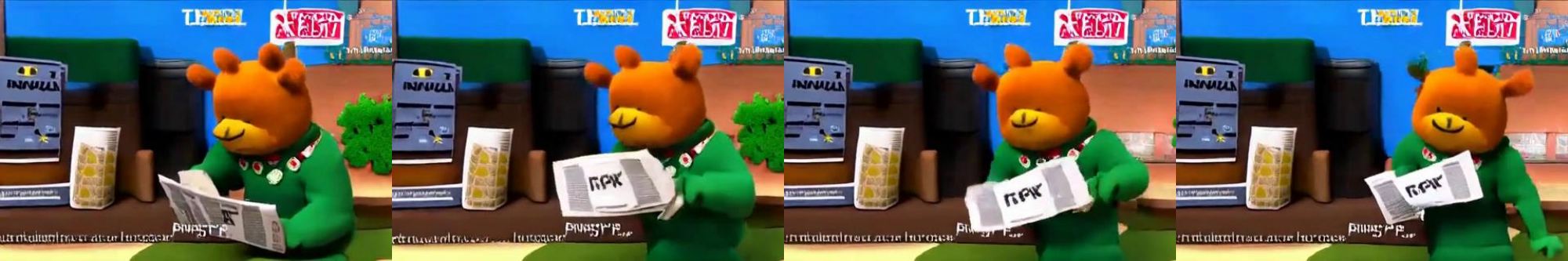}
        \includegraphics[width=14cm,height=2.03cm]{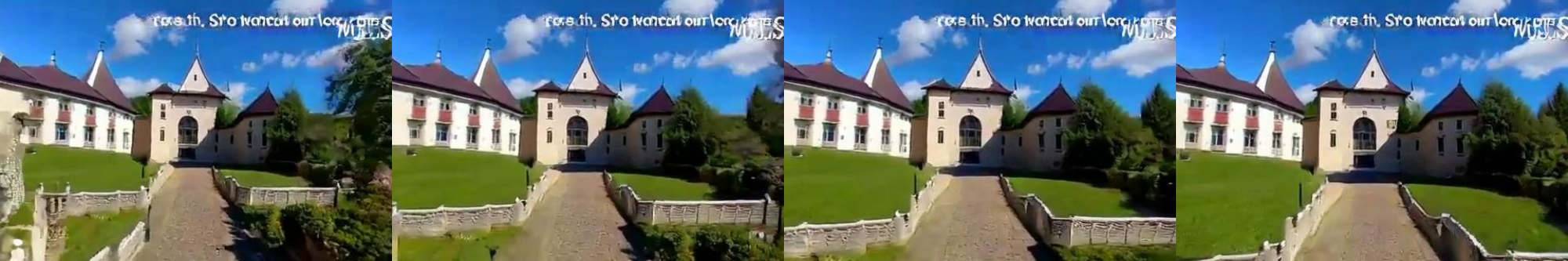}
    \caption{Examples of generated videos by \textbf{HLA-3F-R2-15}.}
    \label{fig:qualitative.HLA-3F-R2-15}
\end{figure*}


\begin{figure*}[htbp]
    \centering
    
        \includegraphics[width=3.5cm,height=2.03cm]{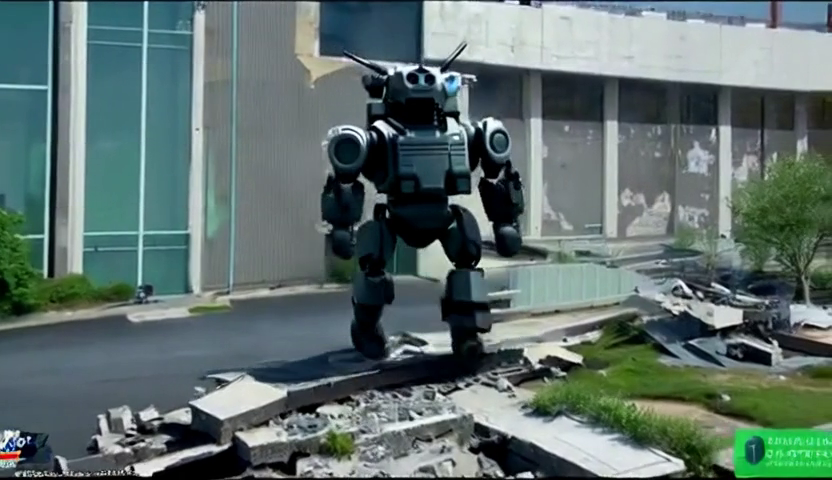}%
        \includegraphics[width=3.5cm,height=2.03cm]{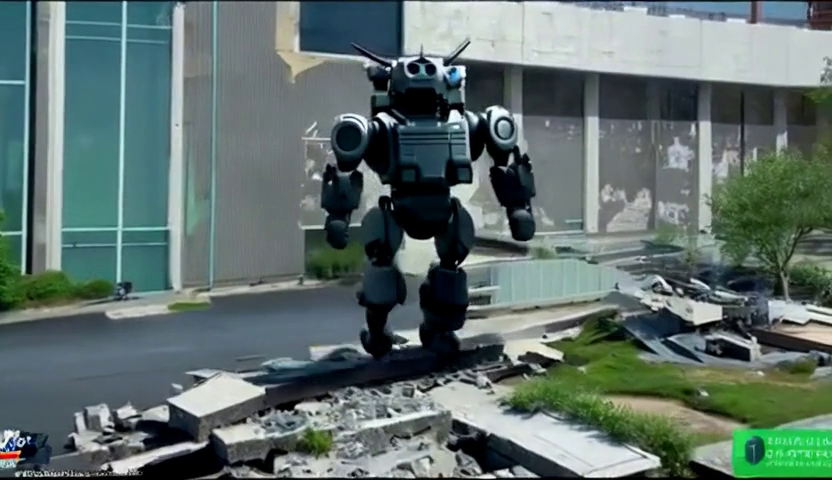}%
        \includegraphics[width=3.5cm,height=2.03cm]{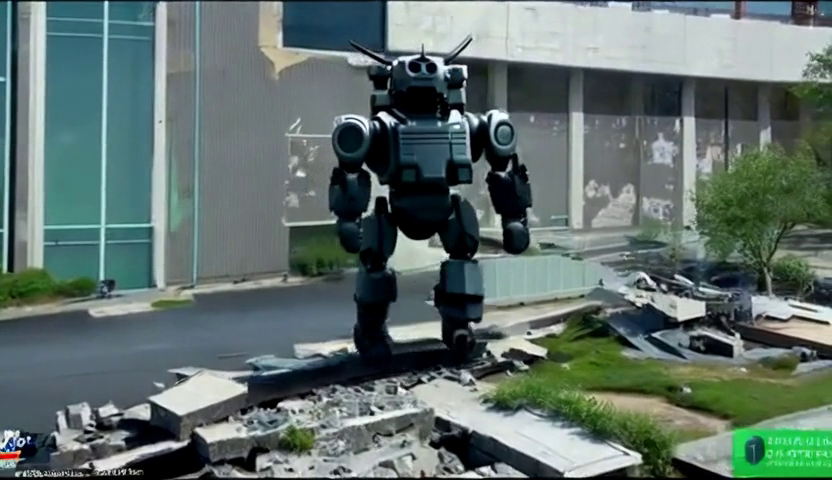}%
        \includegraphics[width=3.5cm,height=2.03cm]{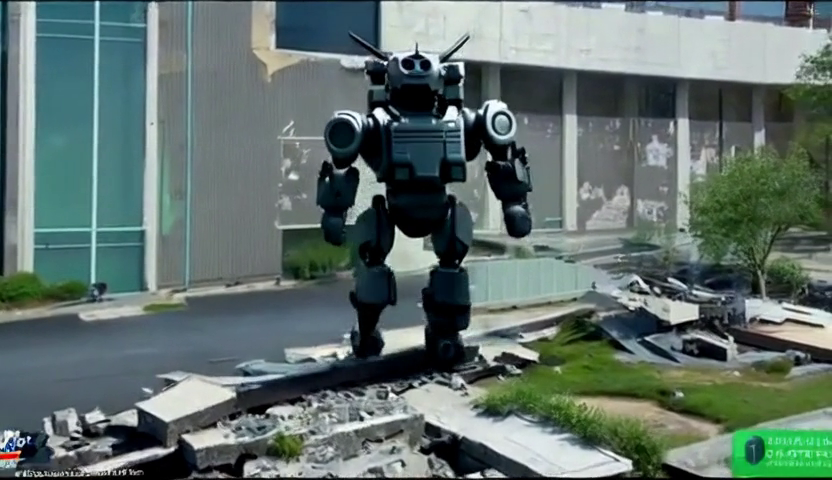}%
        \vspace{0.1cm}
        \includegraphics[width=3.5cm,height=2.03cm]{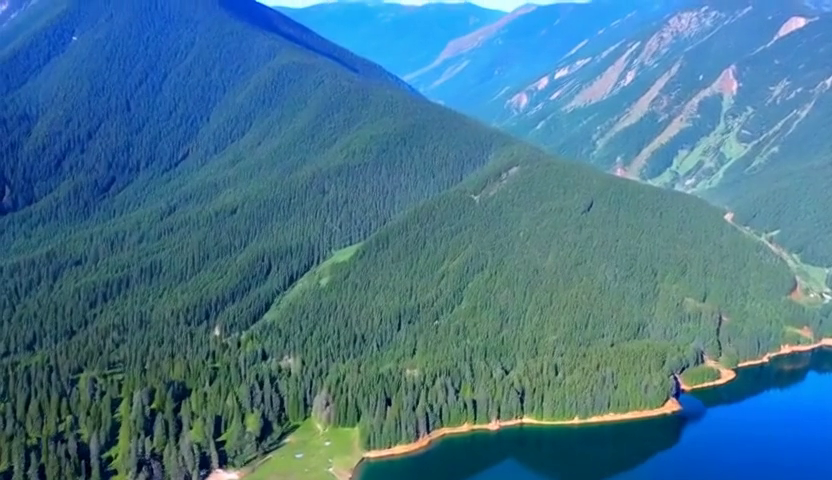}%
        \includegraphics[width=3.5cm,height=2.03cm]{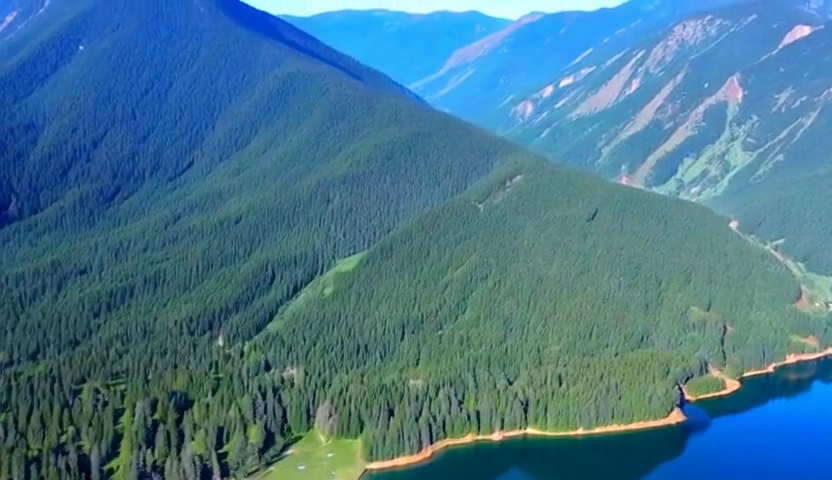}%
        \includegraphics[width=3.5cm,height=2.03cm]{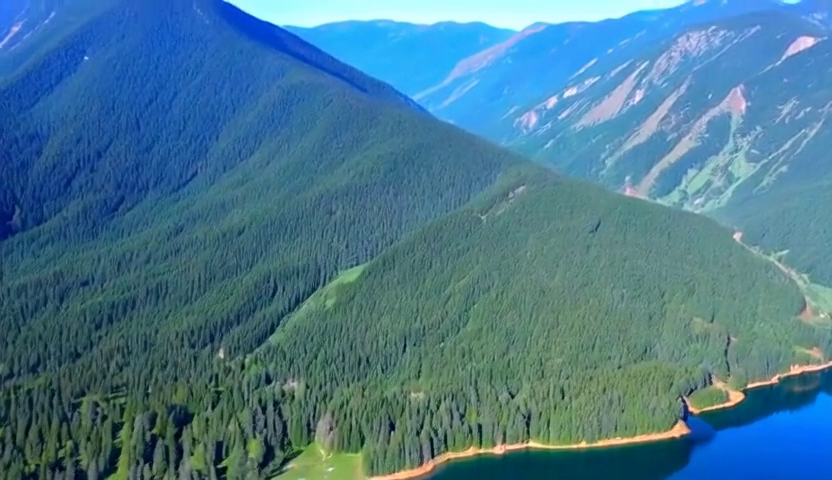}%
        \includegraphics[width=3.5cm,height=2.03cm]{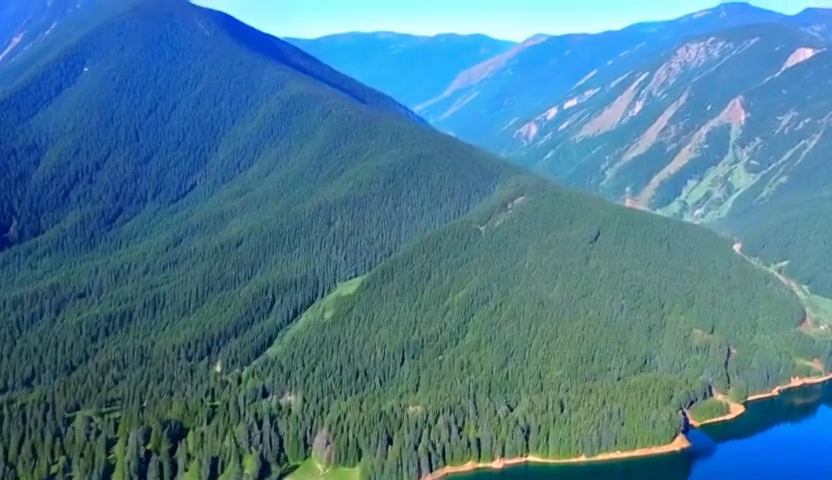}%
        \vspace{0.1cm}
        \includegraphics[width=3.5cm,height=2.03cm]{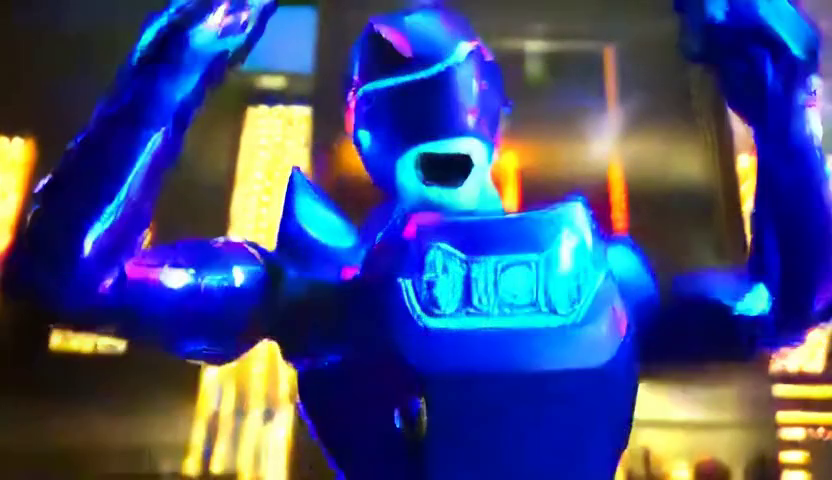}%
        \includegraphics[width=3.5cm,height=2.03cm]{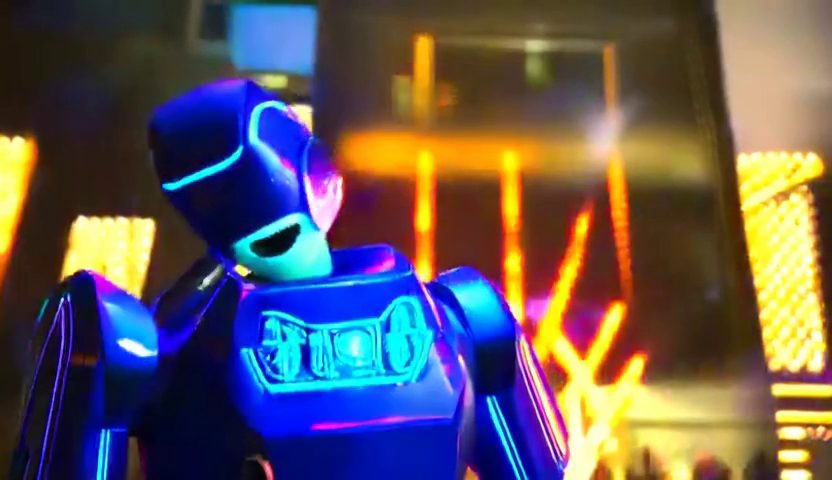}%
        \includegraphics[width=3.5cm,height=2.03cm]{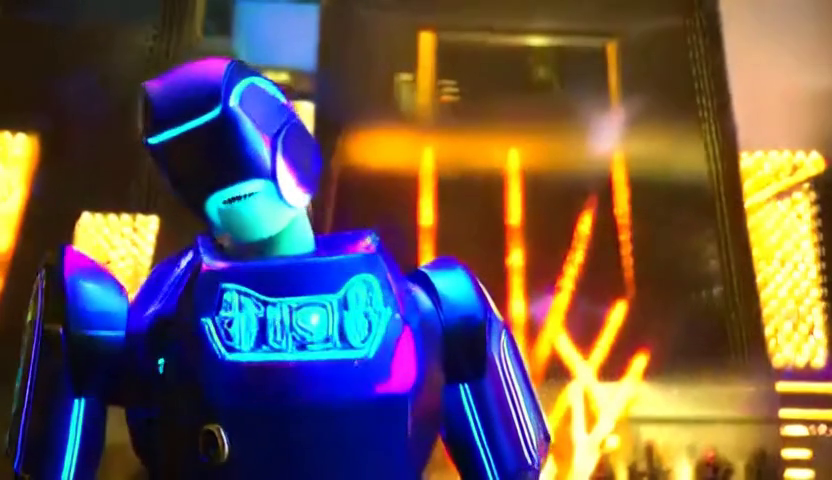}%
        \includegraphics[width=3.5cm,height=2.03cm]{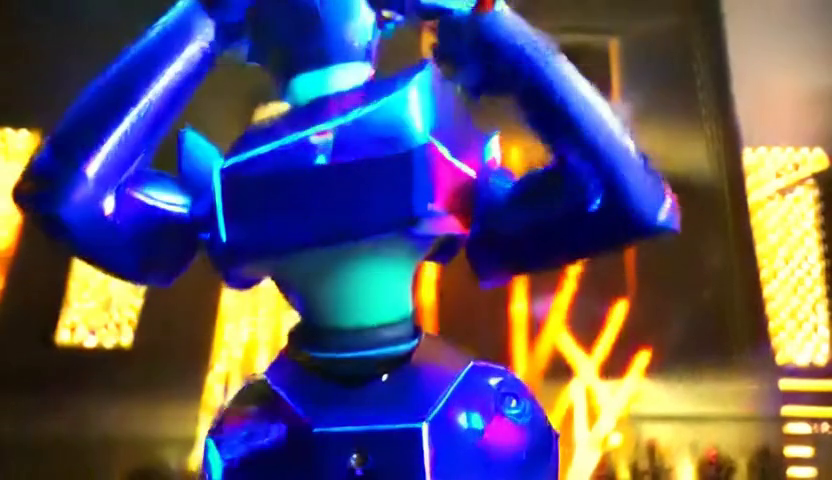}%
        \vspace{0.1cm}
        \includegraphics[width=3.5cm,height=2.03cm]{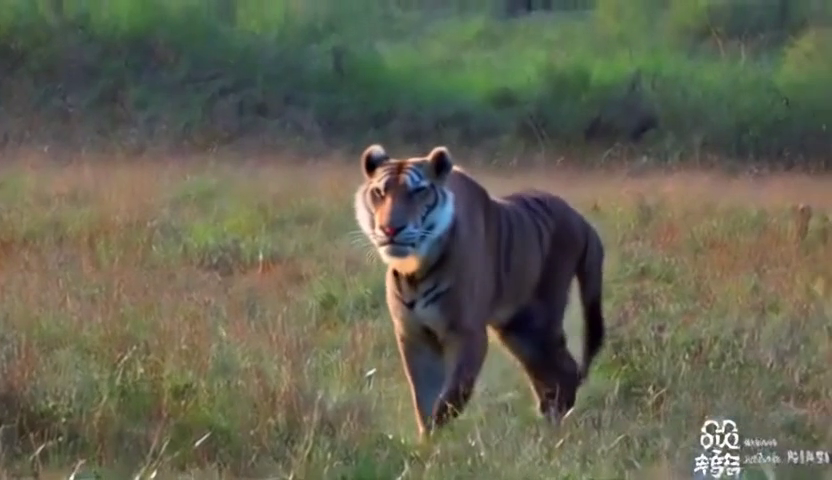}%
        \includegraphics[width=3.5cm,height=2.03cm]{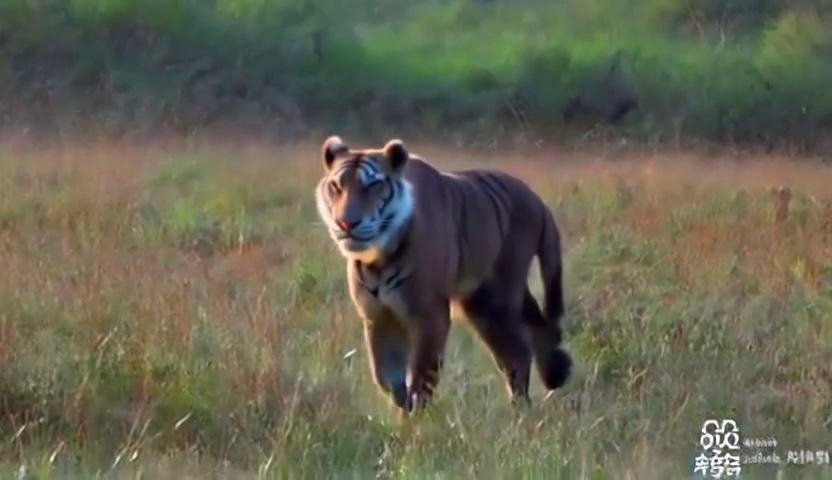}%
        \includegraphics[width=3.5cm,height=2.03cm]{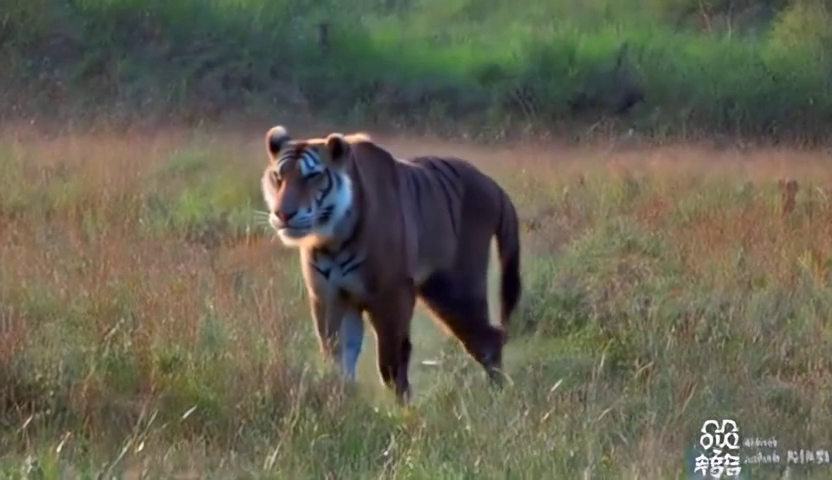}%
        \includegraphics[width=3.5cm,height=2.03cm]{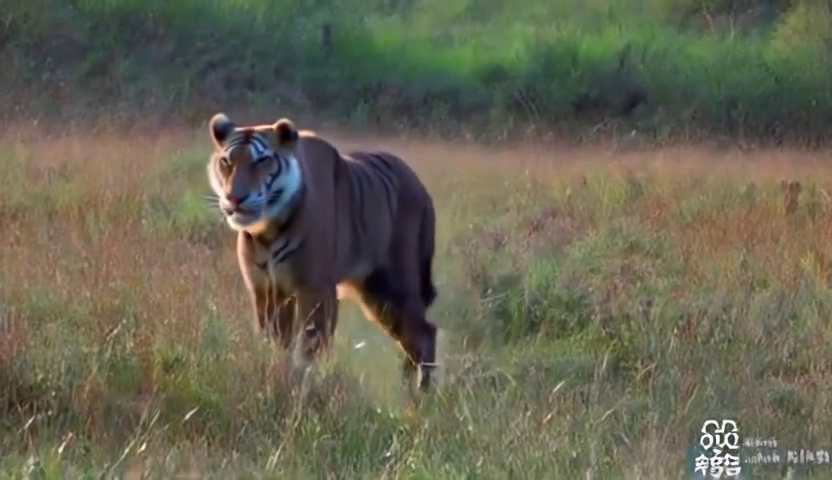}%
        \vspace{0.1cm}
        \includegraphics[width=3.5cm,height=2.03cm]{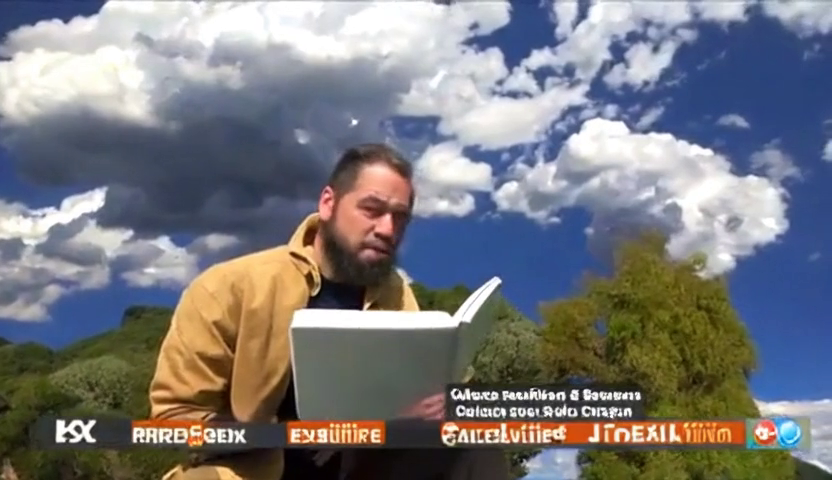}%
        \includegraphics[width=3.5cm,height=2.03cm]{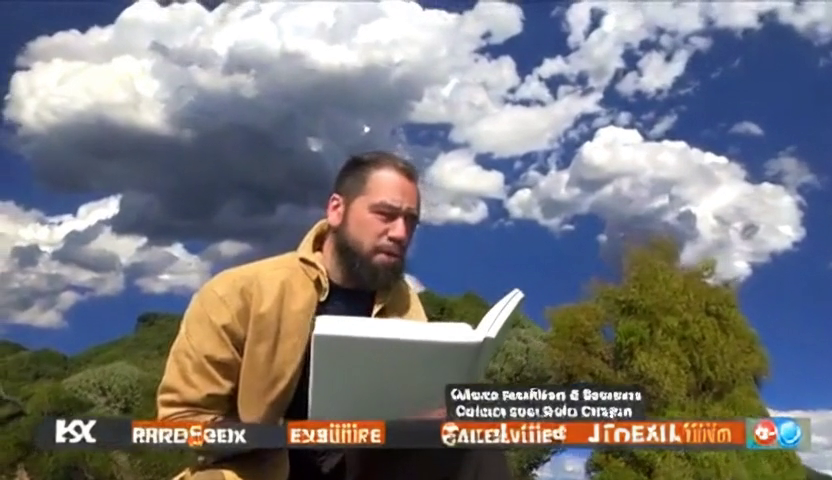}%
        \includegraphics[width=3.5cm,height=2.03cm]{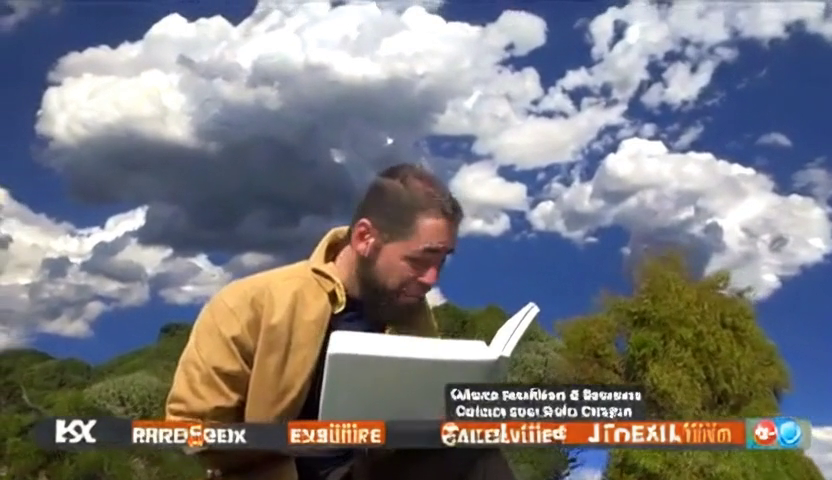}%
        \includegraphics[width=3.5cm,height=2.03cm]{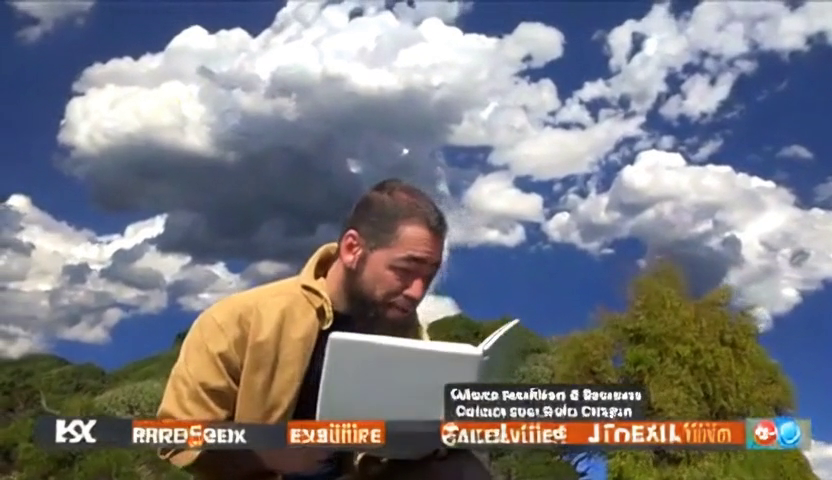}%
        \vspace{0.1cm}
        \includegraphics[width=3.5cm,height=2.03cm]{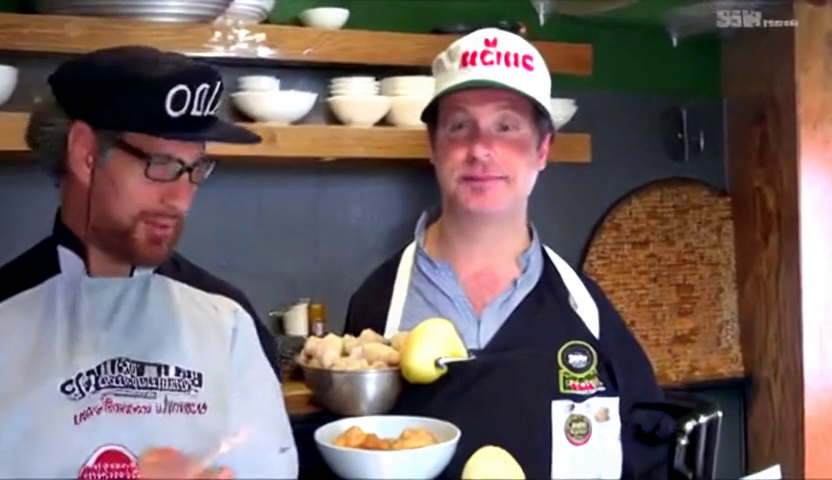}%
        \includegraphics[width=3.5cm,height=2.03cm]{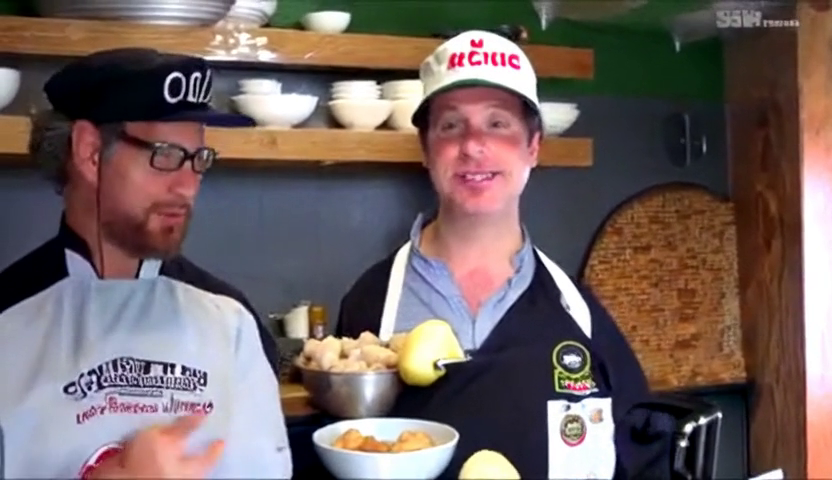}%
        \includegraphics[width=3.5cm,height=2.03cm]{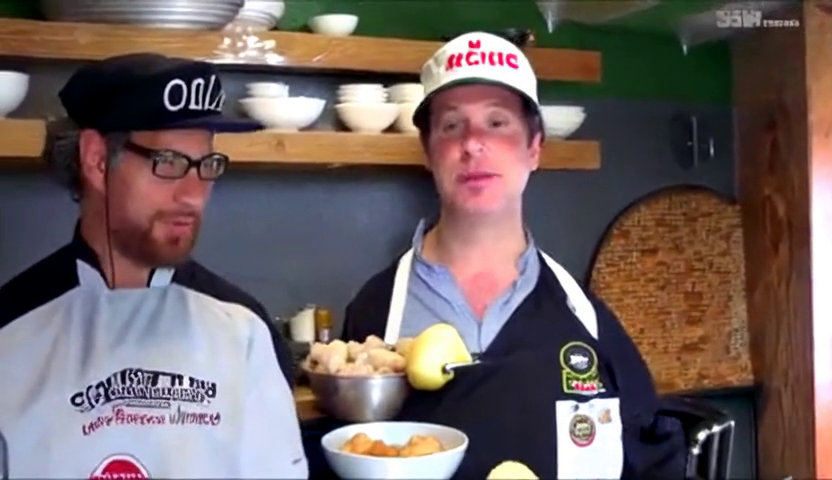}%
        \includegraphics[width=3.5cm,height=2.03cm]{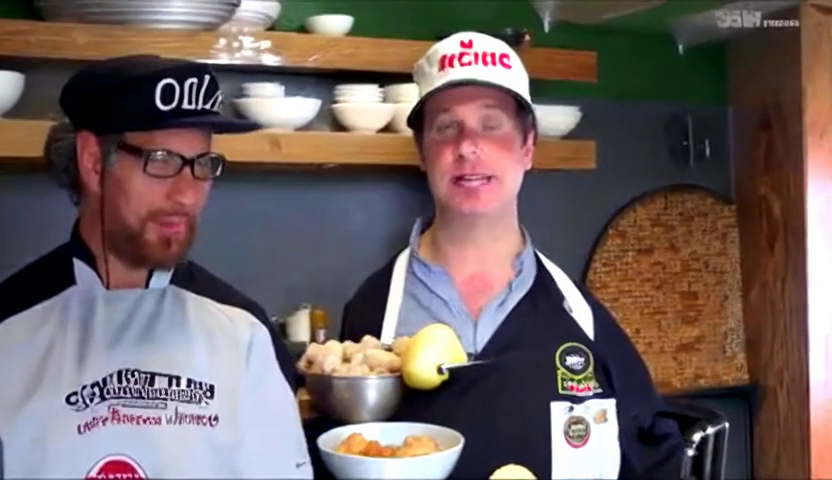}%
    
    \caption{Qualitative Examples \textbf{HLA-3F-R1-10}}
    \label{fig:qualitative.examples.hla.3f.r1.10}
\end{figure*}

\begin{figure*}[htbp]
    \centering
    
        \includegraphics[width=3.5cm,height=2.03cm]{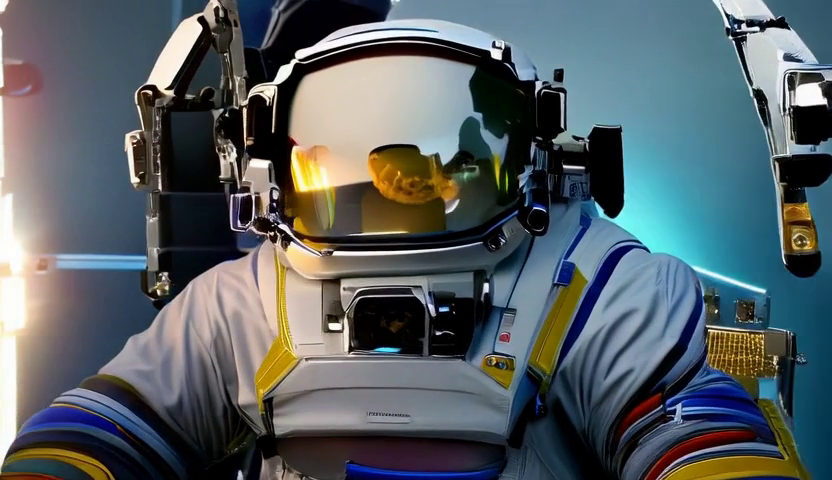}%
        \includegraphics[width=3.5cm,height=2.03cm]{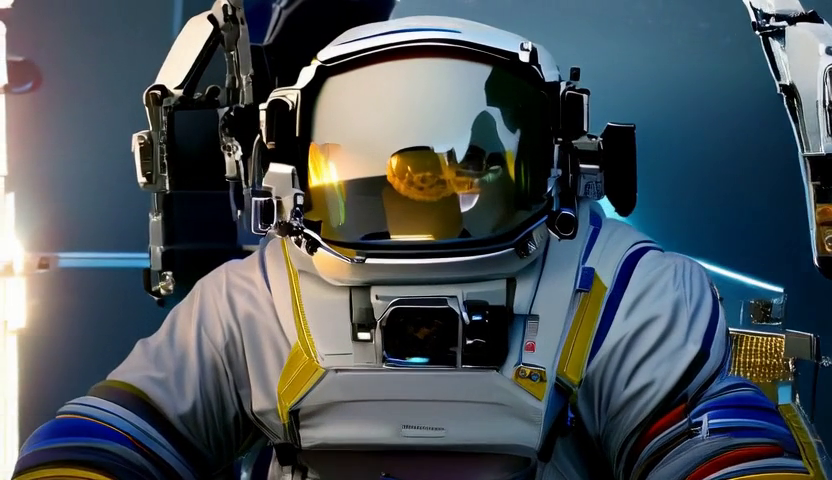}%
        \includegraphics[width=3.5cm,height=2.03cm]{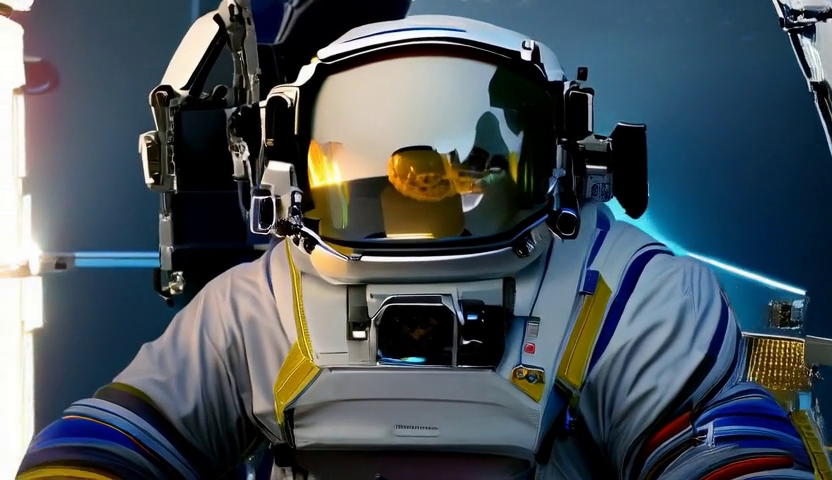}%
        \includegraphics[width=3.5cm,height=2.03cm]{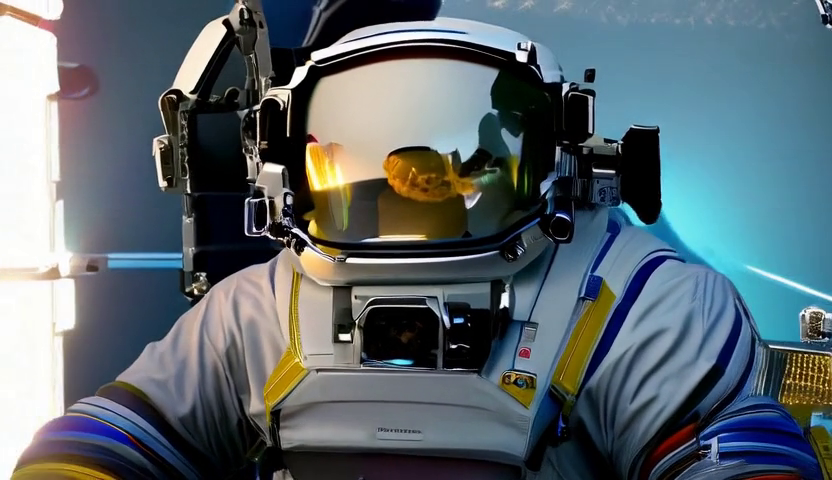}%
        \vspace{0.1cm}
        \includegraphics[width=3.5cm,height=2.03cm]{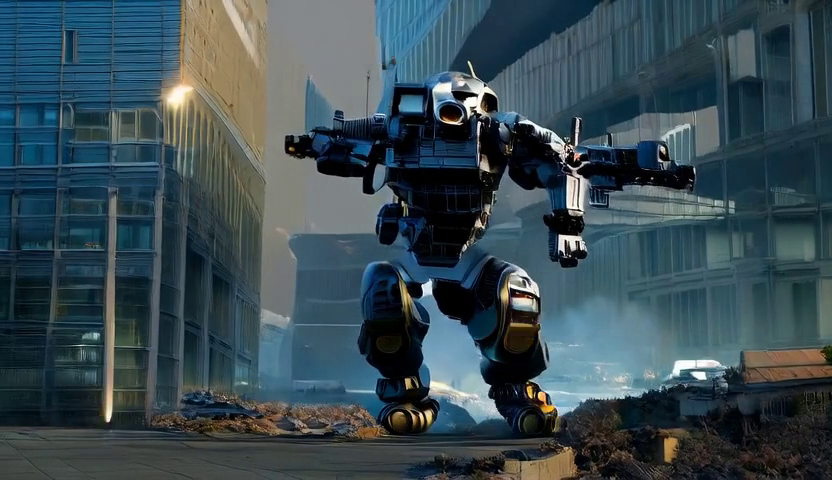}%
        \includegraphics[width=3.5cm,height=2.03cm]{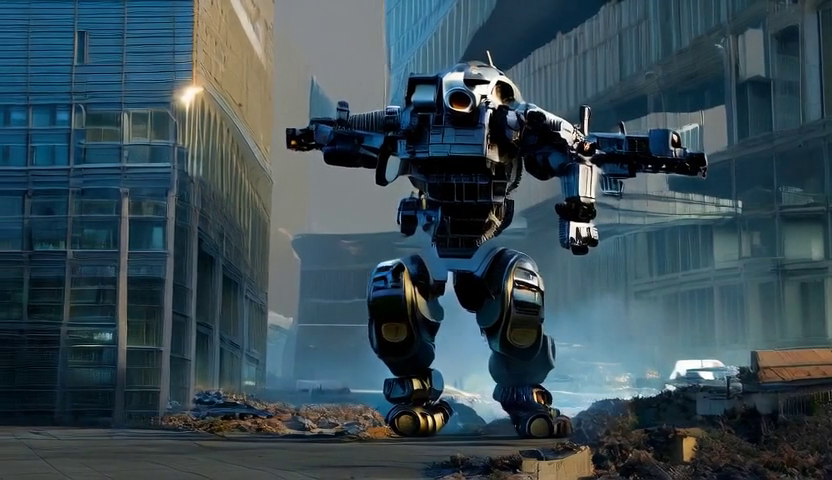}%
        \includegraphics[width=3.5cm,height=2.03cm]{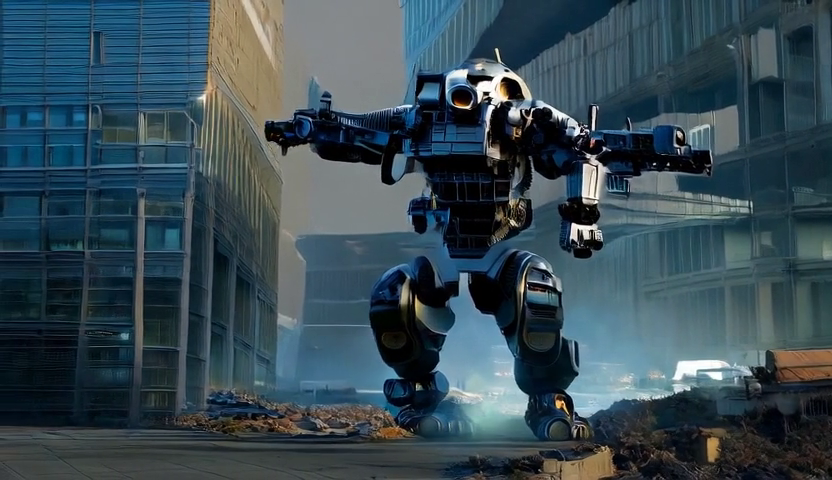}%
        \includegraphics[width=3.5cm,height=2.03cm]{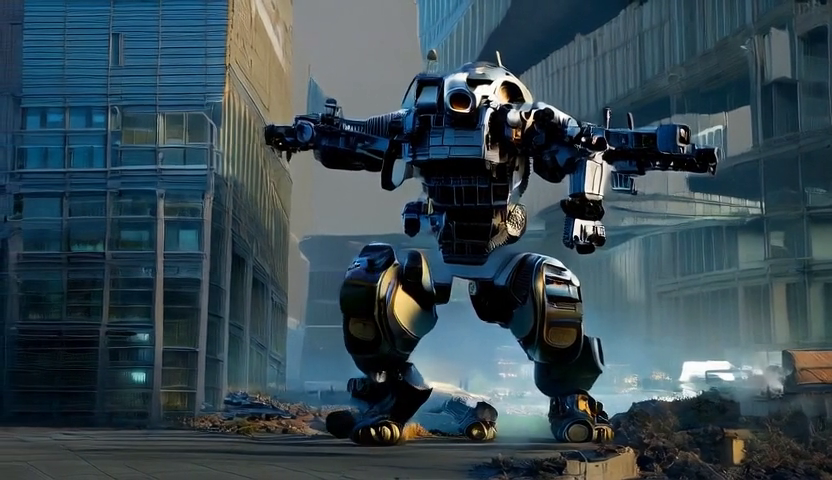}%
        \vspace{0.1cm}
        \includegraphics[width=3.5cm,height=2.03cm]{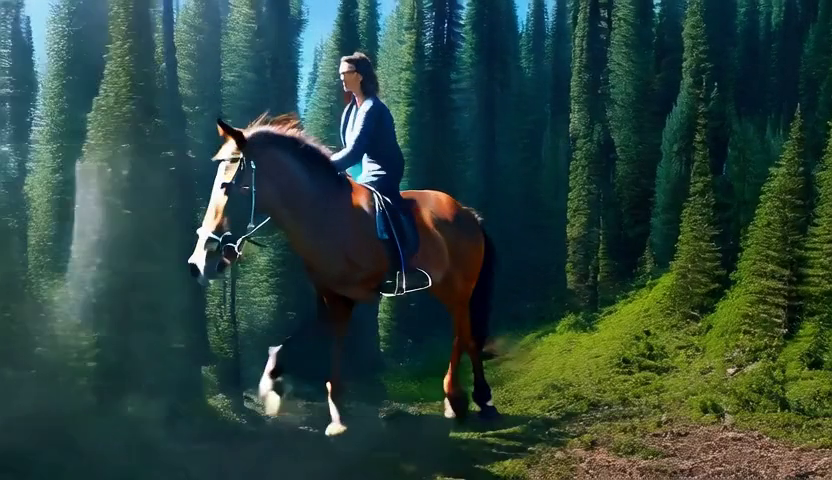}%
        \includegraphics[width=3.5cm,height=2.03cm]{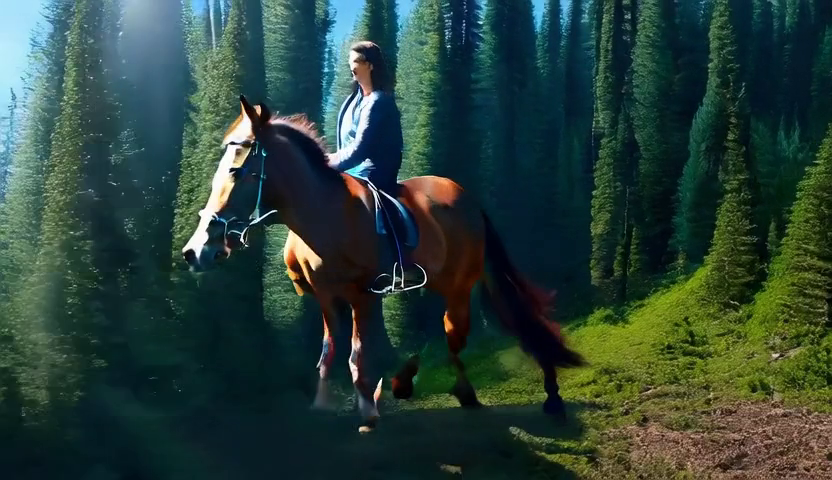}%
        \includegraphics[width=3.5cm,height=2.03cm]{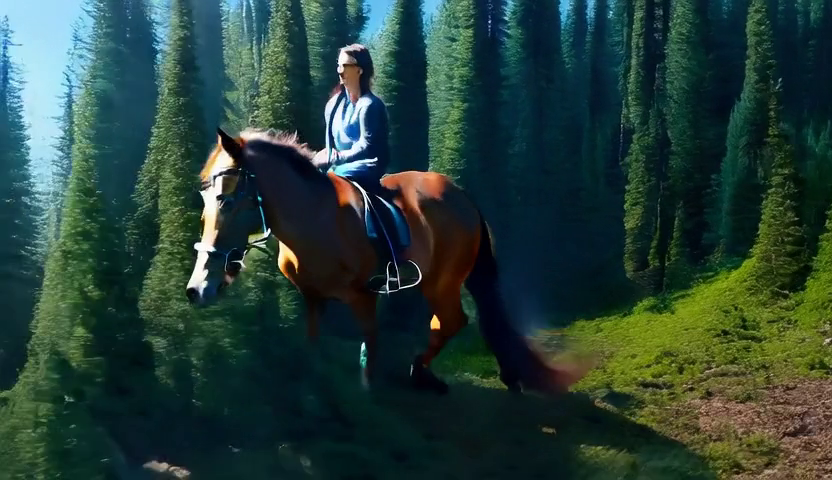}%
        \includegraphics[width=3.5cm,height=2.03cm]{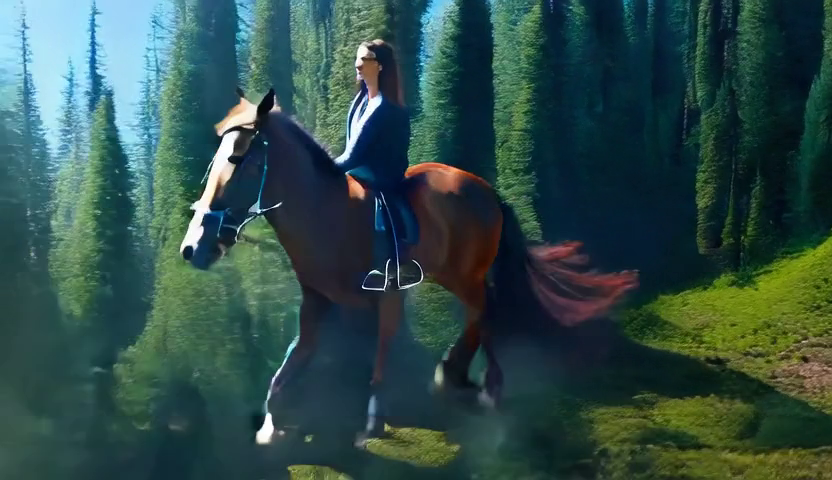}%
        \vspace{0.1cm}
        \includegraphics[width=3.5cm,height=2.03cm]{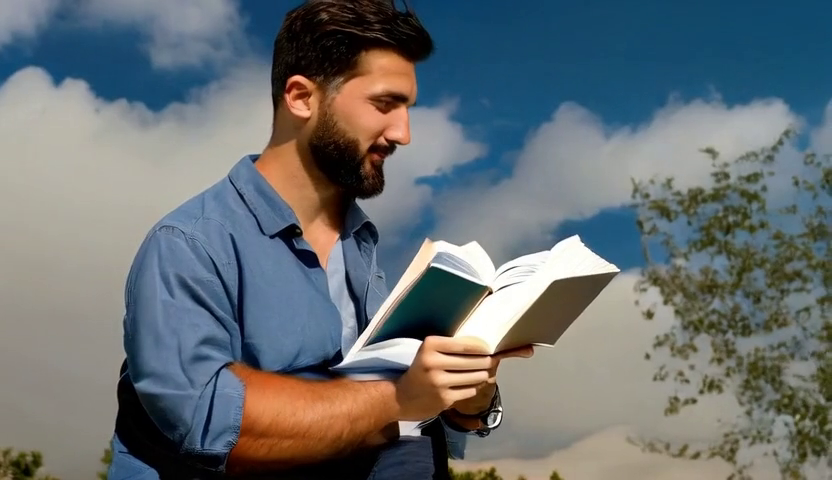}%
        \includegraphics[width=3.5cm,height=2.03cm]{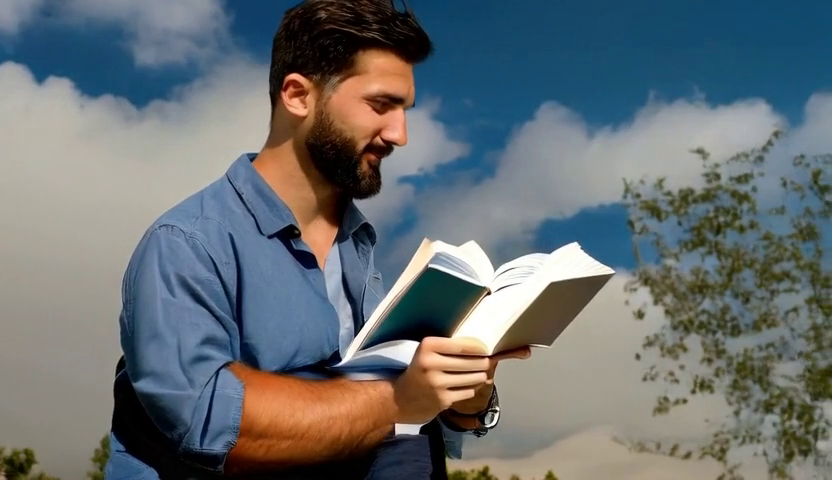}%
        \includegraphics[width=3.5cm,height=2.03cm]{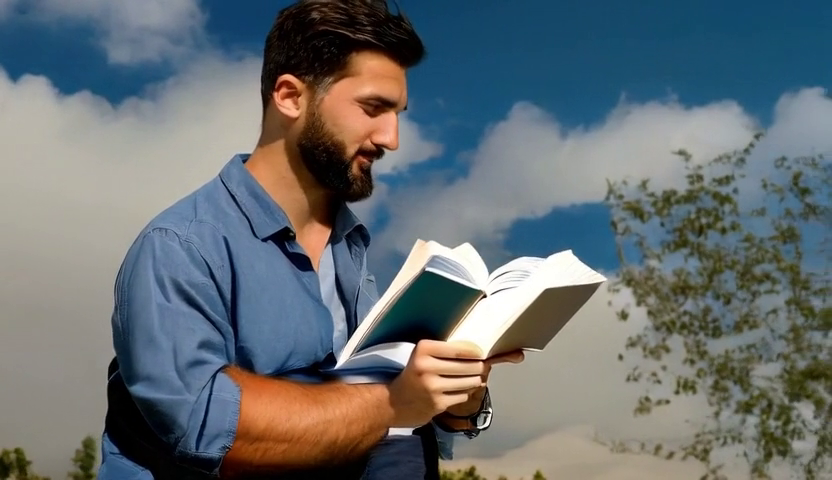}%
        \includegraphics[width=3.5cm,height=2.03cm]{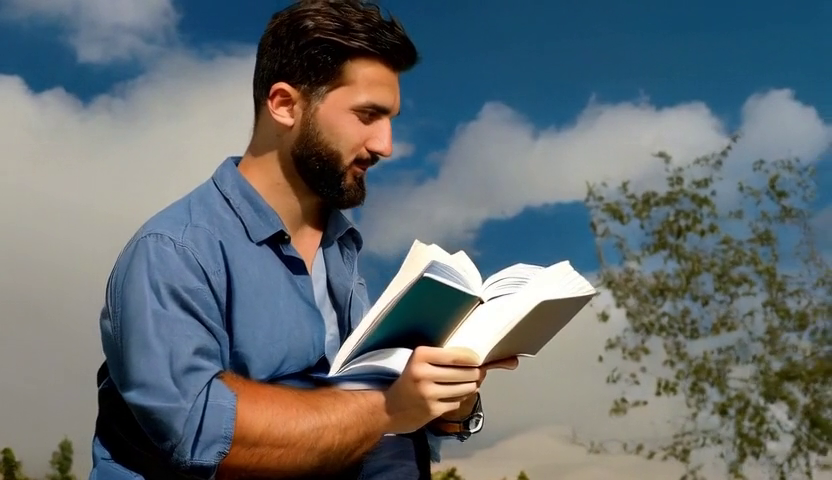}%
        \vspace{0.1cm}
        \includegraphics[width=3.5cm,height=2.03cm]{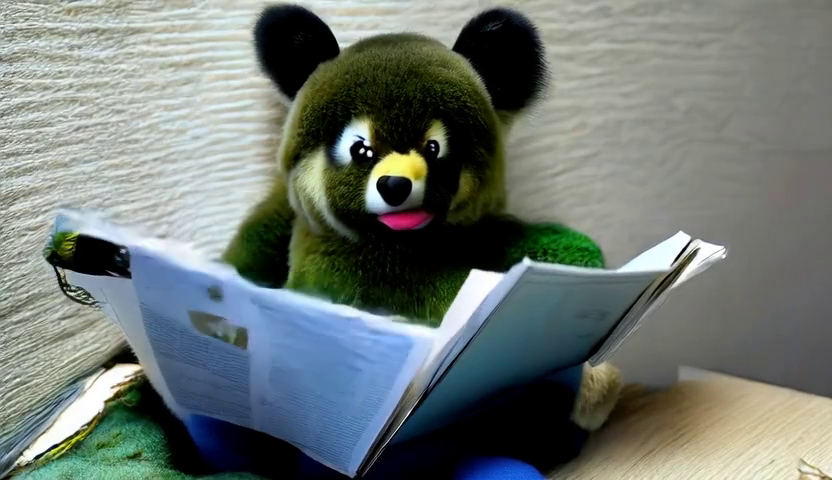}%
        \includegraphics[width=3.5cm,height=2.03cm]{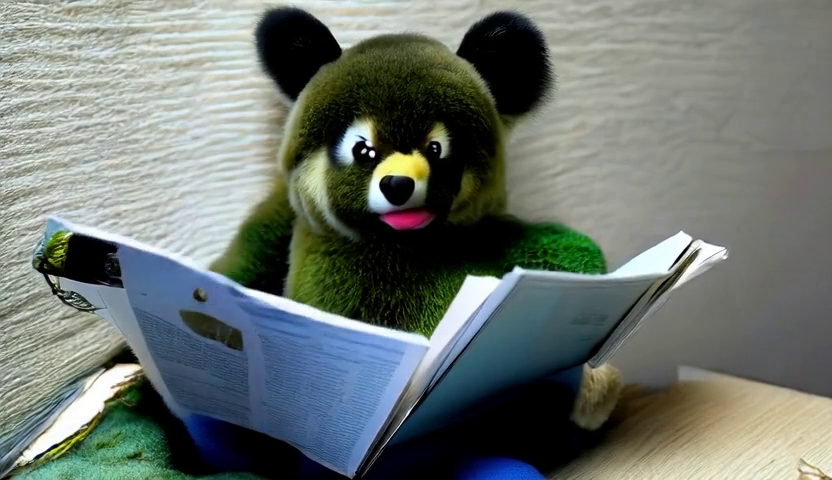}%
        \includegraphics[width=3.5cm,height=2.03cm]{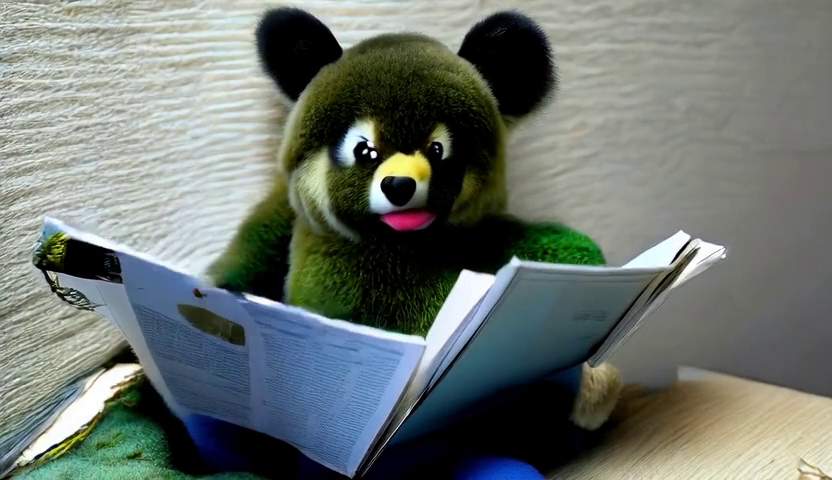}%
        \includegraphics[width=3.5cm,height=2.03cm]{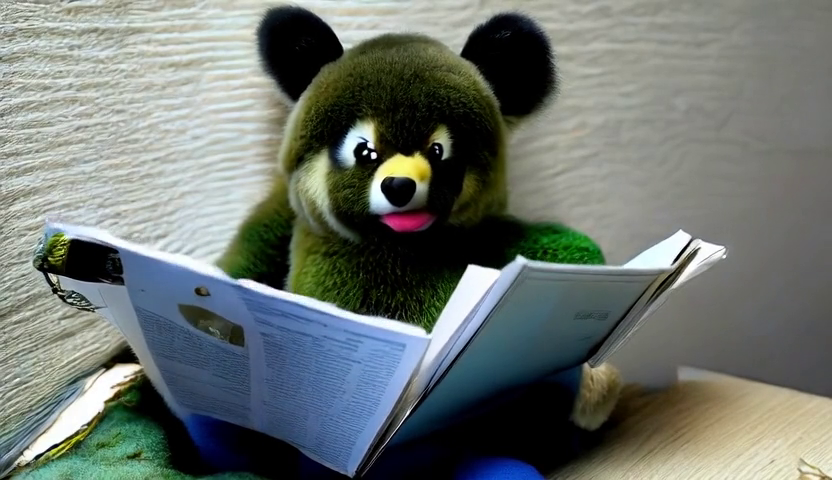}%
        \vspace{0.1cm}
        \includegraphics[width=3.5cm,height=2.03cm]{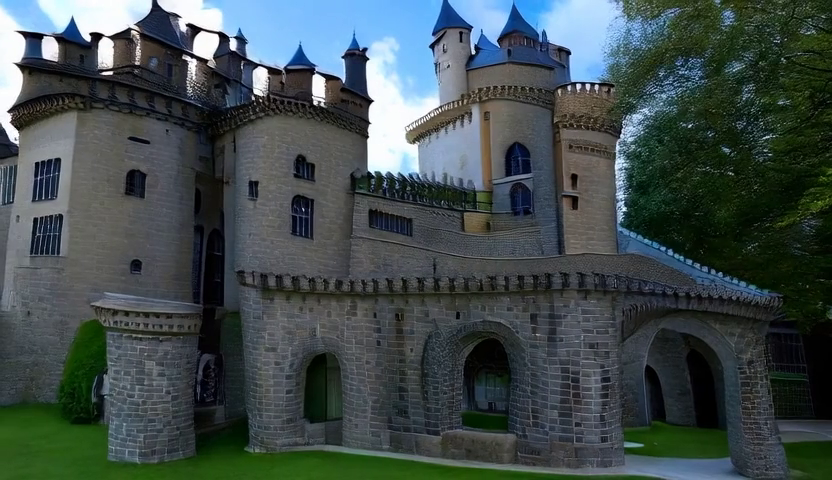}%
        \includegraphics[width=3.5cm,height=2.03cm]{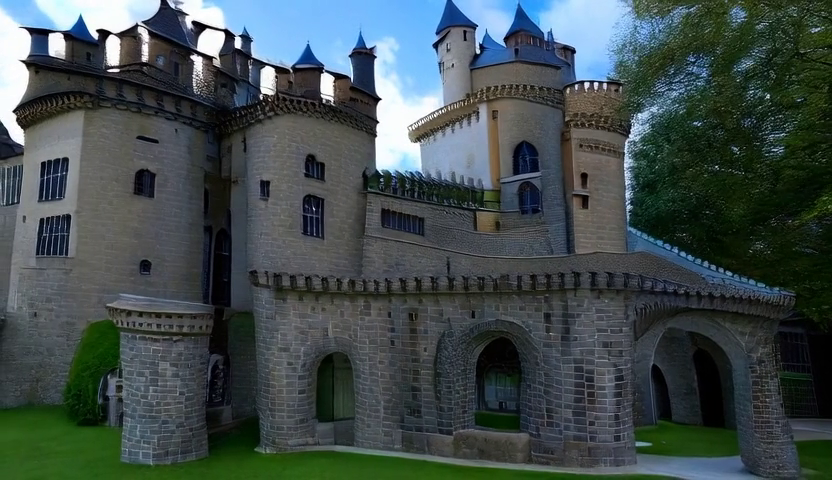}%
        \includegraphics[width=3.5cm,height=2.03cm]{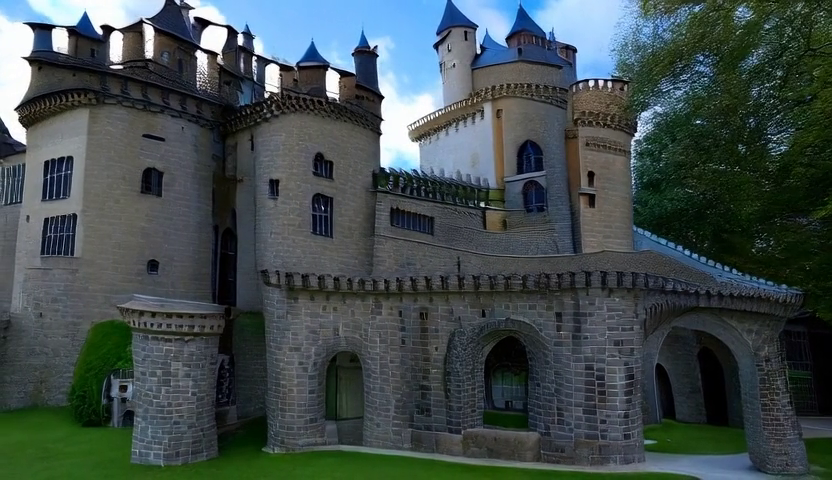}%
        \includegraphics[width=3.5cm,height=2.03cm]{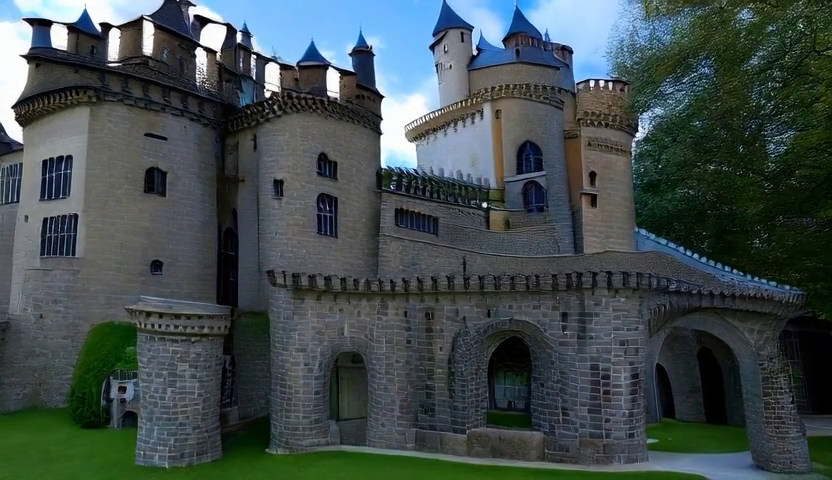}%
    
    \caption{Qualitative Examples \textbf{HLA-3F-R1-21}}
    \label{fig:qualitative.examples.hla.3f.r1.21}
\end{figure*}

\begin{figure*}[htbp]
    \centering
    
        \includegraphics[width=3.5cm,height=2.03cm]{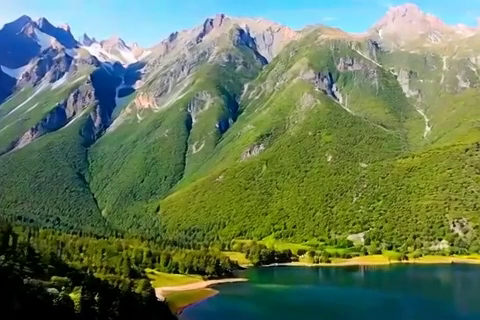}%
        \includegraphics[width=3.5cm,height=2.03cm]{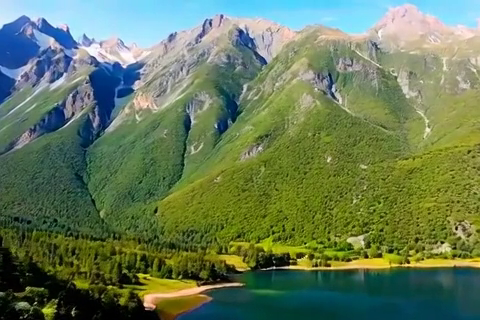}%
        \includegraphics[width=3.5cm,height=2.03cm]{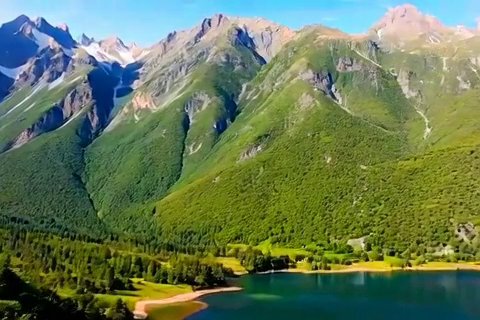}%
        \includegraphics[width=3.5cm,height=2.03cm]{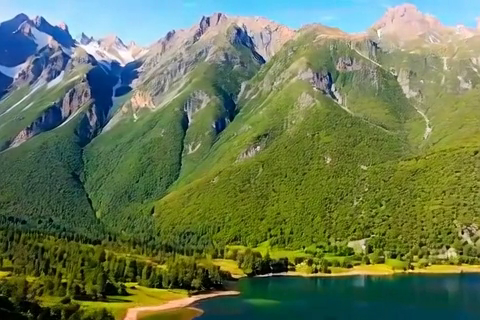}%
        \vspace{0.1cm}
        \includegraphics[width=3.5cm,height=2.03cm]{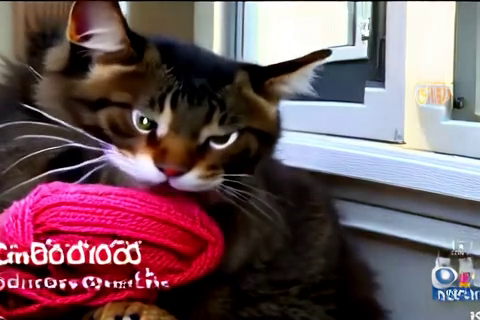}%
        \includegraphics[width=3.5cm,height=2.03cm]{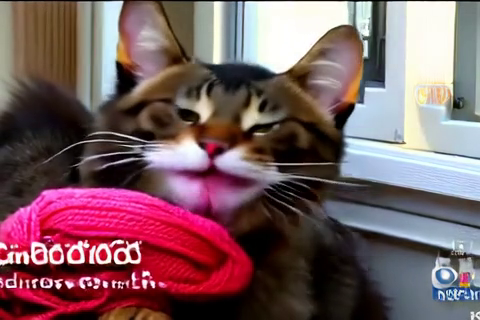}%
        \includegraphics[width=3.5cm,height=2.03cm]{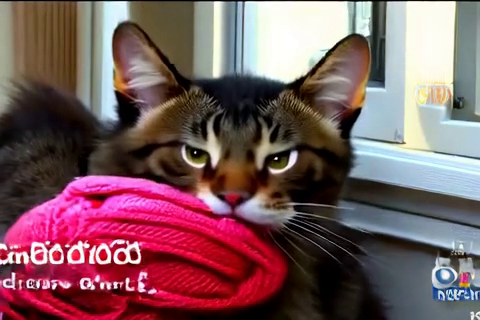}%
        \includegraphics[width=3.5cm,height=2.03cm]{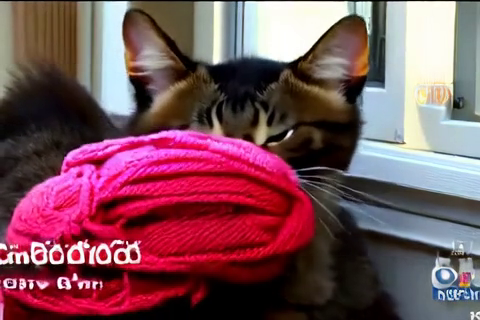}%
        \vspace{0.1cm}
        \includegraphics[width=3.5cm,height=2.03cm]{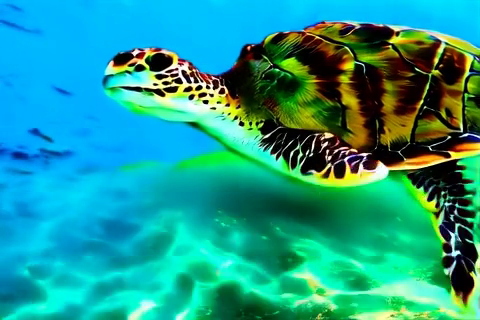}%
        \includegraphics[width=3.5cm,height=2.03cm]{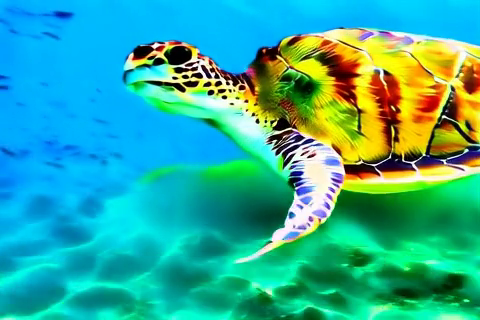}%
        \includegraphics[width=3.5cm,height=2.03cm]{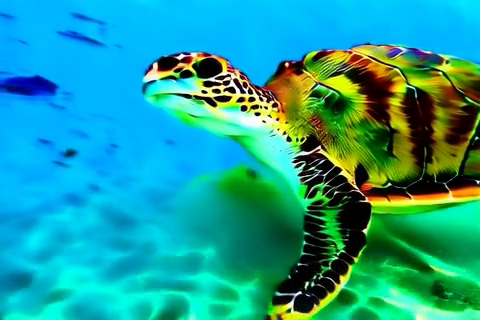}%
        \includegraphics[width=3.5cm,height=2.03cm]{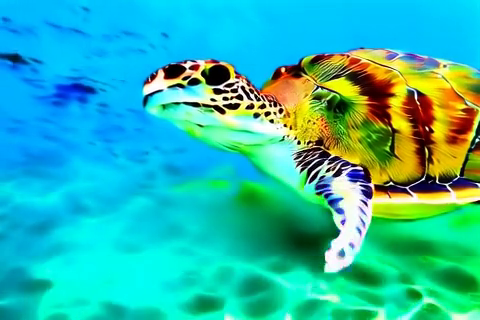}%
        \vspace{0.1cm}
        \includegraphics[width=3.5cm,height=2.03cm]{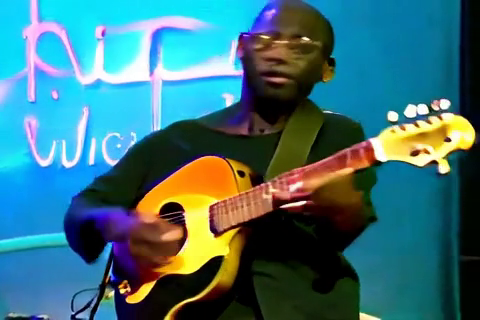}%
        \includegraphics[width=3.5cm,height=2.03cm]{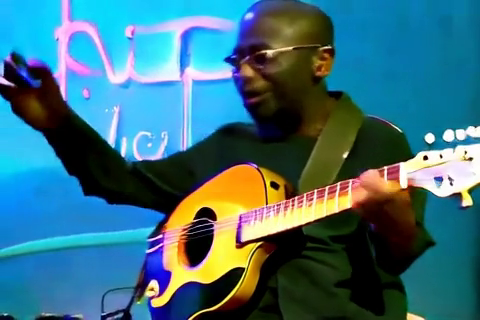}%
        \includegraphics[width=3.5cm,height=2.03cm]{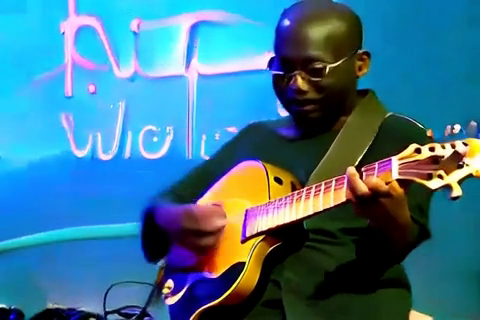}%
        \includegraphics[width=3.5cm,height=2.03cm]{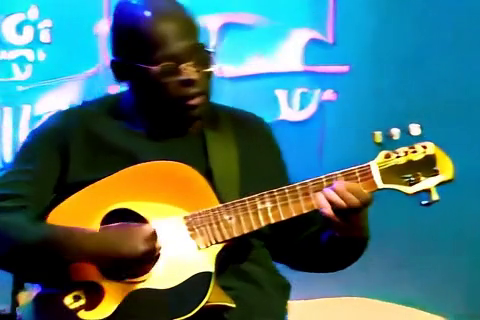}%
        \vspace{0.1cm}
        \includegraphics[width=3.5cm,height=2.03cm]{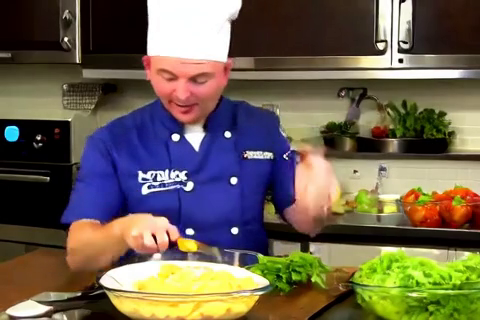}%
        \includegraphics[width=3.5cm,height=2.03cm]{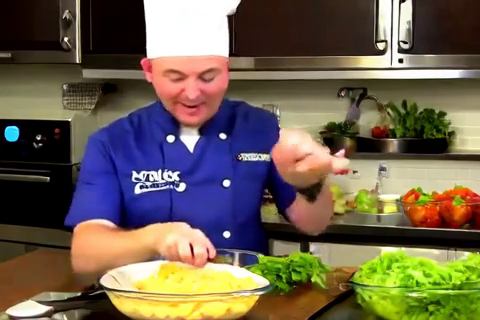}%
        \includegraphics[width=3.5cm,height=2.03cm]{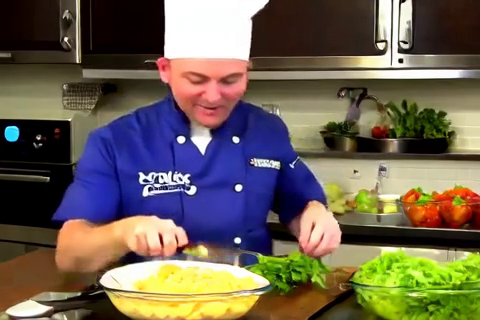}%
        \includegraphics[width=3.5cm,height=2.03cm]{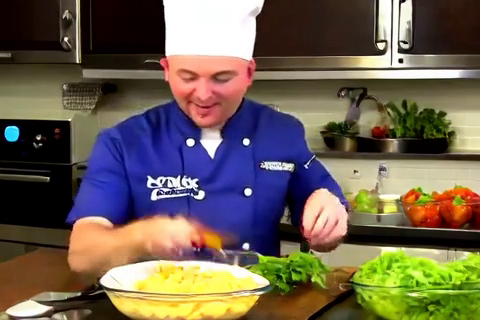}%
        \vspace{0.1cm}
        \includegraphics[width=3.5cm,height=2.03cm]{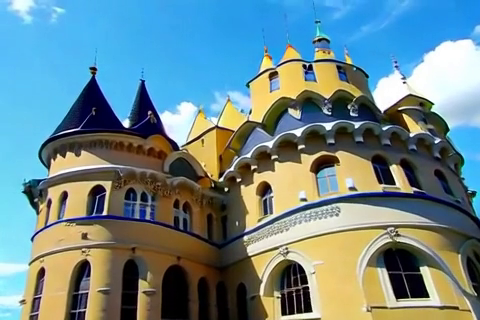}%
        \includegraphics[width=3.5cm,height=2.03cm]{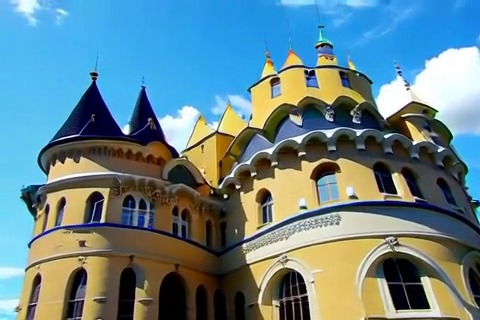}%
        \includegraphics[width=3.5cm,height=2.03cm]{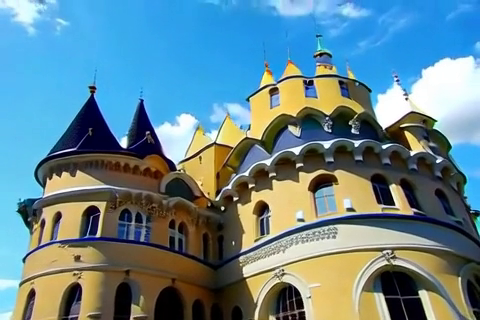}%
        \includegraphics[width=3.5cm,height=2.03cm]{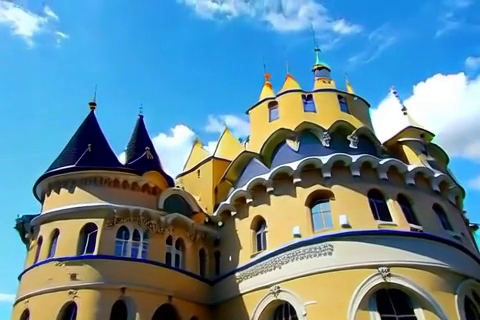}%
    
    \caption{Qualitative Examples \textbf{HLA-3F-R2-15}}
    \label{fig:qualitative.examples.hla.3f.r2.15}
\end{figure*}

\section{Code Listings}
We provide PyTorch-like code in listings~\ref{lst:hla-2f} and \ref{lst:hla-3f}. These algorithms do not require time-consuming tensor reshaping. They only require elementary computational operations.

\FloatBarrier

\begin{figure*}[!h]
\centering
\begin{minipage}{0.95\linewidth}
\begin{pytorchbox}
def HLA_2_factors(query, key1, key2, value):
    # equivalent to ((Q@K1.t()) * (Q@K2.t())) @ V
    # (omitting normalization for simplicity)

    # batch size, number of heads, sequence length, d_phi
    bs, nh, sl, dp = query.shape
                 d = value.shape[-1]

    # compute outer products between query and key1, key2 tokens
    k_outerprod_k = torch.einsum('bhsd,bhse->bhsde',key1, key2)
    context       = torch.einsum(
                        'bhmde,bhmg->bhdeg',
                        k_outerprod_k, value
                    )
    q_outerprod_q = torch.einsum('bhsd,bhse->bhsde',query, query)

    # normalization
    sum_over_k    = k_outerprod_k.sum(dim=2, keepdim=True)
    eta           = torch.sum(
                        q_outerprod_q.view(bs,nh,sl,-1) 
                        * 
                        sum_over_k.view(bs,nh,1,-1),
                        dim=-1, keepdim=True
                    )
    q_outerprod_q = q_outerprod_q / (eta.view(bs,nh,sl,1,1)+eps)

    # memory efficient computation
    q_outerprod_q = q_outerprod_q.view(bs, nh, sl, dp**2)
    context       = context.view(bs, nh, dp**2, d)

    # output of linear attention (before ouput projection)
    attention     = torch.einsum(
                        'bhsd,bhde->bhse',
                        q_outerprod_q, context
                    )

    return attention
\end{pytorchbox}
\captionof{figure}{PyTorch-like code of Hadamard Linear Attention with 2 factors.}
\label{lst:hla-2f}
\end{minipage}
\end{figure*}

\begin{figure*}[ht]
\centering
\begin{minipage}{0.99\linewidth}
\begin{pytorchbox}
def HLA_3_factors(query, key1, key2, key3, value):
    # equivalent to ((Q@K1.t()) * (Q@K2.t()) * (Q@K3.t())) @ V
    # (omitting normalization for simplicity)

    # batch size, number of heads, sequence length, d_phi
    bs, nh, sl, dp = query.shape
                 d = value.shape[-1]

    # compute outer products between query and key1, key2 tokens
    k_outerprods  = torch.einsum(
                        'bhsd,bhse,bhsf->bhsdef',
                        key1, key2, key3
                    )
    context       = torch.einsum(
                        'bhmdef,bhmg->bhdefg',
                        k_outerprods, value
                    )
    q_outerprods  = torch.einsum(
                        'bhsd,bhse,bhsf->bhsdef',
                        query, query, query
                    )

    # normalization
    sum_over_k    = k_outerprods.sum(dim=2, keepdim=True)
    eta           = torch.sum(
                        q_outerprods.view(bs,nh,sl,-1) 
                        * 
                        sum_over_k.view(bs,nh,1,-1),
                        dim=-1, keepdim=True
                    )
    q_outerprods  = q_outerprods / (eta.view(bs,nh,sl,1,1,1)+eps)

    # memory efficient computation
    q_outerprods  = q_outerprods.view(bs, nh, sl, dp**3)
    context       = context.view(bs, nh, dp**3, d)

    # output of linear attention (before ouput projection)
    attention     = torch.einsum(
                        'bhsd,bhde->bhse',
                        q_outerprods, context
                    )

    return attention
\end{pytorchbox}
\caption{PyTorch-like code of Hadamard Linear Attention with 3 factors.}
\label{lst:hla-3f}
\end{minipage}
\end{figure*}



\FloatBarrier



\end{document}